\documentclass{article}
\usepackage{arxivstyle/arxiv}
\usepackage[utf8]{inputenc}
\usepackage[T1]{fontenc}
\usepackage{amsmath,latexsym}
\usepackage{xurl}
\usepackage[boxruled]{algorithm2e}
\newcommand{\red}[1]{\textcolor{red}{#1}}
\usepackage{float,graphicx, subcaption, pgf,amsthm}

\usepackage[round,comma]{natbib}
\usepackage{afterpage,tabularx}
\usepackage{caption}
\usepackage{siunitx} 
\usepackage{multirow,multicol}
\usepackage{mathrsfs}

\usepackage[bb=dsserif]{mathalpha}
\usepackage{bm}
\usepackage[hyperfootnotes=false,hidelinks]{hyperref}

\newcommand{\IN}{\mathbb{N}}
\newcommand{\IE}{\mathbb{E}}
\newcommand{\IR}{\mathbb{R}}
\newcommand{\X}{\mathcal{X}}

\newcommand{\F}{\mathcal{F}}
\newcommand{\eps}{\varepsilon}
\newcommand{\minitab}[2][l]{\begin{tabular}{#1}#2\end{tabular}}
\newcommand{\nef}{n_{\text{eff}}}

\newtheorem{thm}{Theorem}[section]

\newtheorem{prop}[thm]{Proposition}

\usepackage{moreverb} 

\usepackage[resetlabels]{multibib}
\newcites{A}{Appendix References}

\immediate\write18{texcount -inc -incbib 
-sum main.tex > /wordcount.tex}

\newcommand\blfootnote[1]{%
  \begingroup
  \renewcommand\thefootnote{}\footnote{#1}%
  \addtocounter{footnote}{-1}%
  \endgroup
}

\usepackage[page]{appendix}

\renewcommand{\appendixtocname}{Appendix Contents.}

\makeatletter
\let\oldappendix\appendices

\renewcommand{\appendices}{%
  \clearpage
  \renewcommand{\thesection}{\Roman{section}}
  \let\tf@toc\tf@app
  \addtocontents{app}{\protect\setcounter{tocdepth}{2}}
  \immediate\write\@auxout{%
    \string\let\string\tf@toc\string\tf@app^^J
  }
  \oldappendix
}%

\newcommand{\listofappendices}{%
  \begingroup
  \renewcommand{\contentsname}{\appendixtocname}
  \let\@oldstarttoc\@starttoc
  \def\@starttoc##1{\@oldstarttoc{app}}
  \tableofcontents
  \endgroup
}

\makeatother

\title{Pitfalls of Climate Network Construction}
\author{Moritz Haas\\
\texttt{mo.haas@uni-tuebingen.de}\\
University of Tübingen, Germany
\And
Bedartha Goswami\\
\texttt{bedartha.goswami@uni-tuebingen.de}\\
University of Tübingen, Germany
\And
Ulrike von Luxburg\\
\texttt{ulrike.luxburg@uni-tuebingen.de}\\
University of Tübingen, Germany
}

\begin{document}

\flushbottom
\thispagestyle{empty}


\maketitle

\vspace{8mm}

\begin{abstract}

Network-based analyses of dynamical systems have become increasingly popular in climate science. Here we address network construction from a statistical perspective and highlight the often ignored fact that the calculated correlation values are only empirical estimates. To measure spurious behaviour as deviation from a ground truth network, we simulate time-dependent isotropic random fields on the sphere and apply common network construction techniques. We find several ways in which the uncertainty stemming from the estimation procedure has major impact on network characteristics. When the data has locally coherent correlation structure, spurious link bundle teleconnections and spurious high-degree clusters have to be expected. Anisotropic estimation variance can also induce severe biases into empirical networks. We validate our findings with ERA5 reanalysis data. Moreover we explain why commonly applied resampling procedures are inappropriate for significance evaluation and propose a statistically more meaningful ensemble construction framework. By communicating which difficulties arise in estimation from scarce data and by presenting which design decisions increase robustness, we hope to contribute to more reliable climate network construction in the future.\blfootnote{This Work has been submitted to Journal of Climate. Copyright in this Work may be transferred without further notice.}

\end{abstract}

\tableofcontents


\addtocontents{toc}{\setcounter{tocdepth}{2}}

\section{Introduction}
\label{sec:intro}

Climate networks are constructed to find complex structure such as teleconnections \citep{Boers2019}, clusters \citep{Rheinwalt2015}, hubs, regime transitions \citep{Fan2018} or bottlenecks \citep{Donges2009} in the climatic system. Network-based approaches have shown considerable improvements in the prediction of several climate phenomena \citep{Ludescher2021} such as El Ni\~{n}o events \citep{Ludescher2014}, extreme regional precipitation patterns \citep{Boers2014} or anomalous polar vortex dynamics \citep{Kretschmer2017}. Typically, climate networks are constructed in a 3-step procedure. First, choose a data set of climatic variables, such as temperature or precipitation, measured on a fixed spatial grid. Then choose a notion of similarity between pairs of locations based on the corresponding time series in the data set. Finally, construct a network with spatial locations as nodes and with edges between those pairs of locations that have the strongest similarities. 
Since we only have access to noisy time series of finite length, the calculated similarity values between pairs of locations will be noisy themselves: they are subject to \textit{estimation variability}. As a consequence, any climate network that is constructed on a finite amount of data might contain false edges (which should not be present) and have missing edges (which should be present). This leads us to the following important questions that have not received enough attention so far: 

\begin{center}
\fbox{%
  \parbox{0.78\textwidth}{Which kinds of distortions are induced into climate networks due to the sampling variability of the underlying time series? Which features of climate networks can be attributed to underlying structure, and which are random artifacts due to finite-sample variation?}%
  }
\end{center}
 

These are the questions we discuss in this paper from a decisively statistical point of view.
%
First observe that a climate network is built on a large number of pairwise similarity estimates: if our grid consists of $10^4$ locations, a naive procedure needs to estimate $10^8$ pairwise similarities. Even extremely well-behaved estimators with a small variability will create a non-negligible number of wrong edges in the network. But not many "wrong" or "missing" edges are necessary to distort important structural network characteristics. Even a single false long-range edge can substantially distort important network measures such as shortest path lengths, small world properties, centrality and betweenness measures, or the emergence of teleconnections. And through the local correlation structure that is inherent in climate data, wrong edges can propagate, leading to many wrong edges, even inducing wrong `link bundles', i.e. distinct regions being connected by multiple edges.

To assess the severeness of this problem, we introduce a new null model for sampling time series that shares important properties with the Earth's climate system, but at the same time is simple enough so we can control it, understand it, and simulate from it. 
To achieve this, we employ a locally correlated, isotropic data generating process: Isotropic random fields on the sphere. 
The key feature of this model is that the similarity of two time series only depends on the distance of the respective locations, nothing else. Locations that are close-by tend to have more similar time series than locations that are far apart. Our model can thus capture important properties of real climate networks such as link-length distributions, but through its isotropic nature it is simple enough so that erroneous patterns in the network can be clearly identified as statistical distortions. We introduce time dependence via a vector autoregression process (VAR(1)), which allows us to adjust the autocorrelation on each node. Consequently, the temporal autocorrelation structure can depend on the location, but the spatial correlation structure, and with it the ground truth network, remains approximately isotropic.



Sampling our null model allows us to systematically investigate the connection between noise in the similarity estimates and distortions in the network. Although the simulated data is only locally correlated, we find that complex network structures arise in the estimated networks because of imperfect estimation. For example, global spatially coherent betweenness patterns emerge (Fig. \ref{fig:gauss_betweenness}),  which do not represent any ground truth structure. We also study the influence of choosing different similarity estimators, the influence of network sparsity on betweenness, distortions of other popular network measures, the emergence of spurious link bundles and high degree clusters, and the biases introduced through anisotropic estimation variability. For example, we find that inappropriate estimators can result in arbitrarily wrong network estimates (Fig. \ref{fig:truevsemp}). On the other hand, we illustrate that a conscious choice of network construction techniques may increase robustness with respect to ground truth networks and may uncover different dynamics in the system.
To filter out spurious edges, \citet{Boers2019} consider links as significant that do not appear alone but in bundled form. We show that when the data are locally highly correlated, the presence of one spurious long link makes the presence of neighboring links quite likely as well, leading to entire spurious link bundles.



In addition to our simulation results, we validate our findings with reanalysis data from the ERA5 project. We find that the tendency to form bundled connections increases with the strength of local correlations (Fig. \ref{fig:realbundles}). The betweenness structure in temperature-based networks highly depends on the network density and the used data set. This raises the question whether finding a "betweenness backbone", as in \citep{Donges2009}, is possible and meaningful. For most climatic variables, we detect severe instability for long links. The nodes of highest degree tend to have autocorrelation (cf. \citet{Palus2011}). We conjecture that some edges from these nodes are spurious and induced by the increased estimation variability on these nodes.

The wide range of potential empirical distortions makes a re-assessment of many of the previous findings in the climate network literature desirable. However, this poses a big challenge: while our simulation study is based on a model with known ground truth, such a ground truth is not available for real-world climate networks. Yet, as our simulations show, it is extremely important to assess the reliability and robustness of findings based on empirical climate networks.
Typically, researchers use approaches based on node-wise reshuffling of the time series or edge-wise reshuffling of the given network. But we demonstrate that such techniques are inadequate to capture the inherent uncertainty of the network. Instead we propose to estimate the variability in the network by computing multiple correlation estimates for each edge, while retaining the original spatial similarity structure. With this approach, we might get a statistically meaningful sense of the reliability of network patterns constructed from real, noisy time series.


To summarize our main contributions,
\begin{itemize}
    \item We introduce a VAR$(1)$-process of isotropic random fields on the sphere as a suitable null model for geophysical processes, for which deviations from the ground truth are easily detectable.
    \item We identify systematically occuring random artifacts and distortions in empirical networks, and analyse why they arise.
    \item We show which design decisions increase the robustness of constructed networks.
    \item We validate our findings with ERA5 reanalysis data.
    \item We discuss the shortcomings of common network resampling procedures for significance evaluations and propose a statistically more meaningful framework based on jointly resampling the underlying time series. 
\end{itemize}

The rest of the paper is organized as follows. In Section \ref{sec2} we describe typical network construction steps and introduce the isotropic data-generating process we employ in our simulations. We present intuitions about the ground truth networks and explain when spurious behaviour is to be expected in the empirical networks. Section \ref{sec3} demonstrates several common patterns of spurious behaviour in typically constructed networks, categorized into: a. estimator selection, b. network measures, c. link bundles and d. anisotropy. Section \ref{sec4} points out problematic practices in significance testing and potential improvements. Finally Section \ref{sec:conclusion} provides conclusions and possibilities for future work. For readers who are unfamiliar with climate network methodology we have assembled an introduction in Appendix \ref{sec:intro_cn_terminology}.

\section{Network Construction for Data from Spatio-Temporal Random Fields}
\label{sec2}


To study artifacts that are introduced by estimation procedures, we need access to a "ground truth network", which is not available for real-world climate data.
We therefore introduce a manageable stochastic process over the Earth with known ground truth structure. We then use the model to evaluate how estimation procedures introduce random artifacts into the network estimates depending on network construction steps and features of the data distribution.

\subsection{Climate Network Construction}
\label{sec:net_constr}

The generic procedure of constructing climate networks from spatio-temporal data is described in Algorithm \ref{alg:netconstr}: Most studies deal with univariate real-valued data at each point in time and space such as temperature, pressure or precipitation; so do our experiments. Given a dataset of such time series on a fixed grid, the similarity between pairs of grid points is estimated. Popular similarity measures include Pearson correlation, mutual information and event synchronisation \citep{Quian2002}. There are several ways to construct a network based on all pair-wise similarity estimates. Most often unweighted density-threshold graphs are constructed \citep{Tsonis2004,Yamasaki2008,Agarwal2019,Kittel2021}, which means that an edge of weight $1$ is formed between two grid locations $v_i$ and $v_j$ when the corresponding similarity estimate $\hat{S}_{ij}$ surpasses a certain threshold. This threshold is chosen so that a desired network density is attained. Another popular approach is edge formation based on significance tests with respect to reshuffled time series \citep{Palus2011,Boers2013,Boers2014,Deza2015}. Here the time series at both end locations of an edge are reshuffled to get a baseline distribution of how similarity estimates behave when the time series are independent. The edge is then formed if the original similarity estimate surpasses a predefined significance threshold.

\begin{minipage}{1\linewidth}
\begin{algorithm}[H]
\SetAlgoLined
\phantom{;}\textbf{Input} : \textbf{Spatio-temporal data $\{X_{i t}\}_{i\in [p], t\in [n]}, X_{i \cdot}=(X_{i1},\dots,X_{in})$} of time length $n$ measured on \textbf{$p$ fixed locations $V=\{v_i\;|\;i\in [p]\}$} in some metric space $(\X,d)$ such as the sphere; similarity measure of interest $S: \X \times \X \to [0,\infty)$ between two locations and estimator $\tilde{S}:\IR^n \times \IR^n \to [0,\infty)$ of $S$ based on the finite time series.\\
 \phantom{;}1.) \textbf{Estimate the similarity} between two points $v_i$ and $v_j$ based on the data and some estimator $\tilde{S}$ of the chosen similarity measure $S$: $\hat{S}_{ij}=\tilde{S}(X_{i \cdot},X_{j \cdot})$.\\
 \phantom{;}2.) \textbf{Construct a graph} with adjacency matrix $\hat{A}$ from the similarity estimates $\hat{S}$, parameters $\theta$ and potentially summary statistics of the data. For example, in case of the unweighted $\tau$-threshold graph, $
 \hat{A}_{ij} = \begin{cases} 1, & \hat{S}_{ij} \geq \tau,\\ 0, & \hat{S}_{ij} < \tau.  \end{cases}
 $

\caption{ \textsc{Functional Network Construction from Spatially Gridded Data} }
\label{alg:netconstr}
\end{algorithm}
\end{minipage}

\subsection{Stochastic Ground Truth Model for Spatio-Temporal Data}
\label{sec:igrf}

To quantify how the similarity estimation process influences the induced networks, we specify a ground truth model, using random fields over the sphere, approximating Earth's surface. Our goal is not to give an accurate model of Earth's climate, but to point out generic patterns of spurious behaviour in networks constructed from a limited amount of spatio-temporal data. The simpler the data-generating model remains, the more accurately we can attribute spurious behaviour to certain features of the data distribution or the employed network construction steps. We use a data generating process in which the correlation between data measured on different locations depends only on the distance between the locations. Such isotropic random fields are common in geostatistics \citep{Cressie1993,Lang2015} and they allow us to attribute anisotropies in the estimated networks as erroneous.

Here we first introduce the spatial stochastic process and, in a second step, add time-dependence. The mathematical process that we are going to use is an `Isotropic Gaussian random field'. A \textit{random field} assigns a real value to every point of the sphere, imagine a surface temperature field. Centering (and possibly detrending and normalizing) data on each point in space yields a \textit{zero-mean} random field, representing so called (detrended standardized) anomalies. When evaluating a \textit{Gaussian} random field on finitely many points, its values are jointly Gaussian distributed. For \textit{isotropic} random fields, the covariance between two points is solely determined by the distance between the points. Hence, a zero-mean isotropic Gaussian random field is fully characterized by its covariance function $k$, which determines how smoothly and to what extent the random field varies across space.

Formally, a \textbf{zero-mean isotropic Gaussian random field} $G$ on the sphere with covariance function $k:[0,\pi]\to \IR$ is defined as a collection of real-valued random variables $\{ G(v)\}_{v\in S^2}$ such that $\IE [\,G(v)\,] = 0$ for all $v \in S^2$ and, given a finite grid $\{v_i\}_{i=1,\dots,p}\subset S^2$, the random field's values on the grid points are jointly Gaussian distributed, \[
\big(G(v_1),\dots,G(v_p)\big)\sim N(0,\Sigma),
\]
with covariance $\Sigma_{ij}=k(|v_i-v_j|)$.

One popular covariance function is the Mat\'ern covariance function (Appendix \ref{sec:matern}), whose smoothness parameter $\nu$ and scale parameter $\ell$ make it flexible as well as interpretable. It allows to interpolate between the absolute exponential kernel and the Gaussian radial basis function (\citealp{Stein1999}, ch. 2.10), and monotonically decreases with distance irrespective of parameter choice. We introduce the abbreviation MIGRF($\nu$, $\ell$) for a zero-mean isotropic Gaussian random field with Mat\'ern covariance, smoothness $\nu$ and length scale $\ell$. Fig. \ref{fig:matern}d shows realizations of an MIGRF with varying parameters, when traversing the sphere from South to North pole. Low smoothness $\nu$ results in abrupt changes. In expectation, processes with smaller length scales $\ell$ contain larger fluctuations on a fixed interval. We choose $\nu\in\{0.5,1.5\}$ and $\ell\in\{0.1,0.2\}$ (in radians) to reflect realistic values for climatic time series \citep{Guinness2016} (Appendix \ref{sec:nu_ell}) as well as to point out their influences on the estimation procedure.
%

We introduce time dependence via a vector autoregression VAR(1) (Appendix \ref{sec:time_dep}) that allows us to assign any desired lag-1 autocorrelation to each node of the network. Under this basic time-dependence, we will be able to separate the effect of autocorrelation on the estimation procedure from other influences.

\begin{figure}[H]
    \noindent\includegraphics[width=39pc,angle=0]{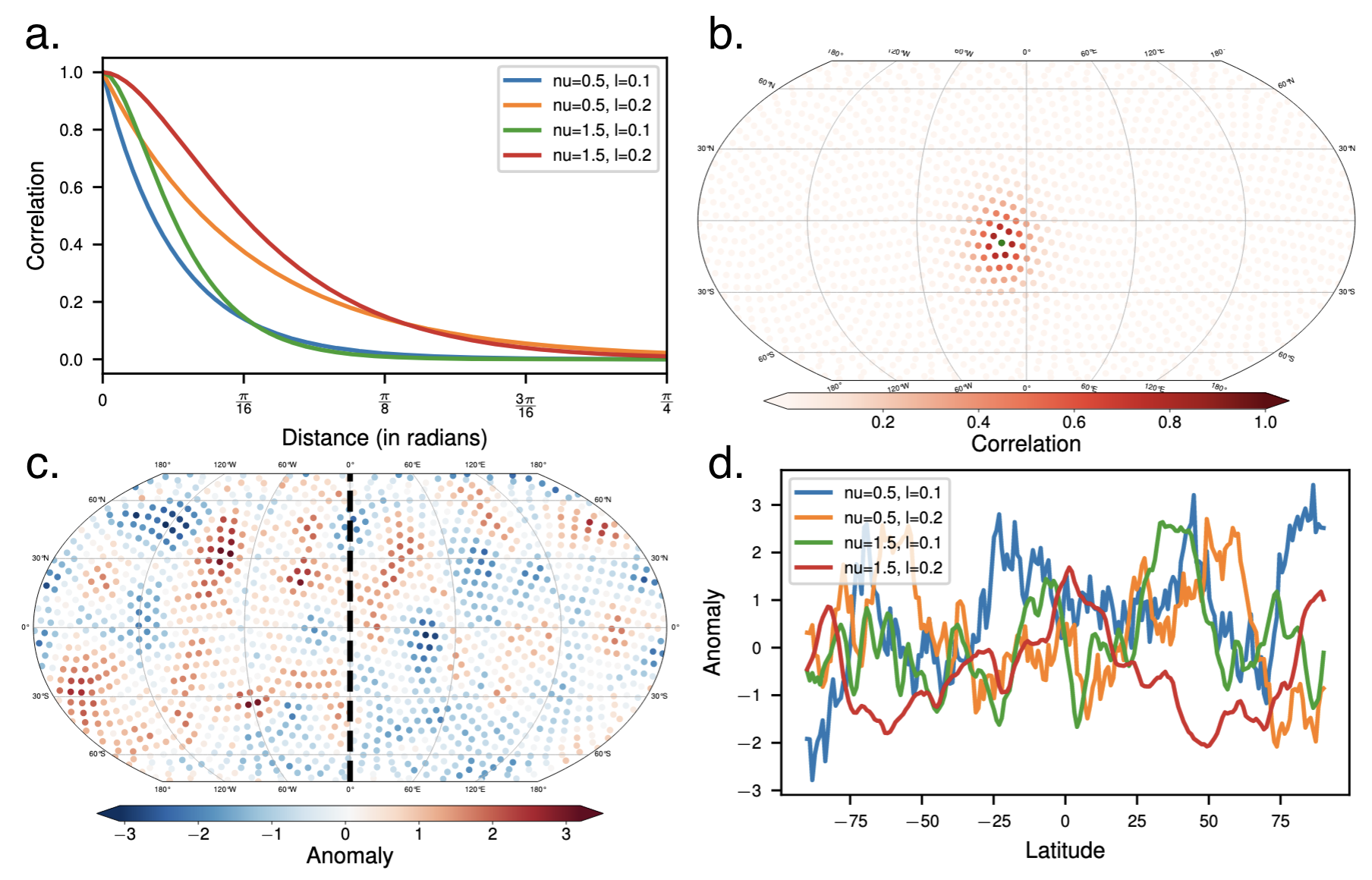}
    \caption{\textbf{Isotropic Gaussian Random Fields.} \textbf{a)} The Mat\'ern correlation function for different parameter choices. \textbf{b)} True Mat\'ern correlation with respect to the green point for $\nu =1.5$ and $\ell=0.2$. \textbf{c)} A realization of an MIGRF($\nu=1.5$, $\ell=0.2$), representing anomalies at a fixed time point. The correlation function induces smoothly varying values. The dashed black line shows the geodesic path from South to North pole used in d. \textbf{d)} Random realisations of MIGRFs with different parameters evaluated on the geodesic path shown in c.}
    \label{fig:matern}
\end{figure}

\subsection{Ground Truth Networks and Imprecise Estimates}
\label{sec:fine_grids}

\textbf{Ground truth networks.} If we fix a grid and a network construction method, a ground truth data distribution leads to a ``true network'' on this grid. Given the underlying data distribution, as in our model, we can calculate the true pairwise similarities between grid points. The network construction procedure then determines the ground truth network based on the true similarities. For example, a ground truth density-threshold graph simply consists of the edges corresponding to the largest similarity values. How much ground truth structure of a climatic process can be captured in the ground truth network depends on the choice of climatic variable, grid, similarity measure and network construction scheme. Another question is then whether this ideal network can be approximated with the available empirical data and estimators.

\textbf{Errors in the estimated networks.} Given a finite amount of data, we only have access to imperfect estimates of the true similarity values. Consequently, networks constructed from data as well as their characteristics will only be estimates of the corresponding ground truth quantities and inherit \textit{intrinsic variability}. When the chosen similarity estimator is not suitable for the estimation task, the constructed graphs can look arbitrarily wrong (Fig. \ref{fig:truevsemp}). However, by simple inspection it is not possible to judge whether a constructed climate network reflects ``true'' aspects of the physical system or whether it is dominated by random artifacts introduced through the estimation procedure. For this reason, in our simulations we mainly address the question:
\textbf{How do the estimated networks and their characteristics differ from their corresponding ground truth quantities?}
The answer depends on the properties of the random field, the employed estimator and the considered network characteristic (see Section \ref{sec3}). To get started, let us discuss how wrong individual edges occur, and then how wrong link bundles arise. 

\textbf{Errors in individual edges.} Errors in the network occur because the similarity estimates between locations vary around the ground truth similarity values. 
False positives edges are wrongly included in the empirical network but are not present in the ground truth network; false negative edges appear in the ground truth network but are missing in the empirical network. Let us understand when these two cases arise in threshold graphs. Assume that the similarity estimate $\hat{S}$ over an edge with ground truth correlation $S$ is imprecise and follows the distribution $\mathcal{N}(\IE [\, \hat{S}\,],\sigma^2)$ (we consider the normal distribution for simplicity; other distributions lead to qualitatively similar behavior). The probability that this edge is formed in the $\tau$-threshold graph is then given by $\Phi(\frac{\IE [\,\hat{S}\,]-\tau}{\sigma})$, where $\Phi$ denotes the cumulative distribution function of the standard normal distribution. A false positive can only occur when the true similarity $S$ is smaller than the threshold $\tau$. Then the error probability is not negligible when the similarity estimates are upward biased ($\IE [\,\hat{S}\,] \gg S$), or when the estimation variance $\sigma^2$ is large. Analogously, the probability of a false negative is not negligible when the estimate is downward biased or when the estimation variance $\sigma^2$ is large. As we will see in Section \ref{sec:anisotropy}, the variability in the estimates grows, or in other words their signal-to-noise ratio decreases, with data scarcity and increasing autocorrelation in the observations. A bias in the similarity estimates can be introduced by the estimator. Taken together, finding an estimator with good bias-variance trade-off for the given similarity measure can significantly reduce the number of false edges. In particular, when the desired graph density is chosen so large that many ground truth correlation values of included and excluded edges are similarly small, the likelihood of spurious behaviour increases as these edges cannot be well distinguished under the estimation variance. We see this in our experiments below when we construct dense graphs over small-scale correlation structure (Fig. \ref{fig:fdr}).

\begin{figure}[H]
    \noindent\includegraphics[width=39pc,angle=0]{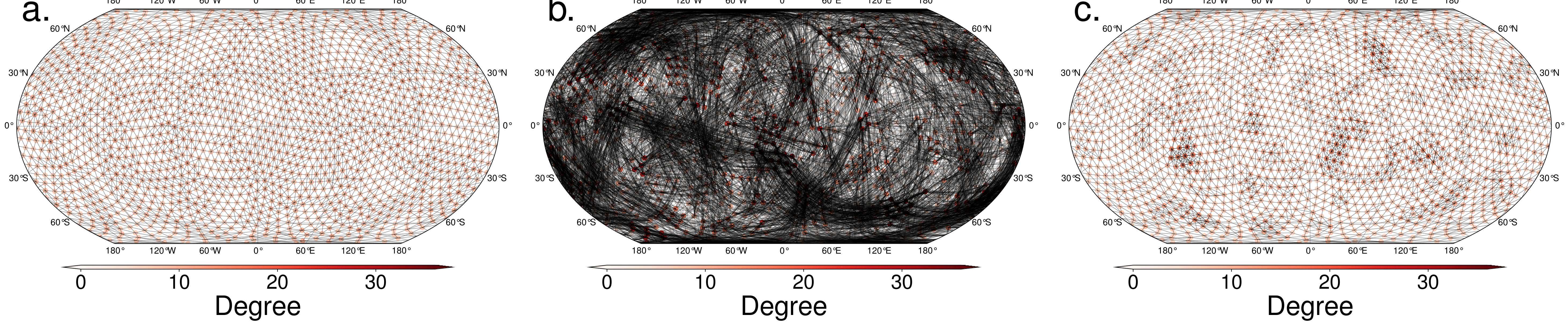}
    \caption{\textbf{a)} Ground truth network of density $0.005$ for monotonically decreasing correlation structure. The shortest links possess the largest ground truth correlation values. The graph is not perfectly isotropic because an isotropic grid does not exist. \textbf{b)} Empirical estimate of the left ground truth network with the \textit{same network density} given log-normal data based on an MIGRF$(\nu=1.5,\, \ell=0.2)$ with variance $10$ and $n=100$ using empirical Pearson correlation (see Section \ref{sec:heavytails}). Many false links arise due to high estimation variance. Many long links clutter the image. Observe (strong) spurious \textit{bundled} teleconnections.
    \textbf{c)} Empirical estimate over the same data using Spearman correlation. No long links are formed, but we can observe spuriously dense regions.}
    \label{fig:truevsemp}
\end{figure}

\textbf{How errors spread locally due to covariance.} When the data are locally highly correlated, as is typical for climatic variables, this correlation may carry over to the joint distribution of similarity estimates. As a result, an error may propagate from one edge to edges on neighboring nodes in the following way: When the similarity estimate on one false edge is spuriously large, it is likely that the correlation estimates on edges on neighboring nodes are similarly large, so that these neighboring edges are also falsely included in the empirical network, resulting in false bundles of edges. In density-threshold graphs, this makes some regions to spuriously appear denser than others. A formal argument is given in Appendix \ref{sec:false_edges}. 
Combining the thoughts above, \textbf{false bundles of edges occur with high probability when measurements from closeby points are highly correlated and the similarity estimates are imprecise}. Find related simulation results in Section \ref{sec:bundles_clusters}.


\section{Spurious Behaviour in Networks from Finite Samples}
\label{sec3}

In this section we explore the effects that imprecise estimates impose on commonly constructed climate networks. We do so by simulating the isotropic Gaussian random fields introduced above.

\textbf{Network construction.} We construct networks following Algorithm \ref{alg:netconstr}. To approximately remove the effects of anisotropic grids we generate a Fekete grid \citep{Bendito2007} after 1000 iterations with 5981 points, approximately realizing an isotropic grid of $2.5^\circ$ resolution. If not stated differently we sample an MIGRF($\nu$, $\ell$) independently in time with $n=100$. From the ERA5-reanalysis dataset between 1979 and 2019, we consider monthly temperature of air 2 meters above the surface (t2m), surface pressure (sp), total precipitation (pr) and geopotential height at $250$, $500$ or $850$ mb (z250, z500, z850) as well as daily t2m (dt2m). We linearly detrend all ERA5 variables and subtract the monthly climatology. Finally, real and simulated data sets are centered and normalized in each grid point to result in detrended anomalies. In some simulations time-dependence is introduced, amplifying our findings (see Section \ref{sec:anisotropy} and Appendix \ref{sec:effective_samples}).
Many studies construct correlation networks from sliding windows \citep{Radebach2013,Hlinka2014,Fan2017,Kittel2021}. Typically, these windows cover at most a year of daily observations. While more measurements in time increase the accuracy of estimated networks, our findings also hold for larger $n$ (Appendix \ref{sec:large_n}).

\textbf{Visualisations.} For network visualisations we use a Fekete grid with 1483 points, approximately realizing a $5^\circ$ resolution. In our figures, dashed lines always denote ground truth values. Uncertainty bands cover the range between the empirical $0.025$ and $0.975$ quantile from 30 independent repetitions. `x' and circle denote $95\%$- and $99\%$-quantiles of a distribution respectively, and triangles minimal and maximal values of a distribution.

\subsection{Estimation}
\label{sec:estimation}

\subsubsection{Unsuitable Estimators Can Induce Many Wrong Edges}
\label{sec:heavytails}

\textbf{Problem.} When the marginal distribution on the nodes is heavy-tailed, as in the case for precipitation data, commonly applied estimators become inadequate if they are sensitive to outliers. For instance, the naive correlation estimator has unusably large variance under heavy-tailed distributions; yet it has been applied to precipitation data in several studies \citep{Scarsoglio2013,Ekhtiari2019,Ekhtiari2021}.

\textbf{Simulation results.} We simulate heavy-tailed data by exponentiating the MIGRF data. Let $G$ be a centered MIGRF with correlation function $k(\cdot)$ and variance $\sigma^2$. By setting $H(x) = \exp(G(x))$, $H$ defines a log-normal isotropic random field. For each point $x$ on the sphere, we get \[ Cor(H(x),H(y)) = \frac{e^{\sigma^2 k(|x-y|)}-1}{e^{\sigma^2}-1}.\]

Choosing $\sigma^2$ allows to continuously adjust the heaviness of the tails: While small values of $\sigma^2$ approximately recover the original correlation function $k$, increasing $\sigma$ exponentially enhances the tail strength. For large $\sigma^2$ the correlation between grid points quickly drops to 0 with distance. We choose $\sigma^2=10$, which is the correct order of magnitude to fit precipitation tails on global mean \citep{Papalexiou2018}. Note that the precipitation distribution on Earth crucially depends on the location. Here, we solely aim to illustrate the intricacy of handling heavy-tailed data through isotropic simulations.
Figure \ref{fig:truevsemp} demonstrates that the empirical correlation fails as a correlation estimator of data sampled from $H$. Because the empirical covariance is an average of log-normal random variables, it will be a large variance estimator of the population covariance. The \textbf{large estimation variability induces many (possibly bundled) false and missing links}. For short time series, you can imagine single events dominating on each node. When these events occur at the same time for a pair of nodes, the nodes will show high empirical correlation although the true correlation may be zero.

\textbf{Consequences.} Removing outliers or finding a suitable data transformation reduces this problem. By design $\log(\cdot)$ would transform the random field back to a Gaussian random field. Alternatively we can employ an estimator that is robust to heavy-tailed distributions \citep{Minsker2017}. Since Spearman correlation is invariant under monotonous transformations, it produces exactly the same results for the normal and log-normal data. An alternative to Spearman correlation with faster convergence rates is Kendall's Tau \citep{Gilpin1993}. \citet{Barber2019} consider several of the above ideas to estimate correlation in the context of hydrologic data.

\subsubsection{Comparing Similarity Measures as well as Estimators}
\label{sec:estimators}

\textbf{Problem.} 
While the empirical Pearson correlation estimator has often been equated with the corresponding similarity measure, we can strictly reduce the estimation variance in Pearson correlation networks by considering a different estimator --- even for Gaussian data. \citet{Radebach2013} have shown that many characteristic network patterns are already visible in Pearson correlation networks, and historically Pearson correlation has been the most popular similarity measure (e.g. \citet{Tsonis2004,Tsonis2008,Yamasaki2008,Palus2011,Fan2022}). As estimators of mutual information need to be able to capture arbitrarily complex dependence structures, they tend to require even larger sample size to achieve reliable accuracy than correlation estimators, resulting in more spurious behaviour given the same sample size.

\textbf{Simulation results.} We consider the following similarity measures and their estimators. For mutual information (MI) we use a simple binning estimator as applied in the complex network python package \texttt{pyunicorn} \citep{pyunicorn}, where we use $\lfloor \frac{n}{5}\rfloor$ bins as suggested in \citet{Cellucci2005} (a more conservative criterion than Cochran's, which was applied in \citet{Donges2009b}). To evaluate the importance of the estimator, we also employ a bias-corrected version of the popular KSG mutual information estimator \citep{gao2016demystifying,Kraskov2003} with $k=5$. As an alternative to mutual information we explore an estimator of the Hilbert Schmidt Independence Criterion (HSIC) for random processes \citep{Chwialkowski2014}. For correlation, we also employ the linear Ledoit-Wolf estimator \citep{Ledoit2004}, which counteracts the distortion of high-dimensional empirical correlation matrices by shrinking their eigenvalues.

Figure \ref{fig:fdr} shows the false discovery rate (FDR), which measures the fraction of false links, as a function of network density. Although less true links are available for small densities, sparse graphs are more accurate in terms of FDR, because the correlation values of ground truth links can be empirically distinguished with high certainty from most false links under estimation variability. For random fields with long length scale, this empirical separability remains intact for longer edges. Therefore the FDR remains low up to larger network densities (see also Fig. \ref{fig:mad_fdr}).
Given the same amount of data, more complex similarity measures perform worse. For sparse graphs, the Hilbert-Schmidt independence criterion shows promising performance compared to the mutual information estimators. The BI-KSG estimator is computationally expensive with fluctuating performance, and the binned MI estimator is strictly worse than HSIC. Unweighted empirical Pearson and Ledoit-Wolf density-threshold networks coincide because they produce the same ranking of edge weights.
Figure \ref{fig:fdr}c shows the error of the estimated correlation matrix compared to the ground truth in Frobenius norm under various hyperparameter settings. The ground truth correlations grow monotonously from left to right. Note that the empirical correlation matrix makes large estimation errors irrespective of the parameters of the random field. The linear Ledoit-Wolf estimator improves the estimation in all cases. Consequently, fixed-threshold networks as well as weighted networks are better approximated by the Ledoit-Wolf estimator. The less correlated the grid points are, the lower the error of the Ledoit-Wolf estimator as it shrinks the correlation estimates towards an identity matrix.

\begin{figure}[H]
    \noindent\includegraphics[width=39pc,angle=0]{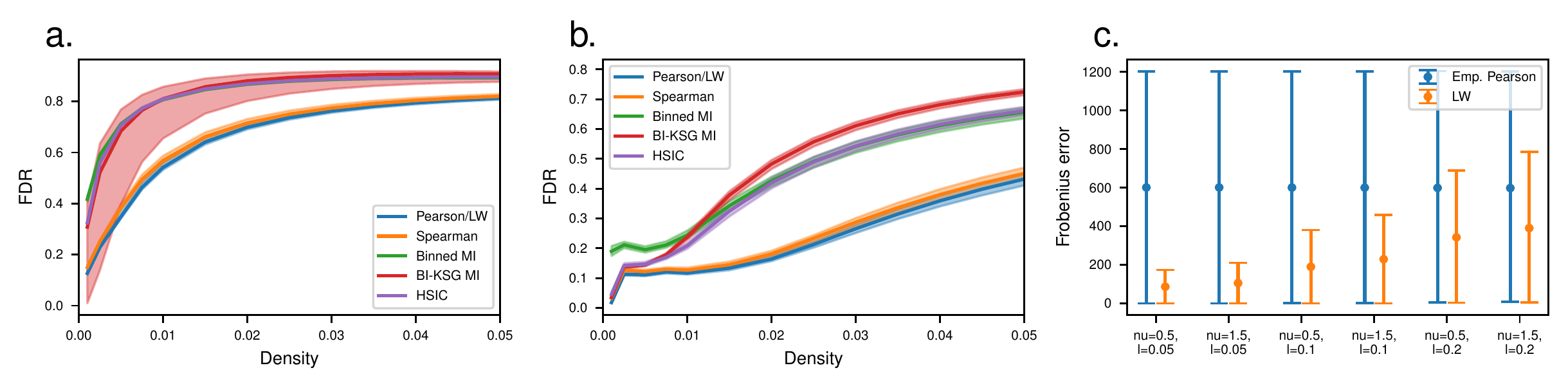}
    \caption{\textbf{Errors in empirical networks.} \textbf{a)} False discovery rate for various similarity measures and different estimators of the same similarity measures for a non-smooth MIGRF with short length scale $\nu = 0.5$ and $\ell = 0.1$. Sparse networks have a better false discovery rate. Empirical Pearson correlation and the Ledoit-Wolf estimator coincide in unweighted density-threshold networks (see main text). \textbf{b)} Same as a) for a smooth MIGRF with long length scale $\nu=1.5$ and $\ell=0.2$. Estimation performance reamains reasonable good up to larger network density. \textbf{c)} Error of estimated correlation matrix from ground truth in Frobenius norm, which is proportional to the root mean squared error per edge weight estimate. The number of errors in empirical networks is alarming for all hyperparameter settings. Suitable estimators, such as the Ledoit-Wolf estimator, drastically reduce the error in the edge weight estimates compared to empirical Pearson correlation.}
    \label{fig:fdr}
\end{figure}

For real data, we cannot calculate the FDR as we do not know which links are false. Instead, we generate bootstrap samples of all time points and create perturbed data sets by including the measurements on the entire grid at these time points. We then construct several networks with the same density from these perturbed data sets and finally compute the fraction of differing links between pairs of sampled Pearson correlation networks (Fig. \ref{fig:densdiff}a). With this procedure, we approximate the network distribution induced by the data set (see Section \ref{sec:subsamp_ensembles}). High autocorrelation causes the need for blockwise bootstrapping to receive consistent estimates, as the network variability is increasingly underestimated with increasing autocorrelation. The results should be seen as a conservative preliminary insight into the intrinsic network variability and the amount of unstable edges.
A robust network construction procedure should yield a low fraction of fluctuating links across bootstrap draws. Narrow uncertainty bands indicate that varying weighting of climatic regimes among the bootstrap samples has limited influence on the networks. Observe that in t2m- and pr-networks an alarming fraction of links fluctuate (Fig. \ref{fig:densdiff}a), while networks from smooth variables with long length scale, such as sp and z850, fluctuate less (consistent with Fig. \ref{fig:mad_fdr}).
In contrast to the synthetic data, the curves do not grow monotonically in the sparse regime. Therefore resampled networks may be helpful to choose a maximally robust density, minimizing the fraction of varying edges in the empirical networks. Density up to 0.01 seems to be an appropriate choice for t2m, larger densities dramatically decrease the network robustness. The differing links do not contain short edges (Fig. \ref{fig:densdiff}b), as the correlation values on short edges are consistently large. Longer links heavily depend on the sampled time points and are sensitive to slight perturbations of the correlation estimates. Hence the decision which long links to include should not be based on a single correlation estimate.
Geopotential heights behave differently and become more stable at larger densities as they have a correlation structure with extremely large length scale.

\begin{figure}[H]
    \noindent\includegraphics[width=39pc,angle=0]{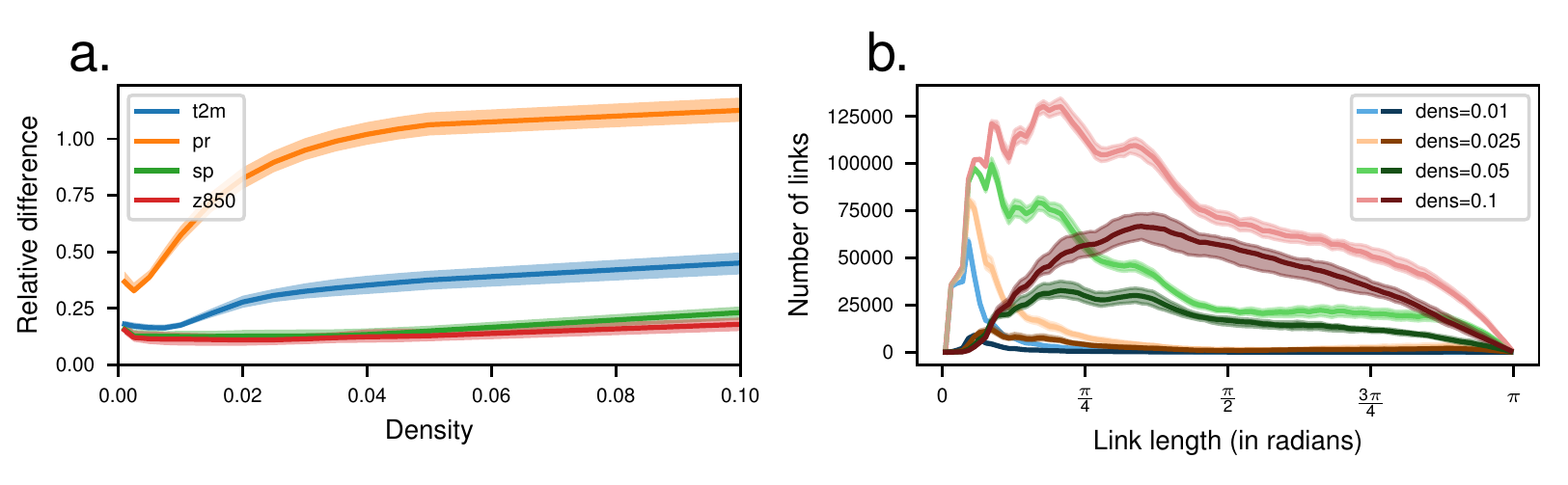}
    \caption{\textbf{Fluctuating edges in real networks.} \textbf{a)} Fraction of differing edges between pairs of empirical Pearson correlation networks from various climatic variables obtained by bootstrapping in time. As we divide by the number of links in one of both compared networks, the fraction of differing edges varies between $0$, when no link differs, and $2$, when all links differ. \textbf{b)} Total link length distribution (light) versus link length distribution of differing edges (dark) between bootstrap networks of t2m. The link length distribution is similar to the one from our MIGRF (cf. Fig. \ref{fig:graphmeasures}c). The short edges do not differ among bootstrap samples, long edges fluctuate heavily.}
    \label{fig:densdiff}
\end{figure}

\textbf{Consequences.} The selection of appropriate similarity measures depends on how much data is available. Mutual information estimators require much more data to yield reliable results than correlation estimators. HSIC shows promising performance in our experiments. It may be worth exploring other alternatives to MI, such as \citet{Romano2018}, in the future. 
Although not particularly well-suited for random field data, the Ledoit-Wolf estimator uniformly improves over naive empirical cross-correlation when estimating weighted Pearson correlation networks. Future work should put more focus on which estimators perform best on meteorological data.

For most climatic observables, we detect severe network variability for long links. In order to quantify structural and link robustness of constructed networks, we need resampling procedures that adequately capture the intrinsic network variability (Section \ref{sec:subsamp_ensembles}). Small network densities yield more robust networks in terms of differing/fluctuating links in resampled networks.

\subsection{Network Measures}
\subsubsection{Extreme Betweenness Values are Unreliable in Sparse Networks}
\label{sec:betw}

\textbf{Problem.} The betweenness centrality of a node $v_k$ is given by the expression $\sum_{i,j\neq k}\frac{\sigma_{i,j}(k)}{\sigma_{i,j}}$, where $\sigma_{i,j}$ is the total number of shortest paths from node $v_i$ to node $v_j$ and $\sigma_{i,j}(k)$ is the number of those paths that pass through $v_k$. In climate networks, high betweenness indicates that a location connects different regions. In temperature-based networks, such locations have been interpreted as key pathways of energy flow \citep{Donges2009}.
When interested in nodes of highest betweenness, it is tempting to construct sparse networks, because the most important points stand out more distinctly. However we find that variability in betweenness also increases drastically when sparsifying the network. \citet{Donges2009} operates exactly in this unreliable regime. Let us explain the influence of sparsity on betweenness in climate networks.

\textbf{Simulation results.} Figure \ref{fig:gauss_betweenness} shows betweenness maps of networks, constructed as in \citet{Donges2009}, from independent draws of our locally correlated, isotropic model. Because of the standard Gaussian grid, randomly fluctuating betweenness `backbones' emerge that form global pathways but do not represent ground truth structure. The density is chosen such that there exist few pathways between the well intraconnected poles. Which exact nodes lie on the few shortest paths between the node-dense polar regions depends on which false links are formed in the region between the poles. Phrased differently, {\em the supposedly important nodes with high betweenness are precisely the ones with false edges and alter between different realizations of the process.} The difference map (Fig.2 of \citealp{Donges2009}) between networks from different data sets shows strikingly similar North-South pathways.
The lattice-like structure makes the sparse ground truth network highly susceptible in terms of betweenness. Since the data generating process is isotropic and the Gaussian grid is symmetric with respect to longitudinal rotations, nodes on the same latitude have equal betweenness values in the ground truth network. The empirical networks consistently show systematically different betweenness distributions. While the maximal betweenness value in the ground truth network is $1.68\cdot 10^{-3}$, the empirical networks have much more pronounced extreme betweenness values of $2.36 \cdot 10^{-2}\pm 8.23\cdot 10^{-4}$. A visualisation of the heavy-tailed betweenness distribution as well as an analysis of Forman curvature can be found in Appendix \ref{sec:betw_forman}. 
Even sparse ground truth networks are highly sensitive to slight perturbations of the grid and network density, as there exist few important pathways connecting different regions in a sparse, locally connected network.

\begin{figure}[H]
    \noindent\includegraphics[width=39pc,angle=0]{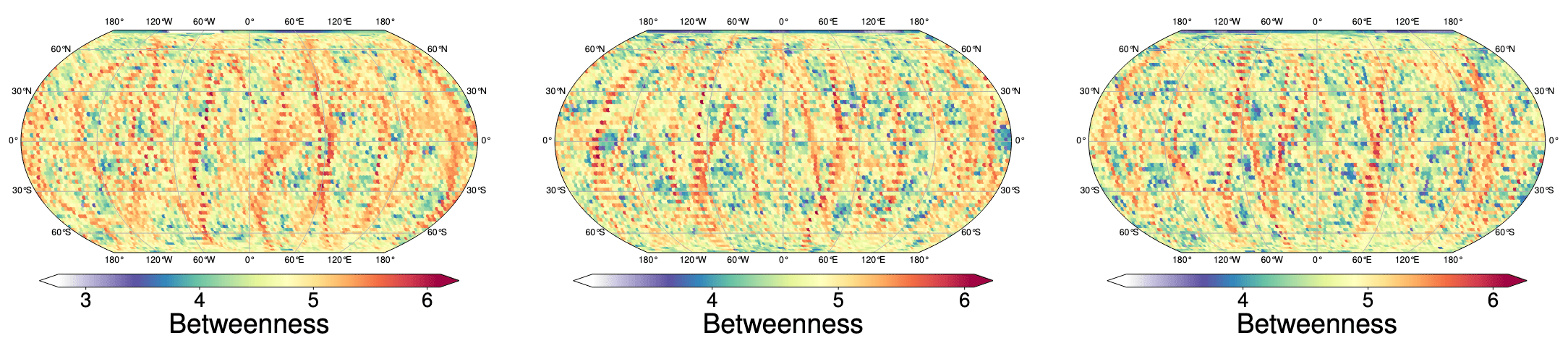}
    \caption{\textbf{Betweenness maps of simulated networks.} Transformed betweenness maps ($\log_{10}(BC+1)$) of empirical networks with density $0.005$ on a Gaussian grid. The maps depict independent realizations of the same data generating process MIRGF$(\nu = 1.5, \ell = 0.2)$. Because of the anisotropic grid, the poles are highly intraconnected. The density is chosen such that there exist few random shortest paths connecting the poles, resulting in pronounced spurious global betweenness pathways, which alter in location and extent among independent realizations of the data and do not represent ground truth structure.}
    \label{fig:gauss_betweenness}
\end{figure}

\begin{figure}[H]
    \noindent\includegraphics[width=39pc,angle=0]{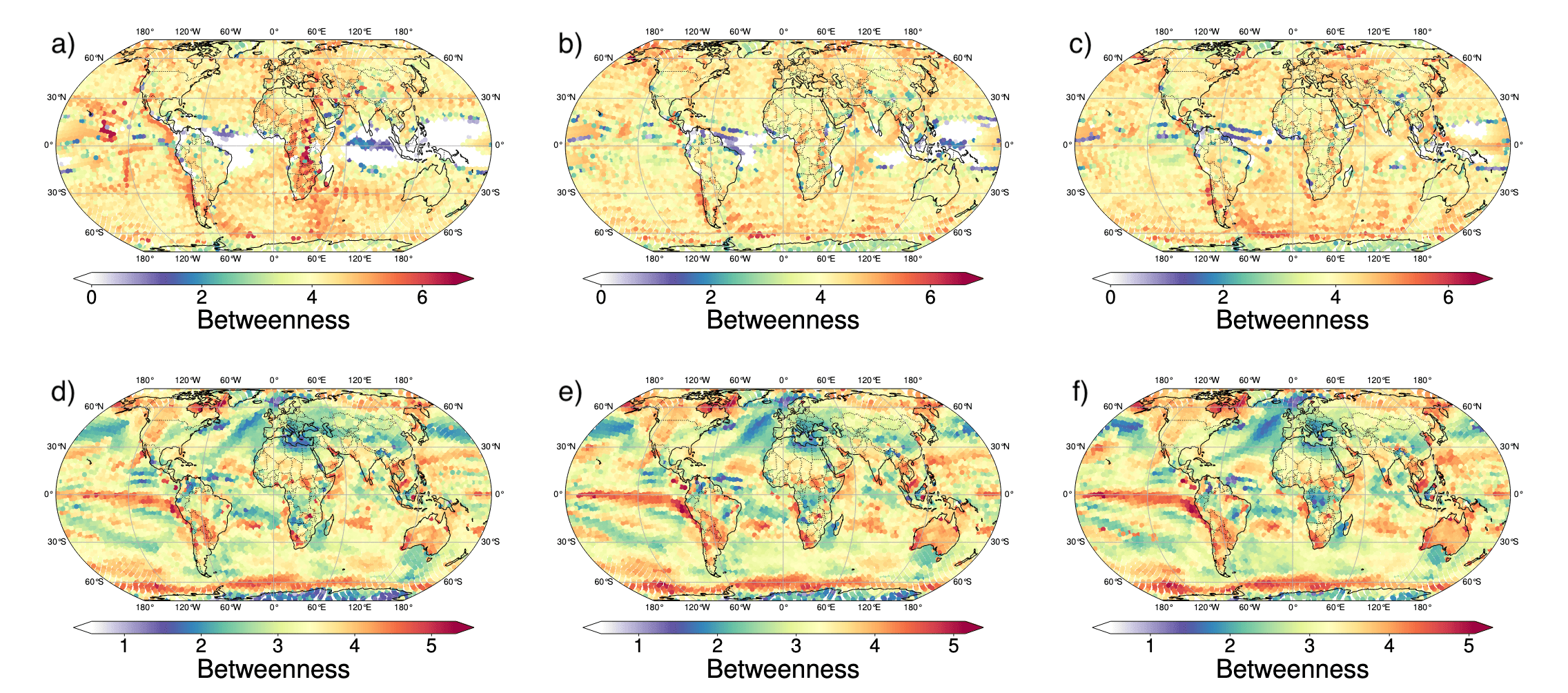}
    \caption{\textbf{Betweenness maps of temperature networks.} Transformed betweenness ($\log_{10}(BC+1)$) for daily t2m in a mutual information network (using the binning estimator) with various densities, as shown in \citep{Donges2009}, but with (asymptotically isotropic) Fekete grid and ERA5 data between $1979$ and $2019$. \textbf{a),b),c)} Network density from left to right: $0.004$, $0.005$, $0.006$. \textbf{d),e),f)} Network density from left to right: $0.08$, $0.1$, $0.12$. Sparse networks fluctuate much more and have few pronounced extreme points. Sparse and dense networks look very different. The results in \citep{Donges2009} also look very different to ours.}
    \label{fig:realbetw}
\end{figure}

Figure \ref{fig:realbetw} shows betweenness maps of networks, constructed from daily t2m as in \citet{Donges2009}, but on the approximately isotropic Fekete grid. The betweenness `backbone' fluctuates more in the sparse than in the dense networks. The maps in Fig. \ref{fig:realbetw} and the ones presented in the original study all look different because betweenness is unstable with respect to grid choice, data set, and network density. 
This raises the question which, if any, map shows a true betweenness `backbone'. 
The dense networks (d)-(f) are more stable than the sparse ones (a)-(c) with respect to network density perturbation. But only after validating that the finite-sample network variability is also low (see Section \ref{sec4}\ref{sec:subsamp_ensembles}) should patterns uncovered by network methods be interpreted with domain knowledge to generate novel insights. Here, a domain expert might point out stable ENSO-like patterns in the Eastern Pacific in the stable dense networks (d)-(f) and get the inspiration to more closely investigate surprising patterns revealed by the betweenness map.

\textbf{Consequences.} Although forming sparse networks yields a better false discovery rate, some network characteristics become extremely sensitive to small perturbations of the networks and their explanatory power diminishes, even in the ground truth networks. For each network measure of interest, the choice of network density constitutes a trade-off between false discovery rate and robustness of the measure. Here too, independent ensemble members can help to identify stable patterns (see Section \ref{sec:subsamp_ensembles}). When a network measure fluctuates too much, as betweenness does in sparse networks, results should not be over-interpreted.

\subsubsection{Empirical Distributions of Network Characteristics are Distorted}
\label{sec:extremevals}

\textbf{Problem.} Here we present further perspectives on systematic empirical distortions of network measures.
%
Most studies focus on the extremal nodes for any network measure, interpreting these as particularly important. Our simulation results show that, under data scarcity, \textbf{random nodes appear spuriously important in the empirical networks, not representing important nodes in the ground truth network}. Several studies have constructed Pearson correlation networks from sliding windows with $2.5^\circ$ and finer resolution \citep{Radebach2013,Hlinka2014,Fan2017,Fan2018,Fan2022}. Our simulation results suggest that the naive correlation estimator and the short time scale are risk factors for false edges and distortions in global measures and extreme values of the networks.

\begin{figure}[H]
    \noindent\includegraphics[width=39pc,angle=0]{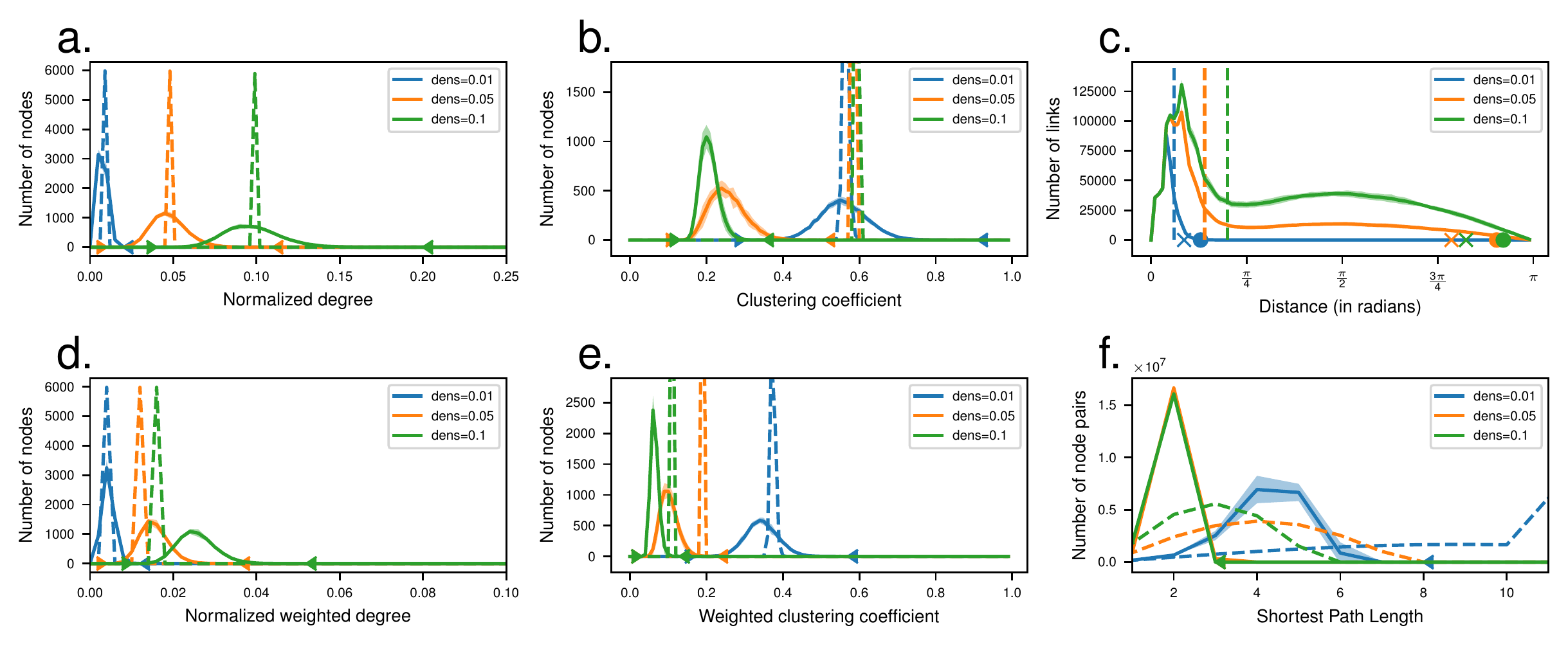}
    \caption{\textbf{Empirical vs ground truth distributions of node/edge measures.} Empirical distributions of node-wise graph characteristics using Pearson correlation for an MIGRF($\nu=0.5$, $\ell=0.2$). From top-left to bottom-right, \textbf{a)} normalized unweighted degree, \textbf{b)} unweighted clustering coefficient, \textbf{c)} unweighted link-length distribution, \textbf{d)} weighted normalized degree, \textbf{e)} weighted clustering coefficient \citep{Onnela2005} and \textbf{f)} unweighted shortest path length. Dashed lines denote the ground truth, solid lines denote the respective distribution in the empirical networks. Triangles denote the empirical extreme values, `x' and circle denote the $95\%$- and $99\%$-quantiles of a distribution averaged between independent realizations. The vertical lines in \textbf{c)} denote the maximal link-length in the ground truth graph; all longer links are false by design. Empirical distributions are more spread out; some measures such as the clustering coefficient or the shortest path length are heavily biased.}
    \label{fig:graphmeasures}
\end{figure}

\textbf{Simulation results.} 
Although all nodes have roughly the same degree/clustering coefficient in the ground truth graph, the observed degree/clustering coefficient distribution is more spread out in the empirical networks (Fig. \ref{fig:graphmeasures}). The random distortions in the empirical networks are similar in type and extent for different independent realizations. Spuriously extreme nodes in the empirical networks vary between independent realizations and do not reflect important or clustered nodes in the ground truth graphs.
While the average unweighted degree is consistent by construction, the weighted degree is systematically upward biased, as more links of low ground truth correlation are available that can be overestimated than links of large ground truth correlation that can be underestimated. The empirical (weighted) clustering coefficient is strongly downward biased as spurious links connect otherwise disconnected regions and bundled connections are not formed between entire neighborhoods. Spurious teleconnections serve as shortcuts in the networks and lead to systematically smaller shortest path lengths. Another network measure related to the clustering coefficient and shortest path lengths is small-worldness. Since both network measures are extremely distorted in the empirical networks, a reliable conclusion about ground-truth small-worldness cannot be drawn from the empirical networks in our setting. A more detailed treatment of small-worldness in spatially extended systems can be found in \citep{Bialonski2010,Hlinka2017}. The conclusions of both studies resemble ours.
The length distribution of the spurious links (longer than the dashed lines in Fig. \ref{fig:graphmeasures}c) behaves as the number of available links at each distance: sinusoidal. This occurs when the corresponding true correlation values are empirically indistinguishable and has also been found in climate networks based on event synchronisation \citep{Boers2019}.
Under large length scales, empirical networks contain less erroneous edges and a more accurate link length distribution up to higher network densities. On the other hand, empirical networks show larger spreads of degree, clustering coefficient and shortest path length distributions as false links occur in bundles (see Section \ref{sec:bundles_clusters}). Under small length scales, bundling behaviour is less pronounced, so that the amount of spurious links averages out, resulting in a more concentrated degree distribution, although more false links occur.

\textbf{Consequences.} When empirical networks are constructed with scarce data, they possess systematically different characteristics compared to the ground truth structure. In our setting, \textbf{distributions of popular node measures are more spread out as well as systematically biased}. These distortions do not become apparent by considering empirical summary statistics based on time series resampling because the empirical behaviour remains consistent between independent repetitions. However, given multiple sufficiently independent network estimates, spurious and ground truth extreme nodes can be distinguished, depending on how systematically they reappear in several networks.
Consequently, both the distribution as well as the extreme values of network measures in empirical networks can primarily be the result of estimation errors and should not be over-interpreted. In particular, whenever the number of formed links scales with the number of available links, large estimation variability can be the cause, so that researchers should take additional efforts to justify the correctness of their network when this link length distribution arises, as in \citet{Boers2019}.

\subsection{Local Correlations Give Rise to Spurious Link Bundles and High-degree Clusters}
\label{sec:bundles_clusters}

\textbf{Problem.} As single false links occur with high probability in estimated networks, \citet{Boers2019} considered teleconnections in a climate network as significant only when a bundle of edges from one region to another is formed. As discussed in Section \ref{sec:fine_grids}, this approach is unreliable when the underlying data is locally correlated, because edges tend to be formed in bundles. Spuriously dense regions in density-threshold graphs are another possible repercussion. Here we provide empirical evidence of spurious bundling behaviour.

\textbf{Simulation results.} Let us first define a link bundle between two locations. Intuitively, we demand that a sufficient portion of edge weight is formed between neighborhoods of both locations. Formally let $B_\varepsilon(v):= \{ u\in V | d(u,v)\le \varepsilon \}$ be the $\varepsilon$-ball around location $v$ and let $A$ be the graph adjacency matrix. We denote the cumulative weights between neighborhoods of $v_i$ and $v_j$ by
\[
W_\eps(v_i,v_j) := \sum_{k,l:\: v_k\in B_\varepsilon(v_i), v_l\in B_\varepsilon(v_j)} |A_{kl}|.
\]
The number of edges between the regions $B_\varepsilon(v_i)$ and $B_\varepsilon(v_j)$ in the complete graph is denoted by $\rho_\varepsilon(v_i,v_j)$. We say that there is a \textbf{$(\varepsilon,c)$-many-to-many link bundle} between $v_i$ and $v_j$ in $(V,E,A)$, if
\[
\frac{W_\eps(v_i,v_j)}{\rho_\varepsilon(v_i,v_j)} \geq c,
\]
for some minimal connectivity $c>0$. Of interest might also be a one-to-many and a locally weighted version of this notion (defined in Appendix \ref{sec:linkbundles}).).

Figure \ref{fig:tele} shows the maximal distance of occuring link bundles (a,c) and the fraction of false links that belong to some bundle (b,d) for various notions of link bundles and for unweighted (a,b) and weighted (c,d) networks. The hyperparameters $\nu=0.5$ and $\ell=0.1$ amount to weak local correlation structure and hence constitute an adversarial choice for bundling behaviour.
There is no unique definition in climate science literature of when an edge constitutes a teleconnection. \citet{Boers2019} calls an edge a teleconnection, when it is longer than $2500$km or $0.12\pi$ radians, \citet{Kittel2021} sets the threshold to $5000$km or $0.25\pi$ radians. Irrespective of the exact distance, long 1-to-many link bundles already arise in unweighted networks with low density, while many-to-many link bundles consistently arise for intermediate network densities. Utilizing edge weights and tuning the minimal connectivity parameter $c$ can reduce the number of spurious long range link bundles by several orders of magnitude, because spuriously included links tend to lie marginally above the threshold and can therefore be distinguished from strong links. However there is no hope to remove all spurious link bundles without removing true bundles as well. Since both true and false links naturally occur in bundles when the data is locally correlated, bundling properties cannot answer questions of significance. Results for other hyperparameters and mutual information are provided in Appendix \ref{sec:more_plots}. 

\begin{figure}[H]
    \noindent\includegraphics[width=39pc,angle=0]{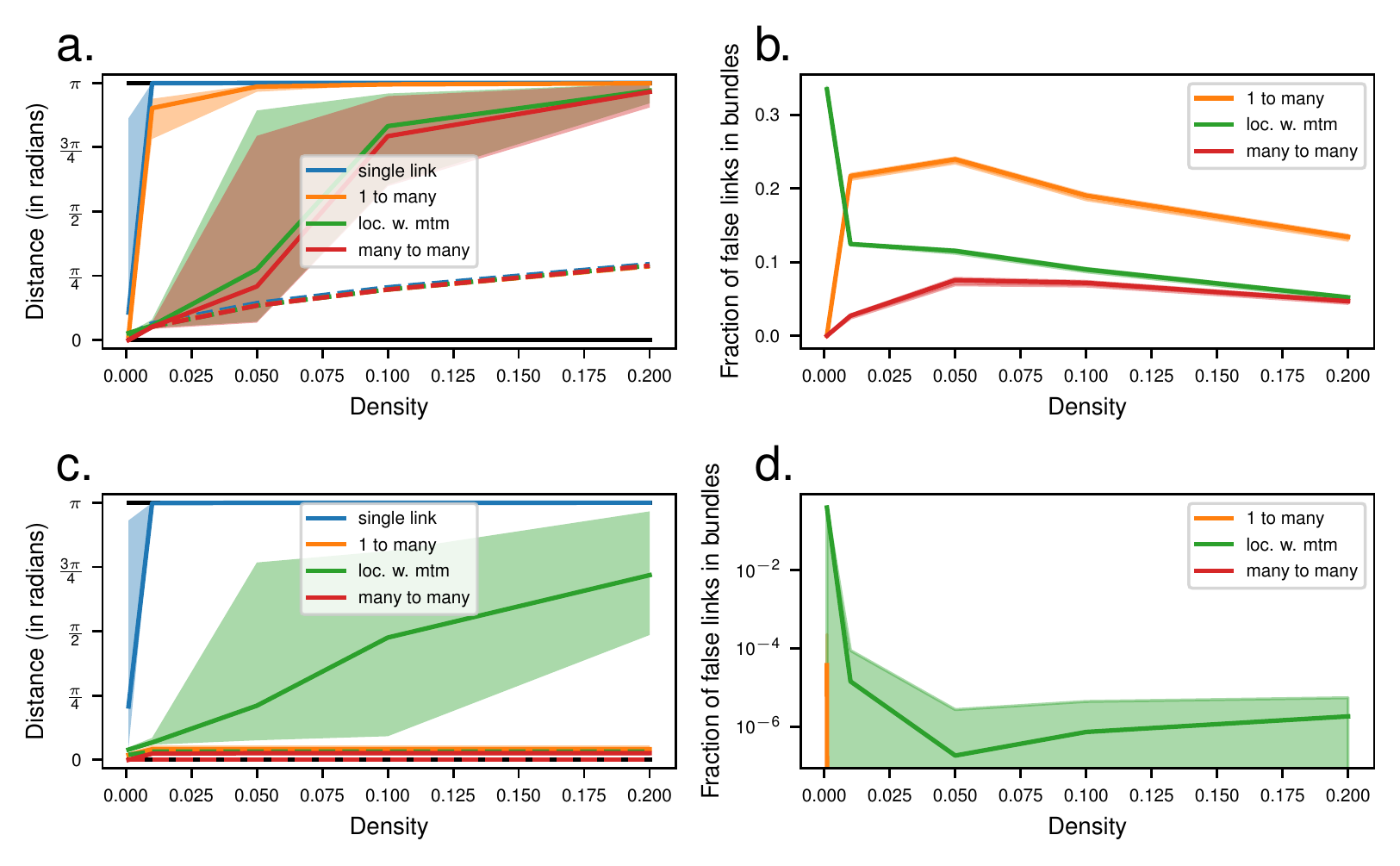}
    \caption{\textbf{Spurious link bundles frequently occur.} \textbf{a), c)} Maximal length for which there exists at least one link bundle in the empirical Pearson network from an MIGRF($\nu=0.5$, $\ell=0.1$) for \textbf{a)} unweighted and \textbf{c)} weighted notions of link bundles. The dashed line denotes the maximal length in the ground truth network. The definition of 1-to-many and locally weighted many-to-many (loc. w. mtm) link bundles can be found in Appendix \ref{sec:linkbundles}. 
    For weighted bundles we choose $c=0.5$, for unweighted 1-to-many bundles $c=0.9$ and for unweighted many-to-many bundles $c=0.8$. The radius of neighborhoods is chosen as $\eps=5^\circ$, which corresponds to roughly $556$km. In our $2.5^\circ$-grid, the $\eps$-balls contain $11.4\pm 1.1$ nodes.\\
    \textbf{b), d)} The fraction of false links that belong to some \textbf{b)} unweighted and \textbf{d)} weighted link bundle among all false links. Same setting as a), c). Many spurious long-range link bundles occur with high probability when the data is locally correlated, which is typically the case for spatio-temporal data. Strong teleconnections appear less frequently, so using edge weights can be helpful. 
    }
    \label{fig:tele}
\end{figure}

While for the FDR (Fig. \ref{fig:mad_fdr}a) large smoothness and length scale have a positive impact (because the random field varies less in total), the links that are being formed tend to occur in bundles. The essential distributional parameter for sparse networks is the smoothness of the random field, as mostly short links are formed. Varying the length scale has larger impact for denser graphs as it determines the radius of spurious link bundles and the distance/network density at which ground truth correlations become empirically indistinguishable from $0$.
In order to measure whether there are regions of spuriously high degree due to the dependence between nodes, we find the $\eps$-ball of maximal average degree (MAD) among all $\eps$-balls $ B_\eps(v_i)$ (Fig. \ref{fig:mad_fdr}b). Then we compute the same quantity for randomly permuted degree values so that nodes with spuriously high degree are not spatially clustered anymore. The MAD-values of the empirical networks are consistently larger than the MAD-values of shuffled nodes, so a clustering of high-degree nodes occurs irrespective of the hyperparameters of the random field. The pronounced bundling behaviour for larger length scales is reflected in larger MAD-values.

For real data, observe a strikingly monotonous relationship between the average local correlation and the fraction of long links (longer than $5000$km or $0.25\pi$ radians) that belong to some link bundle (Fig. \ref{fig:realbundles}). Most differing links do not belong to a bundle, but under large local correlations the number of fluctuating long links in bundles can become non-negligible. Also observe that our simulated data, at a given local correlation level, shows a tendency to underestimate the fraction of links in bundles, indicating the existence of true bundled teleconnections in climatic variables.

\textbf{Consequences.} We have seen that bundled connections do not necessarily represent ground truth structure but can occur spuriously when the similarity estimates are locally correlated. Even without teleconnections, random regions can appear spuriously dense. These experiments also explain the distortion of the degree distribution in Fig. \ref{fig:graphmeasures}. Using the edge weights can be helpful to distinguish strong from weak connections, as spurious connections tend to lie marginally above the threshold.
We conclude that  when the data is locally correlated, questions of edge significance can not be easily addressed by considering bundling behaviour. Only bundling behaviour that exceeds the effects of localized correlation structure can be considered significant. Given multiple sufficiently independent empirical networks, only ground truth connections would reappear in many networks with high probability.

 \begin{figure}[H]
    \noindent\includegraphics[width=39pc,angle=0]{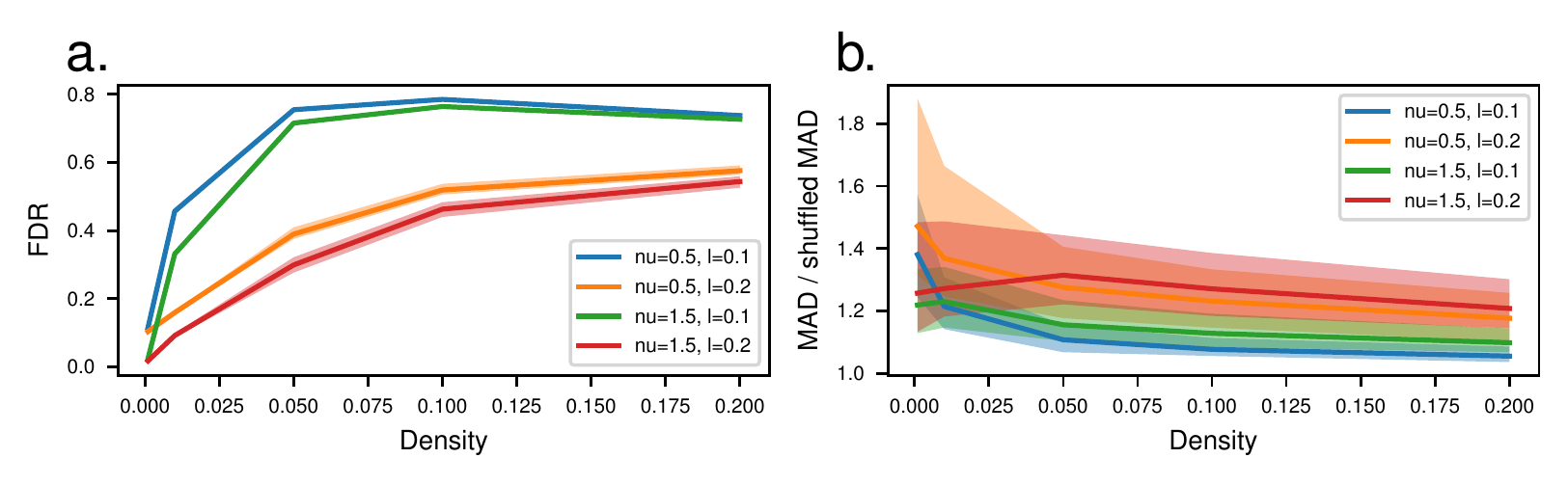}
    \caption{\textbf{Influence of random field parameters.} \textbf{a)} The fraction of false links (FDR) in empirical Pearson correlation networks for various hyperparameter choices of the random field. Sparse networks are more accurate in this sense (low FDR). For sparse networks the smoothness is the essential parameter, while for denser networks only the length scale matters. \textbf{b)} Maximal average degree (MAD) in $\eps$-balls divided by the same quantity under shuffled nodes in unweighted networks for various hyperparameter choices of the random field. Values above 1 indicate that high degree nodes tend to be clustered. The weighted equivalent looks very similar.}
    \label{fig:mad_fdr}
\end{figure}

\begin{figure}[H]
    \noindent\includegraphics[width=39pc,angle=0]{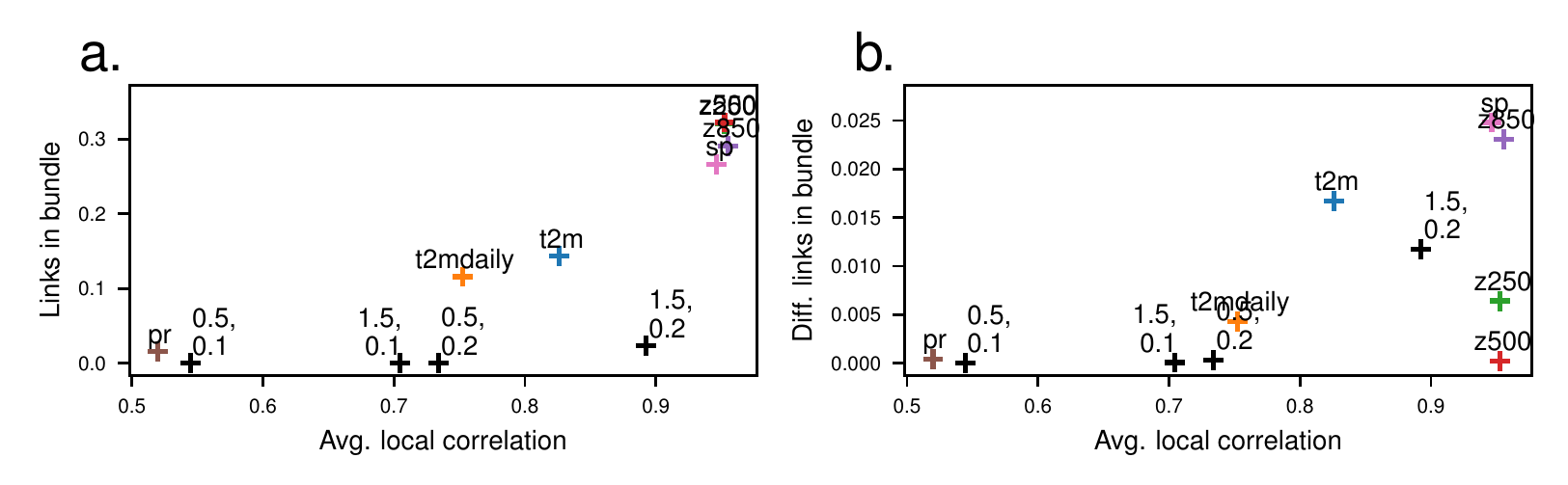}
    \caption{\textbf{Bundling behaviour against local correlation.} \textbf{a)} Fraction of long links (longer than $5000$km or $0.25\pi$ radians) that belong to a many-to-many link bundle with $c=0.8$ as a function of the average correlation in such a ball in Pearson correlation networks of density $0.05$. A tuple $a,b$ with black `x' indicates simulated data with $\nu=a$ and $\ell=b$. With larger average local correlation, a larger fraction of long links is part of some bundle. Real data shows even more bundling behaviour than our simulated data; this indicates ground truth teleconnections. \textbf{b)} Fraction of how many long links that differ between bootstrap samples belong to some bundle. Less differing links belong to some bundle, but their number is not negligible, especially under strong local correlation structure. Again, geopotential heights behave differently. Their correlation structure is so wide-spread that edge formation is highly codependent, but not very localized.}
    \label{fig:realbundles}
\end{figure}

\subsection{Anisotropy}
\label{sec:anisotropy}

\subsubsection{Anisotropic Autocorrelation on the Nodes Causes Biased Empirical Degree}
\label{sec:aniso_ar}

\textbf{Problem.} 
%
Given two nodes with lag-1 autocorrelation $\alpha\in (0,1)$ and $\beta\in (0,1)$, respectively, the asymptotic variance $\sigma^2_{\alpha,\beta}$ of their empirical Pearson correlation scales as \begin{align}
    \sigma^2_{\alpha,\beta}:=1+2\frac{\alpha\beta}{1-\alpha\beta},\label{eq:varscaling}
\end{align}
which explodes for $\alpha,\beta\to 1$ (detailed explanation in Appendix \ref{sec:effective_samples}). Conversely, time length $n$ is effectively only worth $n/\sigma^2_{\alpha,\beta}$ independent observations. 
The same principle applies to other estimators. Under anisotropic autocorrelation on the nodes, similarity estimates have different variability depending on the edge. While unweighted density-threshold networks are not biased by isotropic autocorrelation, anisotropic estimation variability introduces biases during the estimation procedure. \citet{Palus2011} have already observed such biases in climate networks. Here we explain them from a statistical perspective using our null model.

In practice different locations have different autocorrelation patterns. Due to higher effect heat capacity, temperature over oceans has higher autocorrelation than over land \citep{Eichner2003,Vallis2011}. \citet{Guez2014} argue that disagreement between their similarity measures is primarily caused by high autocorrelation. Our simulation results suggest that the cause of this disagreement might more fundamentally be estimation errors that vary between similarity measures.

\begin{figure}[H]
    \noindent\includegraphics[width=39pc,angle=0]{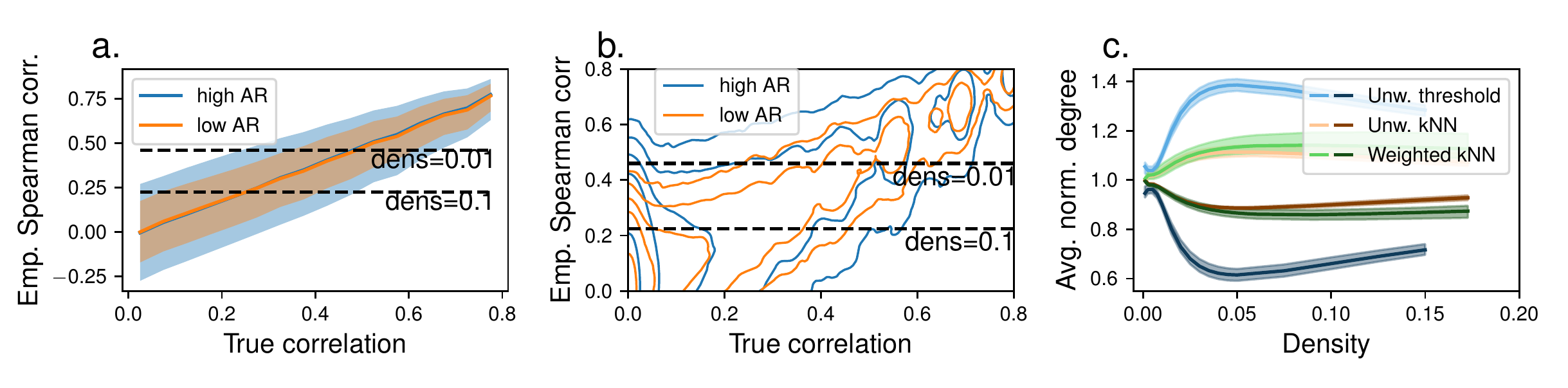}
    \caption{\textbf{Anisotropic variability of empirical estimates induces degree bias in empirical networks.} We initialize a random half of the points with low lag-1 autocorrelation of $0.2$ and the other half with high lag-1 autocorrelation of $0.7$. Then we analyse the distribution of Spearman correlation estimates for strongly and weakly autocorrelated nodes in an empirical network from a MIGRF($\nu=1.5$, $\ell=0.2$).\\
    \textbf{a)} $5\%$-, $50\%$- and $95\%$-quantiles of the empirical Spearman correlations given the true correlation values. High autocorrelation does not introduce a bias in the correlation estimates, but leads to larger variance.\\
    \textbf{b)} Kernel density estimate of the edge distributions with high and low autocorrelation respectively, at the levels $0.001,0.01, 0.05, 0.1, 0.5$. Since most true correlation values are small, a larger variance causes a larger part of the edge distribution of the highly autocorrelated nodes to lie above the thresholds. This leads to a higher average degree for nodes with correlation estimates of higher variance.\\
    \textbf{c)} Average degrees for threshold (blue), unweighted kNN (orange) and weighted kNN graphs (green) for nodes of high (light colors) and low (dark colors) autocorrelation normalized by the average degree in the network. Although kNN graphs cannot eliminate this bias, they can reduce it.}
    \label{fig:aniso_ar}
\end{figure}

\textbf{Simulation results.} We simulate anisotropic autocorrelation (Fig. \ref{fig:aniso_ar}) by employing our VAR(1)-model (Appendix \ref{sec:time_dep}). We initialize a random half of the points with low lag-1 autocorrelation of $0.2$ and the other half with high lag-1 autocorrelation of $0.7$.
On average, empirical Spearman correlation estimates do not depend on the autocorrelation of adjacent nodes (Fig. \ref{fig:aniso_ar}a). But the increased variance on highly autocorrelated nodes leads to both an increase of spuriously low similarity estimates for edges with high ground truth correlation, as well as more spuriously high estimates on edges with small ground truth correlation. Since most ground truth correlations are small (Fig. \ref{fig:aniso_ar}b), overall the number of high similarity values increases. Thus nodes of higher autocorrelation show an increased average degree in threshold graphs (Fig. \ref{fig:aniso_ar}c).

In real climate networks, the nodes of highest degree consistently have high lag-1 autocorrelation (Fig. \ref{fig:t2m_ar}). Together with our simulations, this suggests that anisotropic autocorrelation has non-negligible spurious effects on the networks. Recalling Eq. (\ref{eq:varscaling}), forming false links between highly autocorrelated nodes is much more likely than between nodes of small or intermediate autocorrelation. Hence both false edges at nodes with high autocorrelation and missing edges at nodes with low autocorrelation have to be expected when some nodes attain autocorrelation values close to one, as for t2m.

\begin{figure}[H]
    \noindent\includegraphics[width=39pc,angle=0]{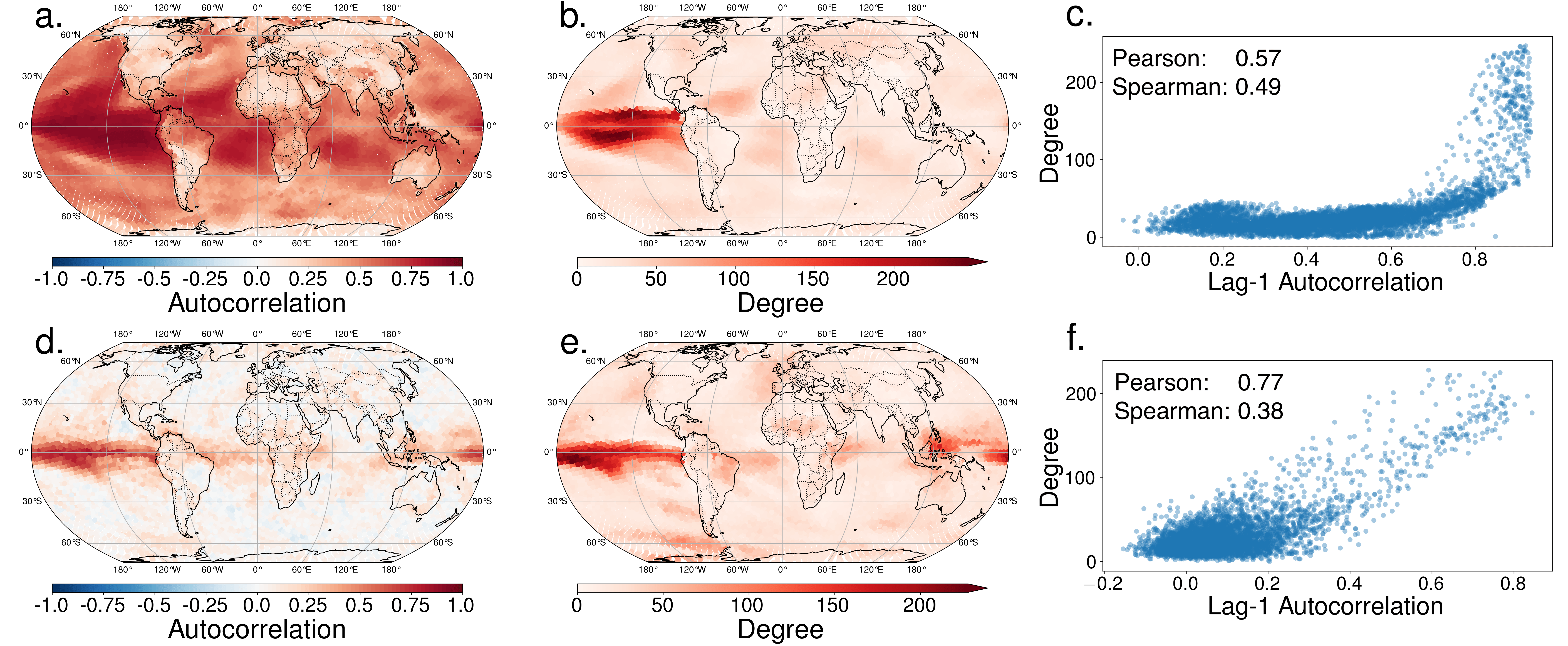}
    \caption{\textbf{Nodes of highest degree have high autocorrelation.} \textbf{a), c)} The autocorrelation on the node, and \textbf{b), e)} the degree in a Pearson correlation network of density $0.005$ for \textbf{a), b)} monthly t2m and \textbf{d), e)} pr. \textbf{c), f)} the degree of each grid point as a function of lag-1 autocorrelation for \textbf{c)} t2m and \textbf{f)} pr. Observe exploding degrees for autocorrelation above $0.7$ for t2m, as predicted by the variance scaling formula (\ref{eq:varscaling}). For pr degrees increase even earlier at autocorrelation $0.2$. This suggests that the most connected nodes have many spurious edges, induced through high autocorrelation.}
 \label{fig:t2m_ar}
\end{figure}

\textbf{Consequences.} Under large isotropic autocorrelation, density-threshold networks have an increased variability but no degree bias. When the variability differs across locations, nodes with high variability receive more false edges than nodes with informative time series. State-of-the-art corrections are discussed in the next section.
Using $k$-nearest neighbor (kNN) graphs prevents disregarding weakly autocorrelated locations. In kNN graphs, each node forms an edge to the $k$ nodes with highest similarity. Although highly autocorrelated locations may still have more spuriously high empirical similarity values, weakly autocorrelated points attain more similar importance in terms of degree in unweighted as well as weighted kNN graphs.


\subsubsection{Fourier Transform-based Reshuffling Reverses the Autocorrelation-Induced Degree Bias in Sparse Networks}
\label{sec:resampling}

\textbf{Problem.} Instead of constructing density-threshold networks some studies only include links that are significant with respect to similarity values from reshuffled data \citep{Boers2014,Deza2015,Boers2019}. For this purpose, the time series on each node are shuffled independently multiple times and similarity values between these shuffled time series are calculated to determine the internal variability of the similarity estimates on each edge. High quantiles of this edge-wise baseline distribution of similarity estimates impose restrictive thresholds above which ground truth dependence is likely. If the quantile and variance estimates are themselves noisy, yet one more source of randomness is introduced into the network estimation procedure.

\textbf{Simulation results.} We follow the popular approach to construct density-threshold networks not from the correlation estimates $\hat{S}_{ij}$ directly but from z-scores $\frac{\hat{S}_{ij}-\hat{\mu}_{ij}^0}{\hat{\sigma}_{ij}^0}$, where the edge-wise mean $\hat{\mu}_{ij}^0$ and variance $\hat{\sigma}_{ij}^0$ are based on different reshuffling procedures. We compare completely random reshuffling of the time series, as in e.g. \citet{Fan2022}, and reshuffling with the iterative amplitude adjusted Fourier transform (IAAFT) \citep{Schreiber1996}, as proposed by \citet{Palus2011}. The IAAFT algorithm was developed by \citet{Schreiber1996} to generate phase-randomized surrogate time series, which share their amplitude distribution and power spectrum with the original time series. We also construct `quantile networks' by including edges for which the empirical correlation value exceeds a high quantile of the edge-wise baseline distribution. 
In order to measure the effect of autocorrelation on empirical correlation estimates, we simulate independent pairs of Gaussian time series of length $n=100$ with varying autocorrelation and calculate the $95\%$-quantile of empirical Pearson correlation (Fig. \ref{fig:resampling}a). Naive reshuffling of the time series (blue) produces a baseline distribution of correlation estimates that does not adapt to the increased estimation variance (black line) under high autocorrelation, resulting in a too permissive threshold and uncalibrated uncertainty. The IAAFT-based quantile estimates adapt to the increased variability but contain large variance between independent realizations. 
To measure impacts of anisotropic estimation variability on entire networks, we simulate spatially anisotropic autocorrelation as in the previous section. Since completely random reshuffling produces the same $\hat{\mu}_{ij}^0$ and $\hat{\sigma}_{ij}^0$ for every edge, the unweighted density threshold network of z-scores from random reshuffling exactly coincides with the unweighted density threshold network from the original estimates $\hat{S}_{ij}$. The IAAFT surrogates remove the autocorrelation-induced degree bias for large densities (Fig. \ref{fig:resampling}b). For sparse IAAFT z-score networks the autocorrelation-induced degree bias is reversed: Because the variance estimates $\hat{\sigma}_{ij}^0$ of nodes with high autocorrelation are systematically larger (Appendix \ref{sec:iaaft_distrib}), the highest z-scores are formed for nodes with low autocorrelation. Since no edges are formed between nodes of high autocorrelation, the fraction of false links explodes in sparse IAAFT z-score networks (Fig. \ref{fig:resampling}c). The IAAFT-based z-score networks fulfill the objective of only forming edges with small estimation variance, but then the resulting network does not represent the spatial ground truth correlation structure. Quantile networks perform slightly better than z-score networks given the same density, but there is no reasonable significance value $\le 0.999$ that achieves a network sparsity necessary for minimal FDR.


In the following way, perfect quantile estimates allow us to determine network densities that lead to empirical networks containing few false edges. Under isotropic autocorrelation, perfect quantile estimates induce a common threshold on the entire network. Applying this threshold to the ground truth correlation matrix induces a density in the ground truth network. For our range of hyperparameters, the densities, induced by the $0.95$-quantile, range from $0.007$ (for $\nu=0.5$, $\ell=0.1$) to $0.032$ (for $\nu=1.5$, $\ell=0.2$) without autocorrelation and from $0.001$ to $0.01$ given isotropic autocorrelation of $0.9$. Choosing larger network densities leads to unreliable empirical networks, as ground truth correlations of longer links are not empirically distinguishable from $0$ with sufficient certainty. Generally, empirical significance-based networks have a larger density than they should, because the number of spuriously high correlation estimates exceeds the number of spuriously low ones.

\begin{figure}[H]
    \noindent\includegraphics[width=39pc,angle=0]{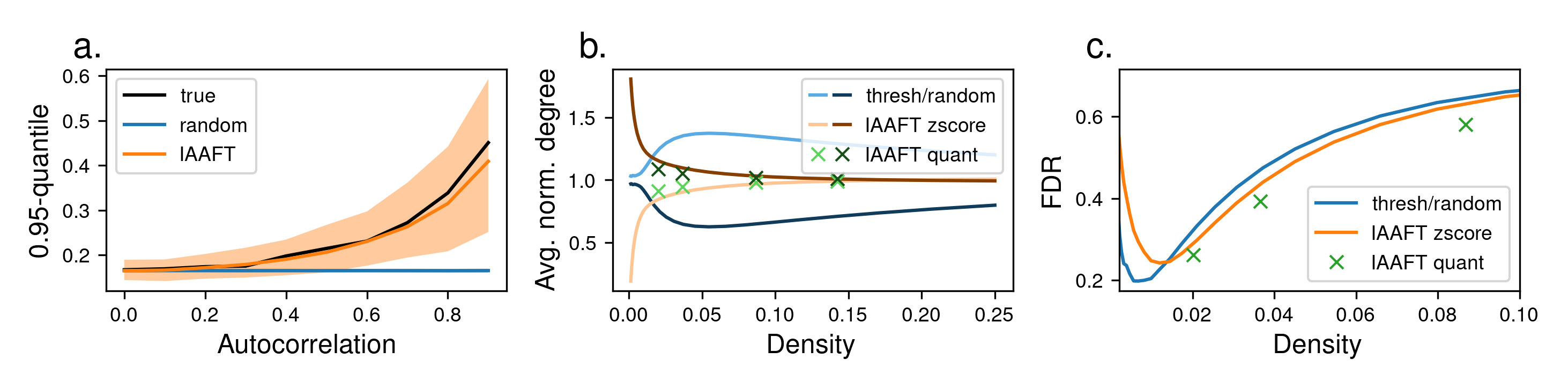}
    \caption{\textbf{Effects of (anisotropic) autocorrelation on significance-based networks.} \textbf{a)} Compares the true 0.95-quantile of empirical correlation values (black) (cf. Eq. (\ref{eq:varscaling})) with the 0.95-quantile obtained from naive node-wise shuffling (blue) and IAAFT node-wise shuffling (orange)
    on a single edge of $0$ ground truth correlation. We therefore calculate 10000 shuffles of 1000 pairs of independent Gaussian time series of length $n=100$ for each autocorrelation value. The growth of quantiles under high autocorrelation is not detected by the naive shuffling estimates. The IAAFT-based procedure detects increased variance but introduces another large source of variance between edge estimates.\\
    \textbf{b)} Average normalized degrees of nodes with high (light colors) and low (dark colors) autocorrelation for unweighted threshold/z-scores from random reshuffling (blue) versus networks from IAAFT-based z-scores (orange) versus quantile networks from IAAFT surrogates (green) with density determined by the quantiles $0.9, 0.95, 0.99, 0.999$. Z-scores from random reshuffling induce the original density-threshold network. The IAAFT surrogates correct the degree bias for large densities, but only form edges between nodes with low autocorrelation in sparse networks.\\
    \textbf{c)} False discovery rates for unweighted threshold/z-scores from random reshuffling (blue) versus IAAFT-based z-scores (orange) versus IAAFT-based quantile networks (green). IAAFT-based z-scores form many false edges in sparse networks. Quantile networks outperform z-score networks given the same density, but do not reach sufficient sparsity.}
    \label{fig:resampling}
\end{figure}

\textbf{Consequences.} In settings of low autocorrelation, completely random reshuffling yields reliable estimates of empirical correlation quantiles, resulting in a controlled false discovery rate. But it is not able to detect anisotropic autocorrelation and can therefore not correct autocorrelation-induced degree bias. The IAAFT-based empirical networks correct this bias in dense networks, but are very biased towards edges with low estimation variance in sparse z-score networks, which results in many false links. The empirical density of significance-based networks only yields an upper bound on a desirable network density. Since \textbf{no considered network construction technique reduces the variability in the edge estimates}, they cannot vastly improve over density-threshold networks.

\subsubsection{Anisotropic Noise Levels on the Nodes Cause Nodes to be Disconnected}
\label{sec:aniso_noise}

\textbf{Problem.} Observational data is generally affected by measurement errors or other sources of noise. Under isotropic additive white noise, variance in the graph construction increases (see Appendix \ref{sec:noisy_meas}). 
Even worse, anisotropic noise levels crucially distort how well nodes are connected in the graph.

A central difficulty in recovering ground truth structure is a distinction of which part of the noise is inherent to the dynamical system ({aleatoric} noise) and which part could be reduced through more sophisticated measurement, preprocessing and estimation procedures ({epistemic} noise). While aleatoric noise affects the ground truth networks and can be seen as an offset of the ground truth correlation function, everything else is an empirical distortion. Nodes over land are commonly less connected in climate networks \citep{Donges2009b} because the underlying distributional characteristics differ across sea and different geological conditions over land. This distributional difference is at least partially aleatoric. Varying availability and reliability of measurements, on the other hand, induce epistemic noise\footnote{In reanalysis data sets, data over nodes in regions of high measurement density can be extrapolated with higher certainty. The density of weather stations in the US or Europe for example is much higher than in parts of Africa or South America. The effort of estimating the measurement/extrapolation error in each node could alleviate the effects of an anisotropic data collection and extrapolation process. Anisotropic measurement/extrapolation noise remains to distort the constructed climate networks and efforts should be made to gather more reliable measurements in neglected regions (cf. \href{https://www.arcgis.com/apps/mapviewer/index.html?layers=c3cbaceff97544a1a4df93674818b012}{overview of WMO weather stations}\citep{wmoweatherstations}).}. In cases where we acknowledge that we cannot satisfactorily judge how much our data is affected by epistemic noise, a conservative approach is to reduce the effects of anisotropy in the network construction. kNN graphs may offer a useful inductive bias in such uncertain settings.

\begin{figure}[H]
    \noindent\includegraphics[width=39pc,angle=0]{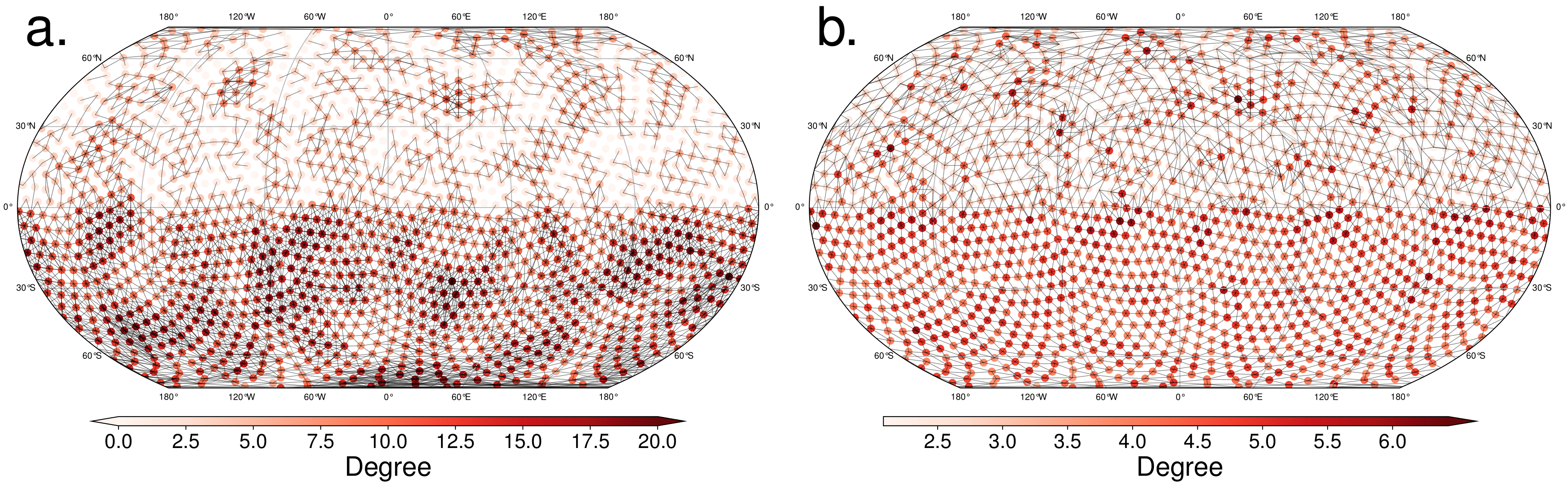}
    \caption{\textbf{Threshold vs kNN networks given anisotropic noise levels.} \textbf{a)} Unweighted threshold graph and \textbf{b)} weighted kNN graph with approximately $0.005$ of possible edges formed from a MIGRF$(\nu=1.5$, $\ell=0.2$) and $5^\circ$ resolution with additive $0.7\cdot N(0,1)$-white noise on the Northern hemisphere. Higher measurement noise on the Northern hemisphere leads to smaller correlation values. These nodes are less connected in the empirical as well as ground truth threshold graph. A kNN graph over the same data ensures similar total connectedness. Nodes on the Northern and Southern hemisphere can still be differentiated in terms of weighted degree. Additionally, spurious high-degree clusters in the empirical threshold graph are not present in the weighted kNN graph.}
    \label{fig:aniso_noise}
\end{figure}

\textbf{Simulation results.} By adding white noise on the Northern hemisphere, we decrease the population correlation for these nodes. As a result we find that (especially sparse) threshold graphs mostly form edges on the nodes with less noise (Fig. \ref{fig:aniso_noise}). To represent all nodes equally in the graph, we propose to use kNN graphs instead. By using weighted edges, the spectrum of node and link importance in terms of weighted degree does not get lost.

\textbf{Consequences.} When the available data is affected by epistemic noise, the connectivity structure in the network is spuriously altered. Effects of anisotropic noise on the empirical networks can be reduced by using kNN graphs.
When the ground truth correlations are higher in some regions than in others (anisotropic aleatoric noise), kNN graphs can also be more informative because weakly correlated nodes are not well represented in density-threshold networks. kNN graphs pose a different inductive bias, which may be useful to detect different patterns. Although ground truth kNN graphs severely differ from ground truth density-threshold networks in anisotropic settings, given useful weights, they have shown to be useful and robust in machine learning applications \citep{Luxburg2007}, while not sacrificing interpretability.

\subsubsection{Ground Truth Networks on Anisotropic Grids}

\textbf{Problem.} Anisotropic grids usually introduce biases in the networks that are not intentional, so that differing node connectivity does not reflect differing correlation structure in the data. Given an anisotropic grid, the nodes will have unequal characteristics in the ground truth network under isotropic correlation structure. It is well known that a regular Gaussian grid is geometrically undesirable due to its two singularities at the poles. Area weighting \citep{Heitzig2011} becomes crucial to correct the distortions in the network. Another effect of anisotropy does not stem from anisotropic grid choice but geographical reality. If we consider an isotropic field with monotonically decaying correlation function on an approximately isotropic grid only defined over oceans, then the nodes in the population network will not be isotropic but encode geometric information about the distribution of land and sea across the earth (Fig. \ref{fig:lsm}). For example sea surface temperatures are only defined over oceans.

\textbf{Simulation results.} The ground truth networks constructed from monotonically decaying isotropic correlation structure simply consist of the shortest possible links. The anisotropic distribution of grid points introduces a bias to the networks that is visible in various network measures. For example, points on paths connecting different regions and points in geometric bottlenecks show higher betweenness values in sparse networks, points with large uninterrupted surrounding show higher degree and points in inlets show larger clustering coefficient, because neighbors towards similar directions are often close to each other and thus also connected. The network density functions as a scale parameter similar to a band width in kernel density estimation, since the connection radius increases with network density.

\begin{figure}[H]
    \noindent\includegraphics[width=39pc,angle=0]{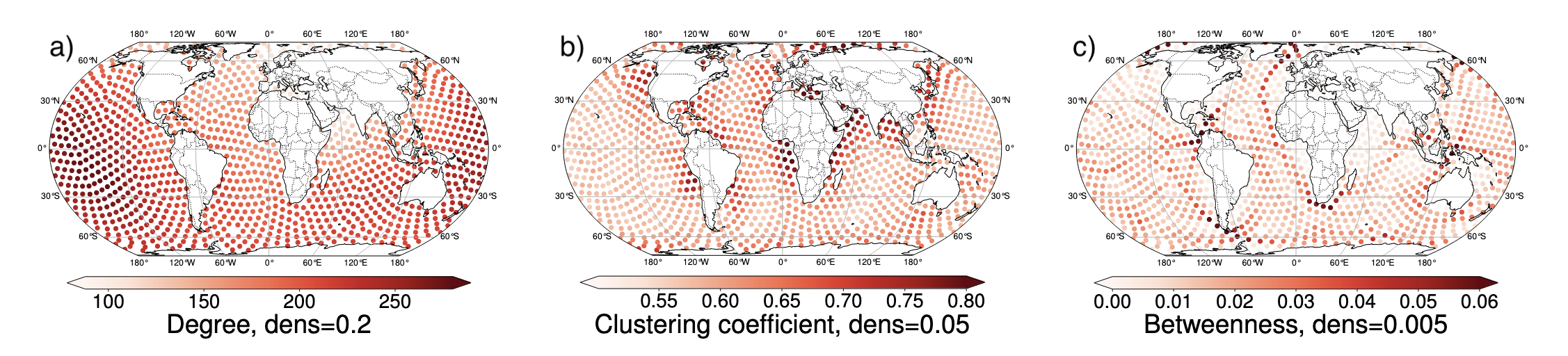}
    \caption{\textbf{Anisotropic grids induce anisotropic node characteristics.} Node characteristics of ground truth networks from monotonically decaying isotropic correlation structure. \textbf{a)} Degree, \textbf{b)} clustering coefficient and \textbf{c)} betweenness for unweighted density-threshold graphs with density $0.2$, $0.05$, $0.005$ from left to right.}
    \label{fig:lsm}
\end{figure}

\textbf{Consequences.} Estimated networks suggest misleading conclusions when false edges distort their characteristics. Even the ground truth network is a result of many design decisions that can lead to prominent behaviour, readily missinterpreted when the its cause is not correctly identified. For betweenness, even the ground truth values are very sensitive to small variations in network density. Even on ground truth, conclusions are not necessarily robust. Boundary correction \citep{Rheinwalt2012} has been proposed for networks that do not cover the entire Earth. A similar correction, using locally connected networks, could be proposed to remove the influence of the distribution of land and sea across the Earth to quantify anisotropic correlation behaviour.

\section{Assessing Significance from Network Ensembles}
\label{sec4}


In practice, researchers are usually confronted with data sets of limited size from an unknown distribution. Based on a single network constructed with state-of-the-art climate network techniques, they cannot judge how many edges are included or excluded because of estimation errors.  Given time length $n$ and number of grid points $p$, the regime of `small sample size', where the observed networks significantly differ from the ground truth, can mean any order of magnitude for $n$ and any ratio $n/p$, depending on dynamics of the spatio-temporal system, measurement error, the employed estimator and subsequent network construction and evaluation steps. This makes general rules of thumb prohibitive, and solid uncertainty estimation based on unrestrictive assumptions crucial for the value of the study.
Constructing networks with various similarity measures, datasets, resolutions and network construction steps (see e.g. \citet{Radebach2013}) can offer qualitative reassurance that observed patterns do not just occur under a specific setting. Significance tests offer a more quantitative approach.
Here, we first discuss shortcomings of common procedures to quantify significance in Section \ref{sec:resampling_practice}, and then offer a new probabilistic framework in Section \ref{sec:subsamp_ensembles}, that addresses these shortcomings.

\subsection{Resampling in Current Practice}
\label{sec:resampling_practice}

The usual approach to quantify the significance of certain findings such as hubs, pathways or teleconnections is to construct an ensemble of networks that share certain aspects of the originally constructed network while randomizing with respect to everything else through reshuffling. The effective null hypothesis of such a permutation test (also called surrogate test) is the (limit) probability distribution over the networks that the ensemble induces. Needless to say, any permutation test can only be as meaningful as its effective null hypothesis.

All previously applied reshuffling approaches for climate networks that have been reported in the literature can be categorized into two types. Either reshuffling is directly performed on the edges to recover, for example, the original degree sequence or the degree sequence and link length distribution (\citealp{Wiedermann2015}, GeoModel II), or the time series are shuffled node-wise to preserve the marginal time-series dynamics with methods such as the iterative amplitude-adjusted Fourier transform \citep{Schreiber1996}. In the latter case, the ensemble networks are then constructed from the shuffled dataset.

\textbf{Node-wise reshuffling.} Whenever researchers have performed permutation tests that recover marginal time-series dynamics, these tests have disregarded the spatial distribution of the data completely. The nodes are assumed to be independent, so that the typical localized correlation structure, which results in a link-length distribution of predominantly short links, is replaced by a uniform one (Fig. \ref{fig:badresampling}a). Since the task is to construct \textit{spatial} networks, such an ensemble is structurally unrealistic and does not induce a physically meaningful network distribution.

\textbf{Edge reshuffling.} Whenever fixing concrete network characteristics to be preserved, a preliminary question has to be addressed: \textbf{Which spatial as well as temporal dynamics of the system or network at hand need to be preserved by the ensemble?}
The authors of the influential paper \citep{Donges2009} use a permutation test with preserved degree sequence. But what is the physical meaning of the observed degree sequence?
One consistently reappearing feature of the underlying physical system is the localized correlation structure, which results in a link-length distribution of predominantly short links. As for node-wise reshuffling, this link-length distribution is destroyed, being replaced by a sinusoidal one.
As a consequence of this inaccurate ensemble distribution, the authors interpret the property that the nodes of highest betweenness show degrees below average in the original network as significant behaviour. In contrast, we have shown in Fig. \ref{fig:gauss_betweenness} that this property is a bias that is introduced through the Gaussian grid and the lattice-like connectivity behaviour of the original network. The original connectivity structure gets destroyed by uniform edge reshuffling, hence the ensemble members have different betweenness properties.
We have seen that betweenness is a highly unstable measure. An indication for robustness of the discovered betweenness `backbone' would be if it consistently reappeared for various subsets of the data, as well as for many network densities and similarity estimators. Since only a single instance of the climate network is presented, it is not clear if the presented backbone appears by chance.

\begin{figure}[H]
    \noindent\includegraphics[width=39pc,angle=0]{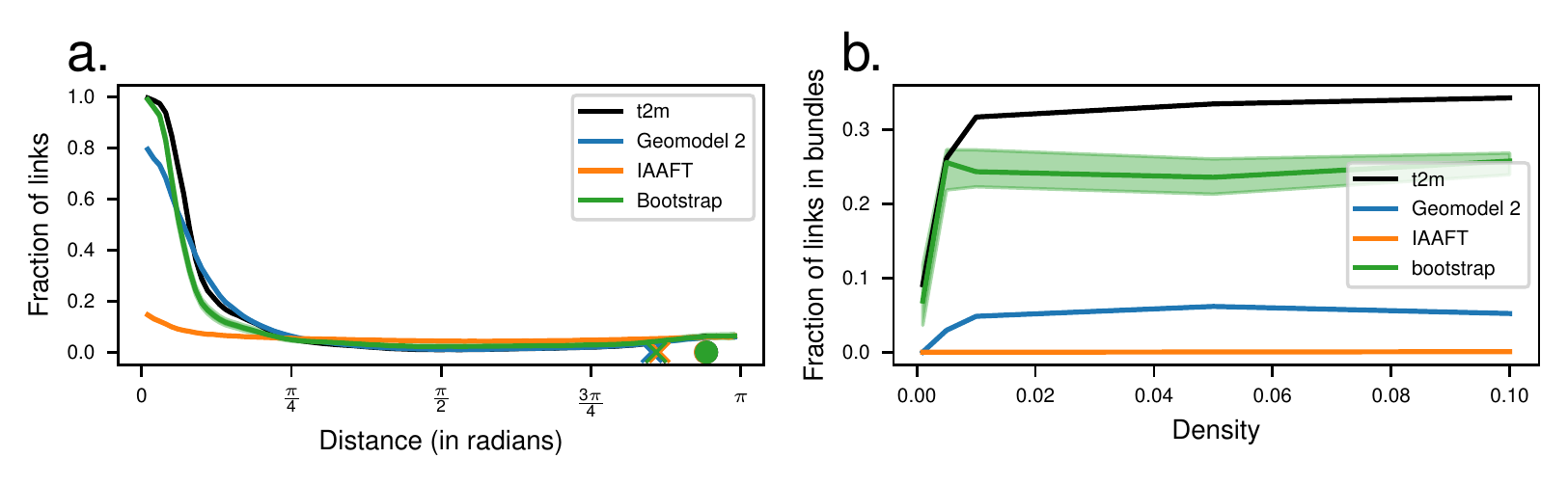}
    \caption{\textbf{Reshuffling procedures produce unrealistic network distributions.} \textbf{a)} The fraction of formed links among all possible links at a given distance, in a Pearson correlation network for t2m. We employ Geomodel 2 with $\epsilon=0.05$ and $\min(10^7,10\cdot\text{number of edges})$ rewirings. Node-wise IAAFT-resampling results in approximately uniform link distribution. A naive bootstrap in time over all grid points simultaneously as well as Geomodel 2 results in a more realistic link length distribution. \textbf{b)} The fraction of links in 1-to-many link bundles. Node-wise reshuffling as well as Geomodel 2 destroys network properties that are induced by localized correlation structure.}
    \label{fig:badresampling}
\end{figure}

A first improvement over random link redistribution or independence between the time-series is GeoModel II \citep{Wiedermann2015}, which approximately preserves both the degree distribution and the node-wise link-length distribution. But 
it still does not recover the natural tendency in the network to form bundled links. Given one link is formed, the likelihood of a neighboring link increases in a locally correlated random field. When simply recovering the total link-length distribution, as does GeoModel II, this likelihood does not increase (Fig. \ref{fig:badresampling}b). Furthermore, fixing the degree sequence of the nodes might not be representative for the distribution of the constructed graph. In Section \ref{sec:bundles_clusters} we have seen spurious high-degree clusters on random locations. In conclusion, no explicit resampling scheme has been proposed that recovers the joint link distribution of locally correlated random fields.

\textbf{Proposing a concrete network resampling scheme always runs the risk of missing or distorting an important aspect of the underlying dynamical system and estimation procedure.} The tendency to form bundled connections depends on the localized correlation structure. The expected number of spurious links depends on the (co-)variability of the estimates, the spatial correlation structure of the random field, the autocorrelation of the time series, among many other factors. And all these aspects do not even cover the more complex time-series dynamics that we might want to account for. 

\subsection{Distribution-preserving Ensembles}
\label{sec:subsamp_ensembles}

Instead of trying to solve the impossible problem of finding which exact characteristics to preserve for the climatic question of interest, we propose to \textbf{construct networks such that all members approximately reflect the network distribution that originates in the underlying physical processes.}
As discussed above, state-of-the-art network resampling approaches calculate many estimates of surrogate networks that do not reflect the original distribution. Crucially, they only obtain \textit{one} noisy estimate of the similarity value on each edge. With multiple estimates, we not only get a more robust total estimate but an approximation of the estimation variability on each edge. Instead of a single network estimate, we have access to an ensemble of equally valuable estimates that allows to judge whether a network estimation procedure really is trustworthy, and empowers us to tackle most of the issues presented in Section \ref{sec3}. Instead of reshuffling the data node-wise, we propose to jointly subsample or resample time-windows for both end points of an edge or even for all nodes simultaneously, preserving the original dynamics in space. There is an abundance of resampling techniques for multivariate time series. One popular approach is block bootstrapping \citep{Lahiri2003,Shao1995}. Another one is subsampling \citep{Politis1999}. Individual ensemble members should both reflect the \textbf{same data distribution} (induced by the selection of included time points), as well as be sufficiently \textbf{independent} (by representing different time-windows that are far enough apart). Such an ensemble could be constructed by the following pipeline:
\begin{enumerate}
    \item[(i)] Decide on a network construction procedure with unbiased edge estimates, such that the dynamics of interest behave robustly within the network distribution.
    \item[(ii)] Construct many networks with the same procedure (i) by subsampling/block-bootstrapping data in time on all grid points simultaneously.
    \item[(iii)] Evaluate reoccuring edges and patterns such as link bundles.
\end{enumerate}

If the quantity of interest can be expressed by summary statistics, the pipeline as a whole should be unbiased to yield calibrated confidence intervals. We have seen that the maximal degree or the link-length distribution of the networks can be systematically biased. Crucially, estimates of single edges or fixed neighborhoods are unbiased when the employed similarity estimator is unbiased. In this case, given a large enough ensemble, uncertainty estimates and p-values would be precise. In practice, unbiased estimators often increase estimation variance so much that the uncertainty becomes too large. In complex estimation tasks like mutual information estimation, this approach might reveal that empirically constructed networks are too incoherent due to a lack of available data. Instead of communicating false certainty, the bootstrap approach would suggest conclusions like: ``The amount of available data does not suffice to significantly detect this teleconnection with our small-bias mutual information estimator.''

From a dynamical systems perspective one could argue that different time points represent a different state of the dynamical system. From a statistical perspective, a distribution over networks is implicitly chosen when selecting the data set. The remaining task is to define what characterizes state of the dynamical system one is interested in, such as different ENSO phases. If we construct an ensemble that is independent of the state, we simply recover links and patterns that are present most of the time. Teleconnections that are only active some of the time become difficult to distinguish from noisy connections.

A single network, that is more accurate than the result of a single similarity estimate per edge, can be derived by selecting the edges that appear in most ensemble members. A related approach is the variable selection method \textit{stability selection} \citep{Meinshausen2010}. In practice, stability selection often markedly improves the baseline variable selection or structure estimation algorithm. Another approach could only accept links with small estimation variability.

While most studies directly perform bootstrapping on the graph structure and not on underlying data \citep{Chen2019,Levin2019}, a similar idea has been previously suggested \citep{Friedman2013}. To the best of our knowledge, it has not yet been applied in climate science.
In practice, data scarcity, distribution shift and varying regimes of the dynamical system complicate finding a suitable resampling or subsampling scheme that produces both sufficiently independent and identically distributed ensemble members without biases. In both bootstrap as well as subsampling techniques, a suitable choice of window size depends on the autocorrelation of the time series at hand. Highly anisotropic autocorrelation (Fig. \ref{fig:t2m_ar}) complicates designing a consistent procedure for all nodes simultaneously. With these complications and the journal's page limit in mind, we postpone proposing an explicit ensemble construction procedure to future work.

\section{Conclusions}
\label{sec:conclusion}

Whenever constructing networks from data, it is not obvious that they reflect ground truth structure. Given a finite amount of data, similarity estimates contain estimation errors. Under such non-negligible estimation variability, we find several types of spurious behaviour using typical network construction schemes:
\begin{itemize}
    \item Not only the choice of similarity measure, but also the choice of estimators is an influential design decision. The properties of the estimator determine how well single empirical networks approximate the population network and if the uncertainty of an ensemble is accurate.
    \item Global properties of finite-sample networks such as averages, variances and maxima of network measures or the spectrum of the adjacency matrix can be heavily distorted.
    \item Links occur in bundles when the data is locally correlated and the estimator transmits the correlation structure. This leads to spurious link bundles and regions of spuriously high or low degree.
    \item Under anisotropic autocorrelation or marginal distributions, differing data distributions on the nodes cause anisotropic estimation variability on the edges, which in turn introduces biases in the empirical networks. Anisotropic noise levels may lead to weak representation of nodes in the network, weighted kNN graphs reduce anisotropic behaviour via the inductive bias to represent all nodes equally, differences can still be detected via edge weights.
    \item We find sparse networks to be more accurate in terms of false discovery rate and spurious teleconnections. Yet popular network measures such as betweenness become highly unstable in sparse networks. This constitutes a different trade-off for each estimation task. Random fields with larger length scales allow for denser networks, but also lead to more pronounced bundling behaviour.
\end{itemize}

Given the variety and extremeness of possible empirical distortions, it is crucial to reliably estimate how `trustworthy' an empirical network is. State-of-the-art resampling procedures only capture particular parts of the empirical network distribution, and consequently miss other possibly relevant aspects of the dynamical system. When the implicit null hypothesis of the resampling technique does not capture all relevant properties of the dynamical system, the value of the significance test is questionable. Specifically surrogate tests, which hypothesize independence between nodes, do not reflect a physically meaningful null hypothesis when the dynamical system is locally correlated, which the link-length distribution typically clearly reflects. Random artifacts that stem from local correlation structure will then appear significant. In the past, climate network approaches have been based on calculating a single similarity estimate on each edge; we propose to generate multiple estimates via sub- or resampling in time, in order to estimate the estimation variability on each edge. This allows to approximate the intrinsic distribution over the constructed networks induced by the underlying data distribution and the chosen estimation procedure.

\textbf{Future work.} Most importantly future network studies in climate should estimate the underlying estimation error in each edge in order to argue about significance in a statistically meaningful way. Given scarce data, the variability of similarities cannot be precisely estimated (Fig. \ref{fig:resampling}a). Further challenges inherent to climatic data have to be addressed to successfully implement our proposed framework of network ensembles, which represent the underlying dynamics by design and are sufficiently independent in time. Adjusting the window length for block bootstrapping on each edge to the autocorrelation of the nodes could yield consistent estimates of the estimation variance on all edges. More work is needed to develop good similarity estimators, robust network construction procedures and resampling techniques for the climate context, respecting distribution shifts, varying regimes, anisotropy and measurement errors in chaotic dynamical systems.


While we only construct undirected networks, directed empirical networks as well as EOFs suffer from insufficient data in analogous ways. From estimating lagged correlations to probabilistic graphical models \citep{Koller2009} and causal networks \citep{Runge2019}, a similar simulation analysis could quantify strengths and weaknesses of different network construction procedures. Some works have reduced the network size by clustering spatial areas of temporally coherent behaviour before network construction (cf. \citet{Rheinwalt2015,Fountalis2018,Runge2019b}). Whether such an approach is statistically beneficial is task-dependent and deserves a more thorough consideration in future work. On the one hand, too many nodes may induce many errors and a systematic distortion in the spectrum of empirical covariance matrices \citep{Donoho2013,Lam2020,Morales2021}, on the other hand single errors have higher impact in smaller networks, and node estimation constitutes yet another challenging task in the network construction pipeline. A key question in this regard is: How can we minimize the edge-wise estimation variance in the downsized network? Previous work had selected representative locations for each cluster, but maybe an aggregation of regional information can boost statistical robustness further.

More complex time dynamics could introduce other kinds of spuriousness into empirical estimates we were not able to cover with our simple autoregressive model. A theoretical analysis of networks from spatio-temporal data would also be very insightful. Exploring alternatives to Mutual Information, such as the Hilbert Schmidt Independence Criterion or the Randomized Information Coefficient, could yield novel insights into dynamics of the climatic system.

Networks have been constructed in several scientific fields to detect complex structures in spatio-temporal data. Geneticists try to identify connections between certain genes and the development of deseases by estimating Pearson correlation networks under the term weighted gene co-expression analysis \citep{Horvath2011,Niu2019}. Neuro scientists \citep{Sporns2010} aim to understand the functional connectivity in the brain with weighted voxel coactivation network analysis \citep{Mumford2010}. While in this work we focus on the application of functional networks in climate and geo science, our conceptual findings hold in any domain where networks are constructed from spatio-temporal data. 

\section*{Acknowledgements}

Funded by the Deutsche Forschungsgemeinschaft (DFG, German Research Foundation) under Germany’s Excellence Strategy – EXC number 2064/1 – Project number 390727645. The authors thank the International Max Planck Research School for Intelligent Systems (IMPRS-IS) for supporting Moritz Haas. We want to thank all members of the Theory of Machine Learning and the Machine Learning in Climate Science group in Tübingen for helpful discussions. We thank all reviewers and our editor at the Journal of Climate for valuable feedback. Finally we would like to thank Joe Guinness for his advice concerning the use of Mat\'ern covariance on the sphere.

\section*{Data and Code Availability Statements}

\citet{era5hs,era5mp,era5ms} were downloaded from the Copernicus Climate Change Service (C3S) Climate Data Store. Python code for reproducing all results in this paper can be found at \url{https://github.com/moritzhaas/climate_nets_from_random_fields/}

\bibliography{arxiv_version.bib}

\begin{thebibliography}{55}
\providecommand{\natexlab}[1]{#1}
\providecommand{\url}[1]{\texttt{#1}}
\expandafter\ifx\csname urlstyle\endcsname\relax
  \providecommand{\doi}[1]{doi: #1}\else
  \providecommand{\doi}{doi: \begingroup \urlstyle{rm}\Url}\fi

\bibitem[Antos and Kontoyiannis(2001)]{Antos2001}
A.~Antos and I.~Kontoyiannis.
\newblock Convergence properties of functional estimates for discrete
  distributions.
\newblock \emph{Random Structures \& Algorithms}, 19:\penalty0 163--193, 2001.

\bibitem[Bendito et~al.(2007)Bendito, Carmona, Encinas, and Gesto]{Bendito2007}
E.~Bendito, A.~Carmona, A.~M. Encinas, and J.~M. Gesto.
\newblock Estimation of fekete points.
\newblock \emph{Journal of Computational Physics}, 225:\penalty0 2354--2376,
  2007.

\bibitem[Bevilacqua et~al.(2012)Bevilacqua, Gaetan, Mateu, and
  Porcu]{Bevilacqua2012}
M.~Bevilacqua, C.~Gaetan, J.~Mateu, and E.~Porcu.
\newblock Estimating space and space-time covariance functions for large data
  sets: A weighted composite likelihood approach.
\newblock \emph{Journal of the American Statistical Association}, 107:\penalty0
  268--280, 2012.

\bibitem[Bevilacqua et~al.(2020)Bevilacqua, Caamaño-Carrillo, and
  Porcu]{Bevilacqua2020}
M.~Bevilacqua, C.~Caamaño-Carrillo, and E.~Porcu.
\newblock Unifying compactly supported and {Matern} covariance functions in
  spatial statistics.
\newblock \emph{Journal of Multivariate Analysis}, 189:\penalty0 104949, 2020.

\bibitem[Bialonski et~al.(2010)Bialonski, Horstmann, and
  Lehnertz]{Bialonski2010}
S.~Bialonski, M.~T. Horstmann, and K.~Lehnertz.
\newblock From brain to earth and climate systems: Small-world interaction
  networks or not?
\newblock \emph{Chaos: An Interdisciplinary Journal of Nonlinear Science},
  20:\penalty0 013134, 2010.

\bibitem[Bullmore and Sporns(2012)]{Bullmore2012}
E.~Bullmore and O.~Sporns.
\newblock The economy of brain network organization.
\newblock \emph{Nature reviews. Neuroscience}, 13:\penalty0 336--49, 2012.

\bibitem[Chen et~al.(2021)Chen, Genton, and Sun]{Chen2021}
W.~Chen, M.~G. Genton, and Y.~Sun.
\newblock Space-time covariance structures and models.
\newblock \emph{Annual Review of Statistics and Its Application}, 8:\penalty0
  191--215, 2021.

\bibitem[Chwialkowski and Gretton(2014)]{Chwialkowski2014}
K.~Chwialkowski and A.~Gretton.
\newblock A kernel independence test for random processes.
\newblock \emph{Proceedings of the 31st International Conference on Machine
  Learning}, 32:\penalty0 1422--1430, 2014.

\bibitem[Clark(2010)]{Clark2010}
I.~Clark.
\newblock Statistics or geostatistics? {Sampling} error or nugget effect?
\newblock \emph{Journal of the Southern African Institute of Mining and
  Metallurgy}, 110:\penalty0 307 -- 312, 2010.

\bibitem[Clifford et~al.(1989)Clifford, Richardson, and Hemon]{Clifford1989}
P.~Clifford, S.~Richardson, and D.~Hemon.
\newblock Assessing the significance of the correlation between two spatial
  processes.
\newblock \emph{Biometrics}, 45:\penalty0 123, 1989.

\bibitem[Datta et~al.(2016)Datta, Banerjee, Finley, and Gelfand]{Datta2016}
A.~Datta, S.~Banerjee, A.~O. Finley, and A.~E. Gelfand.
\newblock Hierarchical nearest-neighbor {Gaussian} process models for large
  geostatistical datasets.
\newblock \emph{Journal of the American Statistical Association}, 111:\penalty0
  800--812, 2016.

\bibitem[Fan et~al.(2015)Fan, Liao, and Liu]{Fan2015}
J.~Fan, Y.~Liao, and H.~Liu.
\newblock An overview on the estimation of large covariance and precision
  matrices.
\newblock \emph{Econometrics Journal}, 19:\penalty0 C1--C32, 2015.

\bibitem[Fardet and Levina(2021)]{Fardet2021}
T.~Fardet and A.~Levina.
\newblock Weighted directed clustering: Interpretations and requirements for
  heterogeneous, inferred, and measured networks.
\newblock \emph{Phys. Rev. Research}, 3:\penalty0 043124, 2021.

\bibitem[Forman(2003)]{Forman2003}
R.~R. Forman.
\newblock Bochner's method for cell complexes and combinatorial {Ricci}
  curvature.
\newblock \emph{Discrete \& Computational Geometry}, 29:\penalty0 323--374,
  2003.

\bibitem[Gao et~al.(2017)Gao, Kannan, Oh, and Viswanath]{ksgmixed}
W.~Gao, S.~Kannan, S.~Oh, and P.~Viswanath.
\newblock Estimating mutual information for discrete-continuous mixtures.
\newblock In \emph{Proceedings of the 31st International Conference on Neural
  Information Processing Systems}, NIPS'17, Red Hook, NY, USA, 2017. Curran
  Associates Inc.
\newblock ISBN 9781510860964.

\bibitem[Gao et~al.(2018)Gao, Oh, and Viswanath]{gao2016demystifying}
W.~Gao, S.~Oh, and P.~Viswanath.
\newblock Demystifying fixed k-nearest neighbor information estimators.
\newblock \emph{IEEE Transactions on Information Theory}, 2018.

\bibitem[Gerhardus and Runge(2020)]{Gerhardus2020}
A.~Gerhardus and J.~Runge.
\newblock High-recall causal discovery for autocorrelated time series with
  latent confounders.
\newblock \emph{Advances in Neural Information Processing Systems},
  33:\penalty0 12615--12625, 2020.

\bibitem[Gilpin(1993)]{Gilpin1993}
A.~R. Gilpin.
\newblock Table for conversion of {Kendall's} tau to {Spearman's} rho within
  the context of measures of magnitude of effect for meta-analysis.
\newblock \emph{Educational and Psychological Measurement}, 53:\penalty0
  87--92, 1993.

\bibitem[Gretton et~al.(2012)Gretton, Borgwardt, Rasch, Sch{{\"o}}lkopf, and
  Smola]{Gretton2012}
A.~Gretton, K.~M. Borgwardt, M.~J. Rasch, B.~Sch{{\"o}}lkopf, and A.~Smola.
\newblock A kernel two-sample test.
\newblock \emph{Journal of Machine Learning Research}, 13\penalty0
  (25):\penalty0 723--773, 2012.

\bibitem[Guinness and Fuentes(2016)]{Guinness2016}
J.~Guinness and M.~Fuentes.
\newblock Isotropic covariance functions on spheres: Some properties and
  modeling considerations.
\newblock \emph{Journal of Multivariate Analysis}, 143:\penalty0 143--152,
  2016.

\bibitem[Guttorp and Gneiting(2006)]{Guttorp2006}
P.~Guttorp and T.~Gneiting.
\newblock Studies in the history of probability and statistics xlix on the
  {Matérn} correlation family.
\newblock \emph{Biometrika}, 93:\penalty0 989--995, 2006.

\bibitem[Hlinka et~al.(2017)Hlinka, Hartman, Jajcay, Tomeĉek, Tintêra, and
  Paluŝ]{Hlinka2017}
J.~Hlinka, D.~Hartman, N.~Jajcay, D.~Tomeĉek, J.~Tintêra, and M.~Paluŝ.
\newblock Small-world bias of correlation networks: From brain to climate.
\newblock \emph{Chaos: An Interdisciplinary Journal of Nonlinear Science}, 27,
  2017.

\bibitem[Hurter et~al.(2012)Hurter, Ersoy, and Telea]{Hurter2012}
C.~Hurter, O.~Ersoy, and A.~Telea.
\newblock Graph bundling by kernel density estimation.
\newblock \emph{Computer Graphics Forum}, 31:\penalty0 865--874, 2012.

\bibitem[Kraskov et~al.(2003)Kraskov, Stoegbauer, and Grassberger]{Kraskov2003}
A.~Kraskov, H.~Stoegbauer, and P.~Grassberger.
\newblock Estimating mutual information.
\newblock \emph{Physical Review E}, 69:\penalty0 16, 2003.

\bibitem[Lam(2020)]{Lam2020}
C.~Lam.
\newblock High-dimensional covariance matrix estimation.
\newblock \emph{Wiley Interdisciplinary Reviews: Computational Statistics},
  12:\penalty0 e1485, 2020.

\bibitem[Lambert et~al.(2010)Lambert, Bourqui, and Auber]{Lambert2010}
A.~Lambert, R.~Bourqui, and D.~Auber.
\newblock 3d edge bundling for geographical data visualization.
\newblock \emph{Proceedings of the International Conference on Information
  Visualisation}, 2010.

\bibitem[Lang and Schwab(2015)]{Lang2015}
A.~Lang and C.~Schwab.
\newblock Isotropic {Gaussian} random fields on the sphere: Regularity, fast
  simulation and stochastic partial differential equations.
\newblock \emph{The Annals of Applied Probability}, 25\penalty0 (6), 2015.

\bibitem[Ledoit and Wolf(2004)]{Ledoit2004}
O.~Ledoit and M.~Wolf.
\newblock A well-conditioned estimator for large-dimensional covariance
  matrices.
\newblock \emph{Journal of Multivariate Analysis}, 88:\penalty0 365--411, 2004.

\bibitem[Ledoit and Wolf(2020)]{Ledoit2020}
O.~Ledoit and M.~Wolf.
\newblock Analytical nonlinear shrinkage of large-dimensional covariance
  matrices.
\newblock \emph{The Annals of Statistics}, 48:\penalty0 3043--3065, 2020.

\bibitem[Levit-Binnun et~al.(2013)Levit-Binnun, Davidovitch, and
  Y.]{LevitBinnun}
N.~Levit-Binnun, M.~Davidovitch, and Golland Y.
\newblock Sensory and motor secondary symptoms as indicators of brain
  vulnerability.
\newblock \emph{Journal of Neurodevelopmental Disorders}, 5, 2013.

\bibitem[Li(2009)]{Li2009}
Y.~Li.
\newblock Modeling and analysis of spatially correlated data.
\newblock 2009.

\bibitem[Lindgren et~al.(2011)Lindgren, Rue, and Lindström]{Lindgren2011}
F.~Lindgren, H.~Rue, and J.~Lindström.
\newblock An explicit link between {Gaussian} fields and gaussian markov random
  fields: the stochastic partial differential equation approach.
\newblock \emph{Journal of the Royal Statistical Society: Series B (Statistical
  Methodology)}, 73:\penalty0 423--498, 2011.

\bibitem[Marčenko and Pastur(1967)]{Marcenko1967}
V.~A. Marčenko and L.~A. Pastur.
\newblock Distribution of eigenvalues for some sets of random matrices.
\newblock \emph{Mathematics of the USSR-Sbornik}, 1:\penalty0 457--483, 1967.

\bibitem[Neal(2017)]{Neal2017}
Z.~P. Neal.
\newblock How small is it? {Comparing} indices of small worldliness.
\newblock \emph{Network Science}, 5\penalty0 (1):\penalty0 30–44, 2017.

\bibitem[Newman(2018)]{Newman2018}
M.~Newman.
\newblock \emph{Networks}.
\newblock OUP Oxford, 2018.

\bibitem[Nocke et~al.(2015)Nocke, Buschmann, Donges, Marwan, Schulz, and
  Tominski]{Nocke2015}
T.~Nocke, S.~Buschmann, J.~F. Donges, N.~Marwan, H.~J. Schulz, and C.~Tominski.
\newblock Review: Visual analytics of climate networks.
\newblock \emph{Nonlinear Processes in Geophysics}, 22:\penalty0 545--570,
  2015.

\bibitem[Ollivier(2010)]{Ollivier2010}
Y.~Ollivier.
\newblock A survey of {Ricci} curvature for metric spaces and {Markov} chains.
\newblock \emph{Probabilistic Approach to Geometry}, 57, 2010.

\bibitem[Onnela et~al.(2005)Onnela, Saramäki, Kertész, and Kaski]{Onnela2005}
J.~P. Onnela, J.~Saramäki, J.~Kertész, and K.~Kaski.
\newblock Intensity and coherence of motifs in weighted complex networks.
\newblock \emph{Physical Review E}, 71:\penalty0 065103, 2005.

\bibitem[Paninski(2003)]{Paninski2003}
L.~Paninski.
\newblock Estimation of entropy and mutual information.
\newblock \emph{Neural Computation}, 15:\penalty0 1191--1253, 2003.

\bibitem[Papana and Kugiumtzis(2009)]{mieval}
A.~Papana and D.~Kugiumtzis.
\newblock Evaluation of mutual information estimators for time series.
\newblock \emph{International Journal of Bifurcation and Chaos}, 19\penalty0
  (12):\penalty0 4197–4215, 2009.

\bibitem[Pourahmadi(2013)]{Pourahmadi2013}
M.~Pourahmadi.
\newblock High-dimensional covariance estimation.
\newblock 2013.

\bibitem[Raymaekers and Rousseeuw(2021)]{Raymaekers2021}
J.~Raymaekers and P.~J. Rousseeuw.
\newblock Fast robust correlation for high-dimensional data.
\newblock \emph{Technometrics}, 63:\penalty0 184--198, 2021.

\bibitem[Romano et~al.(2018)Romano, Vinh, Verspoor, and Bailey]{Romano2018}
S.~Romano, N.~X. Vinh, K.~Verspoor, and J.~Bailey.
\newblock The randomized information coefficient: assessing dependencies in
  noisy data.
\newblock \emph{Machine Learning}, 107:\penalty0 509--549, 2018.

\bibitem[Rubinov and Sporns(2010)]{Rubinov2010}
M.~Rubinov and O.~Sporns.
\newblock Complex network measures of brain connectivity: uses and
  interpretations.
\newblock \emph{Neuroimage}, 52\penalty0 (3):\penalty0 1059--1069, 2010.

\bibitem[Runge et~al.(2012)Runge, Heitzig, Marwan, and Kurths]{Runge2012}
J.~Runge, J.~Heitzig, N.~Marwan, and J.~Kurths.
\newblock Quantifying causal coupling strength: A lag-specific measure for
  multivariate time series related to transfer entropy.
\newblock \emph{Physical Review E, Statistical, nonlinear, and soft matter
  physics}, 86:\penalty0 061121, 2012.

\bibitem[Runge et~al.(2014)Runge, Petoukhov, and Kurths]{Runge2014}
J.~Runge, V.~Petoukhov, and J.~Kurths.
\newblock Quantifying the strength and delay of climatic interactions: The
  ambiguities of cross correlation and a novel measure based on graphical
  models.
\newblock \emph{Journal of Climate}, 27:\penalty0 720--739, 2014.

\bibitem[Runge et~al.(2019)Runge, Bathiany, Bollt, Camps-Valls, Coumou, Deyle,
  Glymour, Kretschmer, Mahecha, Muñoz-Marí, van Nes, Peters, Quax,
  Reichstein, Scheffer, Schölkopf, Spirtes, Sugihara, Sun, Zhang, and
  Zscheischler]{Runge2019b}
J.~Runge, S.~Bathiany, E.~Bollt, G.~Camps-Valls, D.~Coumou, E.~Deyle,
  C.~Glymour, M.~Kretschmer, M.~D. Mahecha, J.~Muñoz-Marí, E.~H. van Nes,
  J.~Peters, R.~Quax, M.~Reichstein, M.~Scheffer, B.~Schölkopf, P.~Spirtes,
  G.~Sugihara, J.~Sun, K.~Zhang, and J.~Zscheischler.
\newblock Inferring causation from time series in {Earth} system sciences.
\newblock \emph{Nature Communications}, 10:\penalty0 1--13, 2019.

\bibitem[Selassie et~al.(2011)Selassie, Heller, and Heer]{Selassie2011}
D.~Selassie, B.~Heller, and J.~Heer.
\newblock Divided edge bundling for directional network data.
\newblock \emph{IEEE Transactions on Visualization and Computer Graphics},
  17:\penalty0 2354--2363, 2011.

\bibitem[Stein(2011)]{Stein2011}
M.~L. Stein.
\newblock Space–time covariance functions.
\newblock \emph{Journal of the American Statistical Association}, 100:\penalty0
  310--321, 2011.

\bibitem[Strnad et~al.(2022)Strnad, Schlör, Fröhlich, and
  Goswami]{Strnad2022}
F.~Strnad, J.~Schlör, C.~Fröhlich, and B.~Goswami.
\newblock Teleconnection patterns of different {El Niño} types revealed by
  climate network curvature.
\newblock \emph{Geophysical Research Letters}, 49, 2022.

\bibitem[von Luxburg(2007)]{Luxburg2007}
U.~von Luxburg.
\newblock A tutorial on spectral clustering.
\newblock \emph{Statistics and Computing}, 17\penalty0 (4):\penalty0 395--416,
  2007.

\bibitem[Wang et~al.(2017)Wang, Ghumare, Vandenberghe, and Dupont]{Wang2017}
Y.~Wang, E.~Ghumare, R.~Vandenberghe, and P.~Dupont.
\newblock Comparison of different generalizations of clustering coefficient and
  local efficiency for weighted undirected graphs.
\newblock \emph{Neural Computation}, 29\penalty0 (2):\penalty0 313--331, 2017.

\bibitem[Watts and Strogatz(1998)]{Watts1998}
D.~J. Watts and S.~H. Strogatz.
\newblock Collective dynamics of ‘small-world’ networks.
\newblock \emph{Nature 1998 393:6684}, 393:\penalty0 440--442, 1998.

\bibitem[Whittle(1954)]{Whittle1954}
P.~Whittle.
\newblock On stationary processes in the plane.
\newblock \emph{Biometrika}, 41:\penalty0 434--449, 1954.

\bibitem[Zerenner et~al.(2014)Zerenner, Friederichs, Lehnertz, and
  Hense]{Zerenner2014}
T.~Zerenner, P.~Friederichs, K.~Lehnertz, and A.~Hense.
\newblock A {Gaussian} graphical model approach to climate networks.
\newblock \emph{Chaos: An Interdisciplinary Journal of Nonlinear Science},
  24:\penalty0 023103, 2014.

\end{thebibliography}


\begin{thebibliography}{75}
\providecommand{\natexlab}[1]{#1}
\providecommand{\url}[1]{\texttt{#1}}
\expandafter\ifx\csname urlstyle\endcsname\relax
  \providecommand{\doi}[1]{doi: #1}\else
  \providecommand{\doi}{doi: \begingroup \urlstyle{rm}\Url}\fi

\bibitem[Agarwal et~al.(2019)Agarwal, L., Marwan, Maheswaran, Merz, and
  Kurths]{Agarwal2019}
A.~Agarwal, L., N.~Marwan, R.~Maheswaran, B.~Merz, and J.~Kurths.
\newblock Network-based identification and characterization of teleconnections
  on different scales.
\newblock \emph{Scientific Reports}, 9:\penalty0 1--12, 2019.

\bibitem[Barber et~al.(2019)Barber, Lamontagne, and Vogel]{Barber2019}
C.~Barber, J.~R. Lamontagne, and R.~M. Vogel.
\newblock Improved estimators of correlation and r2 for skewed hydrologic data.
\newblock \emph{Hydrological Sciences Journal}, 65:\penalty0 87--101, 2019.

\bibitem[Bendito et~al.(2007)Bendito, Carmona, Encinas, and Gesto]{Bendito2007}
E.~Bendito, A.~Carmona, A.~M. Encinas, and J.~M. Gesto.
\newblock Estimation of fekete points.
\newblock \emph{Journal of Computational Physics}, 225:\penalty0 2354--2376,
  2007.

\bibitem[Bialonski et~al.(2010)Bialonski, Horstmann, and
  Lehnertz]{Bialonski2010}
S.~Bialonski, M.~T. Horstmann, and K.~Lehnertz.
\newblock From brain to earth and climate systems: Small-world interaction
  networks or not?
\newblock \emph{Chaos: An Interdisciplinary Journal of Nonlinear Science},
  20:\penalty0 013134, 2010.

\bibitem[Boers et~al.(2013)Boers, Bookhagen, Marwan, Kurths, and
  Marengo]{Boers2013}
N.~Boers, B.~Bookhagen, N.~Marwan, J.~Kurths, and J.~Marengo.
\newblock Complex networks identify spatial patterns of extreme rainfall events
  of the south american monsoon system.
\newblock \emph{Geophysical Research Letters}, 40:\penalty0 4386--4392, 2013.

\bibitem[Boers et~al.(2014)Boers, Bookhagen, Barbosa, Marwan, Kurths, and
  Marengo]{Boers2014}
N.~Boers, B.~Bookhagen, H.~M.J. Barbosa, N.~Marwan, J.~Kurths, and J.~A.
  Marengo.
\newblock Prediction of extreme floods in the eastern central andes based on a
  complex networks approach.
\newblock \emph{Nature Communications}, 5:\penalty0 1--7, 2014.

\bibitem[Boers et~al.(2019)Boers, Goswami, Rheinwalt, Bookhagen, Hoskins, and
  Kurths]{Boers2019}
N.~Boers, B.~Goswami, A.~Rheinwalt, B.~Bookhagen, B.~Hoskins, and J.~Kurths.
\newblock Complex networks reveal global pattern of extreme-rainfall
  teleconnections.
\newblock \emph{Nature}, 566:\penalty0 373--377, 2019.

\bibitem[Brockwell and Davis(1991)]{Brockwell1991}
P.~J. Brockwell and R.~A. Davis.
\newblock \emph{Time series: theory and methods}.
\newblock Springer New York, 1991.

\bibitem[Cellucci et~al.(2005)Cellucci, Albano, and Rapp]{Cellucci2005}
C.~J. Cellucci, A.~M. Albano, and P.~E. Rapp.
\newblock Statistical validation of mutual information calculations: Comparison
  of alternative numerical algorithms.
\newblock \emph{Physical Review E}, 71:\penalty0 066208, 2005.

\bibitem[Chen et~al.(2019)Chen, Gel, Lyubchich, and Nezafati]{Chen2019}
Y.~Chen, Y.~R. Gel, V.~Lyubchich, and K.~Nezafati.
\newblock Snowboot: Bootstrap methods for network inference.
\newblock \emph{R Journal}, 10:\penalty0 95--113, 2019.

\bibitem[Chwialkowski and Gretton(2014)]{Chwialkowski2014}
K.~Chwialkowski and A.~Gretton.
\newblock A kernel independence test for random processes.
\newblock \emph{Proceedings of the 31st International Conference on Machine
  Learning}, 32:\penalty0 1422--1430, 2014.

\bibitem[Cressie(1993)]{Cressie1993}
N.~A.~C. Cressie.
\newblock \emph{Statistics for spatial data}.
\newblock Wiley, 1993.

\bibitem[Deza et~al.(2015)Deza, Barreiro, and Masoller]{Deza2015}
J.~I. Deza, M.~Barreiro, and C.~Masoller.
\newblock Assessing the direction of climate interactions by means of complex
  networks and information theoretic tools.
\newblock \emph{Chaos: An Interdisciplinary Journal of Nonlinear Science}, 25,
  2015.

\bibitem[Donges et~al.(2009{\natexlab{a}})Donges, Zou, Marwan, and
  Kurths]{Donges2009}
J.~F. Donges, Y.~Zou, N.~Marwan, and J.~Kurths.
\newblock The backbone of the climate network.
\newblock \emph{Europhysics Letters}, 87:\penalty0 48007, 2009{\natexlab{a}}.

\bibitem[Donges et~al.(2009{\natexlab{b}})Donges, Zou, Marwan, and
  Kurths]{Donges2009b}
J.~F. Donges, Y.~Zou, N.~Marwan, and J.~Kurths.
\newblock Complex networks in climate dynamics - comparing linear and nonlinear
  network construction methods.
\newblock \emph{European Physical Journal: Special Topics}, 174:\penalty0
  157--179, 2009{\natexlab{b}}.

\bibitem[Donges et~al.(2015)Donges, Heitzig, Beronov, Wiedermann, Runge, Feng,
  Tupikina, Stolbova, Donner, Marwan, Dijkstra, and Kurths]{pyunicorn}
J.~F. Donges, J.~Heitzig, B.~Beronov, M.~Wiedermann, J.~Runge, Q.~Y. Feng,
  L.~Tupikina, V.~Stolbova, R.~V. Donner, N.~Marwan, H.~A. Dijkstra, and
  J.~Kurths.
\newblock Unified functional network and nonlinear time series analysis for
  complex systems science: The pyunicorn package.
\newblock \emph{Chaos: An Interdisciplinary Journal of Nonlinear Science}, 25,
  2015.

\bibitem[Donoho et~al.(2013)Donoho, Gavish, and Johnstone]{Donoho2013}
D.~Donoho, M.~Gavish, and I.~Johnstone.
\newblock Optimal shrinkage of eigenvalues in the spiked covariance model.
\newblock \emph{Annals of Statistics}, 46:\penalty0 1742--1778, 2013.

\bibitem[Eichner et~al.(2003)Eichner, Koscielny-Bunde, Bunde, Havlin, and
  Schellnhuber]{Eichner2003}
J.~Eichner, E.~Koscielny-Bunde, A.~Bunde, S.~Havlin, and H.~Schellnhuber.
\newblock Power-law persistence and trends in the atmosphere: A detailed study
  of long temperature records.
\newblock \emph{Physical review E, Statistical, nonlinear, and soft matter
  physics}, 68:\penalty0 46133, 2003.

\bibitem[Ekhtiari et~al.(2019)Ekhtiari, Agarwal, Marwan, and
  Donner]{Ekhtiari2019}
N.~Ekhtiari, A.~Agarwal, N.~Marwan, and R.~V. Donner.
\newblock Disentangling the multi-scale effects of sea-surface temperatures on
  global precipitation: A coupled networks approach.
\newblock \emph{Chaos: An Interdisciplinary Journal of Nonlinear Science},
  29:\penalty0 063116, 2019.

\bibitem[Ekhtiari et~al.(2021)Ekhtiari, Ciemer, Kirsch, and
  Donner]{Ekhtiari2021}
N.~Ekhtiari, C.~Ciemer, C.~Kirsch, and R.~V. Donner.
\newblock Coupled network analysis revealing global monthly scale
  co-variability patterns between sea-surface temperatures and precipitation in
  dependence on the enso state.
\newblock \emph{The European Physical Journal Special Topics}, 230:\penalty0
  3019--3032, 2021.

\bibitem[Fan et~al.(2017)Fan, Meng, Ashkenazy, Havlin, and
  Schellnhuber]{Fan2017}
J.~Fan, J.~Meng, Y.~Ashkenazy, S.~Havlin, and H.~J. Schellnhuber.
\newblock Network analysis reveals strongly localized impacts of el niño.
\newblock \emph{Proceedings of the National Academy of Sciences of the United
  States of America}, 114:\penalty0 7543--7548, 2017.

\bibitem[Fan et~al.(2018)Fan, Meng, Ashkenazy, Havlin, and
  Schellnhuber]{Fan2018}
J.~Fan, J.~Meng, Y.~Ashkenazy, S.~Havlin, and H.~J. Schellnhuber.
\newblock Climate network percolation reveals the expansion and weakening of
  the tropical component under global warming.
\newblock \emph{Proceedings of the National Academy of Sciences of the United
  States of America}, 115:\penalty0 E12128--E12134, 2018.

\bibitem[Fan et~al.(2022)Fan, Meng, Ludescher, Li, Surovyatkina, Chen, Kurths,
  and Schellnhuber]{Fan2022}
J.~Fan, J.~Meng, J.~Ludescher, Z.~Li, E.~Surovyatkina, X.~Chen, J.~Kurths, and
  H.~J. Schellnhuber.
\newblock Network-based approach and climate change benefits for forecasting
  the amount of indian monsoon rainfall.
\newblock \emph{Journal of Climate}, 35:\penalty0 1009--1020, 2022.

\bibitem[Fountalis et~al.(2018)Fountalis, Bracco, Dilkina, Dovrolis, and
  Keilholz]{Fountalis2018}
I.~Fountalis, A.~Bracco, B.~Dilkina, C.~Dovrolis, and S.~Keilholz.
\newblock {$\delta$}-maps: From spatio-temporal data to a weighted and lagged
  network between functional domains.
\newblock \emph{Applied Network Science}, 3, 2018.

\bibitem[Friedman et~al.(1999)Friedman, Goldszmidt, and Wyner]{Friedman2013}
N.~Friedman, M.~Goldszmidt, and A.~Wyner.
\newblock Data analysis with bayesian networks: A bootstrap approach.
\newblock \emph{Proceedings of the Fifteenth Conference on Uncertainty in
  Artificial Intelligence}, 1999.

\bibitem[Gao et~al.(2018)Gao, Oh, and Viswanath]{gao2016demystifying}
W.~Gao, S.~Oh, and P.~Viswanath.
\newblock Demystifying fixed k-nearest neighbor information estimators.
\newblock \emph{IEEE Transactions on Information Theory}, 2018.

\bibitem[Gilpin(1993)]{Gilpin1993}
A.~R. Gilpin.
\newblock Table for conversion of {Kendall's} tau to {Spearman's} rho within
  the context of measures of magnitude of effect for meta-analysis.
\newblock \emph{Educational and Psychological Measurement}, 53:\penalty0
  87--92, 1993.

\bibitem[Guez et~al.(2014)Guez, Gozolchiani, and Havlin]{Guez2014}
O.~C. Guez, A.~Gozolchiani, and S.~Havlin.
\newblock Influence of autocorrelation on the topology of the climate network.
\newblock \emph{Physical Review E}, 90:\penalty0 062814, 2014.

\bibitem[Guinness and Fuentes(2016)]{Guinness2016}
J.~Guinness and M.~Fuentes.
\newblock Isotropic covariance functions on spheres: Some properties and
  modeling considerations.
\newblock \emph{Journal of Multivariate Analysis}, 143:\penalty0 143--152,
  2016.

\bibitem[Heitzig et~al.(2011)Heitzig, Donges, Zou, Marwan, and
  Kurths]{Heitzig2011}
J.~Heitzig, J.~F. Donges, Y.~Zou, N.~Marwan, and J.~Kurths.
\newblock Node-weighted measures for complex networks with spatially embedded,
  sampled, or differently sized nodes.
\newblock \emph{European Physical Journal B}, 85, 2011.

\bibitem[Hersbach et~al.(2018)Hersbach, Bell, Berrisford, Biavati, Horányi,
  Muñoz~Sabater, Nicolas, Peubey, Radu, Rozum, Schepers, Simmons, Soci, Dee,
  and Thépaut]{era5hs}
H.~Hersbach, B.~Bell, P.~Berrisford, G.~Biavati, A.~Horányi,
  J.~Muñoz~Sabater, J.~Nicolas, C.~Peubey, R.~Radu, I.~Rozum, D.~Schepers,
  A.~Simmons, C.~Soci, D.~Dee, and J-N Thépaut.
\newblock Era5 hourly data on single levels from 1959 to present. copernicus
  climate change service (c3s) climate data store (cds), accessed 13 april
  2022.
\newblock 2018.

\bibitem[Hersbach et~al.(2019{\natexlab{a}})Hersbach, Bell, Berrisford,
  Biavati, Horányi, Muñoz~Sabater, Nicolas, Peubey, Radu, Rozum, Schepers,
  Simmons, Soci, Dee, and Thépaut]{era5mp}
H.~Hersbach, B.~Bell, P.~Berrisford, G.~Biavati, A.~Horányi,
  J.~Muñoz~Sabater, J.~Nicolas, C.~Peubey, R.~Radu, I.~Rozum, D.~Schepers,
  A.~Simmons, C.~Soci, D.~Dee, and J-N Thépaut.
\newblock Era5 monthly averaged data on pressure levels from 1959 to present.
  copernicus climate change service (c3s) climate data store (cds), accessed 13
  april 2022.
\newblock 2019{\natexlab{a}}.

\bibitem[Hersbach et~al.(2019{\natexlab{b}})Hersbach, Bell, Berrisford,
  Biavati, Horányi, Muñoz~Sabater, Nicolas, Peubey, Radu, Rozum, Schepers,
  Simmons, Soci, Dee, and Thépaut]{era5ms}
H.~Hersbach, B.~Bell, P.~Berrisford, G.~Biavati, A.~Horányi,
  J.~Muñoz~Sabater, J.~Nicolas, C.~Peubey, R.~Radu, I.~Rozum, D.~Schepers,
  A.~Simmons, C.~Soci, D.~Dee, and J-N Thépaut.
\newblock Era5 monthly averaged data on single levels from 1959 to present.
  copernicus climate change service (c3s) climate data store (cds), accessed 13
  april 2022.
\newblock 2019{\natexlab{b}}.

\bibitem[Hlinka et~al.(2014)Hlinka, Hartman, Jajcay, Vejmelka, Donner, Marwan,
  Kurths, and Paluš]{Hlinka2014}
J.~Hlinka, D.~Hartman, N.~Jajcay, M.~Vejmelka, R.~Donner, N.~Marwan, J.~Kurths,
  and M.~Paluš.
\newblock Regional and inter-regional effects in evolving climate networks.
\newblock \emph{Nonlinear Processes in Geophysics}, 21:\penalty0 451--462,
  2014.

\bibitem[Hlinka et~al.(2017)Hlinka, Hartman, Jajcay, Tomeĉek, Tintêra, and
  Paluŝ]{Hlinka2017}
J.~Hlinka, D.~Hartman, N.~Jajcay, D.~Tomeĉek, J.~Tintêra, and M.~Paluŝ.
\newblock Small-world bias of correlation networks: From brain to climate.
\newblock \emph{Chaos: An Interdisciplinary Journal of Nonlinear Science}, 27,
  2017.

\bibitem[Horvath(2011)]{Horvath2011}
S.~Horvath.
\newblock \emph{Weighted Network Analysis: Applications in Genomics and Systems
  Biology}.
\newblock Springer New York, 2011.

\bibitem[Kittel et~al.(2021)Kittel, Ciemer, Lotfi, Peron, Rodrigues, Kurths,
  and Donner]{Kittel2021}
T.~Kittel, C.~Ciemer, N.~Lotfi, T.~Peron, F.~Rodrigues, J.~Kurths, and R.~V.
  Donner.
\newblock Evolving climate network perspectives on global surface air
  temperature effects of enso and strong volcanic eruptions.
\newblock \emph{The European Physical Journal Special Topics}, 230:\penalty0
  3075--3100, 2021.

\bibitem[Koller and Friedman(2009)]{Koller2009}
D.~Koller and N.~Friedman.
\newblock \emph{Probabilistic Graphical Models: Principles and Techniques}.
\newblock MIT Press, 2009.

\bibitem[Kraskov et~al.(2003)Kraskov, Stoegbauer, and Grassberger]{Kraskov2003}
A.~Kraskov, H.~Stoegbauer, and P.~Grassberger.
\newblock Estimating mutual information.
\newblock \emph{Physical Review E}, 69:\penalty0 16, 2003.

\bibitem[Kretschmer et~al.(2017)Kretschmer, Runge, and Coumou]{Kretschmer2017}
M.~Kretschmer, J.~Runge, and D.~Coumou.
\newblock Early prediction of extreme stratospheric polar vortex states based
  on causal precursors.
\newblock \emph{Geophysical Research Letters}, 44:\penalty0 8592--8600, 2017.

\bibitem[Lahiri(2003)]{Lahiri2003}
S.~N. Lahiri.
\newblock \emph{Resampling Methods for Dependent Data}.
\newblock Springer New York, 2003.

\bibitem[Lam(2020)]{Lam2020}
C.~Lam.
\newblock High-dimensional covariance matrix estimation.
\newblock \emph{Wiley Interdisciplinary Reviews: Computational Statistics},
  12:\penalty0 e1485, 2020.

\bibitem[Lang and Schwab(2015)]{Lang2015}
A.~Lang and C.~Schwab.
\newblock Isotropic {Gaussian} random fields on the sphere: Regularity, fast
  simulation and stochastic partial differential equations.
\newblock \emph{The Annals of Applied Probability}, 25\penalty0 (6), 2015.

\bibitem[Ledoit and Wolf(2004)]{Ledoit2004}
O.~Ledoit and M.~Wolf.
\newblock A well-conditioned estimator for large-dimensional covariance
  matrices.
\newblock \emph{Journal of Multivariate Analysis}, 88:\penalty0 365--411, 2004.

\bibitem[Levin and Levina(2019)]{Levin2019}
K.~Levin and E.~Levina.
\newblock Bootstrapping networks with latent space structure.
\newblock \emph{arXiv preprint}, 2019.

\bibitem[Ludescher et~al.(2014)Ludescher, Gozolchiani, Bogachev, Bunde, Havlin,
  and Schellnhuber]{Ludescher2014}
J.~Ludescher, A.~Gozolchiani, M.~I. Bogachev, A.~Bunde, S.~Havlin, and H.~J.
  Schellnhuber.
\newblock Very early warning of next el niño.
\newblock \emph{Proceedings of the National Academy of Sciences}, 111:\penalty0
  2064--2066, 2014.

\bibitem[Ludescher et~al.(2021)Ludescher, Martin, Boers, Bunde, Ciemer, Fan,
  Havlin, Kretschmer, Kurths, Runge, Stolbova, Surovyatkina, and
  Schellnhuber]{Ludescher2021}
J.~Ludescher, M.~Martin, N.~Boers, A.~Bunde, C.~Ciemer, J.~Fan, S.~Havlin,
  M.~Kretschmer, J.~Kurths, J.~Runge, V.~Stolbova, E.~Surovyatkina, and H.~J.
  Schellnhuber.
\newblock Network-based forecasting of climate phenomena.
\newblock \emph{Proceedings of the National Academy of Sciences of the United
  States of America}, 118, 2021.

\bibitem[Meinshausen and Bühlmann(2010)]{Meinshausen2010}
N.~Meinshausen and P.~Bühlmann.
\newblock Stability selection.
\newblock \emph{Journal of the Royal Statistical Society: Series B (Statistical
  Methodology)}, 72:\penalty0 417--473, 2010.

\bibitem[Minsker and Wei(2017)]{Minsker2017}
S.~Minsker and X.~Wei.
\newblock Estimation of the covariance structure of heavy-tailed distributions.
\newblock \emph{Advances in Neural Information Processing Systems}, 2017.

\bibitem[Morales-Jimenez et~al.(2021)Morales-Jimenez, Johnstone, McKay, and
  Yang]{Morales2021}
D.~Morales-Jimenez, I.~M. Johnstone, M.~R. McKay, and J.~Yang.
\newblock Asymptotics of eigenstructure of sample correlation matrices for
  high-dimensional spiked models.
\newblock \emph{Statistica Sinica}, 31:\penalty0 571, 2021.

\bibitem[Mumford et~al.(2010)Mumford, Horvath, Oldham, Langfelder, Geschwind,
  and Poldrack]{Mumford2010}
J.~A. Mumford, S.~Horvath, M.~C. Oldham, P.~Langfelder, D.~H. Geschwind, and
  R.~A. Poldrack.
\newblock Detecting network modules in fmri time series: A weighted network
  analysis approach.
\newblock \emph{NeuroImage}, 52:\penalty0 1465, 2010.

\bibitem[Niu et~al.(2019)Niu, Zhang, Zhang, Hou, Pu, Chu, Bai, and
  Zhang]{Niu2019}
X.~Niu, J.~Zhang, L.~Zhang, Y.~Hou, S.~Pu, A.~Chu, M.~Bai, and Z.~Zhang.
\newblock Weighted gene co-expression network analysis identifies critical
  genes in the development of heart failure after acute myocardial infarction.
\newblock \emph{Frontiers in Genetics}, 10:\penalty0 1214, 2019.

\bibitem[Online(2022)]{wmoweatherstations}
ArcGIS Online.
\newblock Map of wmo weather stations. accessed on 24 october 2022., 2022.

\bibitem[Onnela et~al.(2005)Onnela, Saramäki, Kertész, and Kaski]{Onnela2005}
J.~P. Onnela, J.~Saramäki, J.~Kertész, and K.~Kaski.
\newblock Intensity and coherence of motifs in weighted complex networks.
\newblock \emph{Physical Review E}, 71:\penalty0 065103, 2005.

\bibitem[Paluš et~al.(2011)Paluš, Hartman, Hlinka, and Vejmelka]{Palus2011}
M.~Paluš, D.~Hartman, J.~Hlinka, and M.~Vejmelka.
\newblock Discerning connectivity from dynamics in climate networks.
\newblock \emph{Nonlinear Processes in Geophysics}, 18:\penalty0 751--763,
  2011.

\bibitem[Papalexiou(2018)]{Papalexiou2018}
S.~M. Papalexiou.
\newblock Unified theory for stochastic modelling of hydroclimatic processes:
  Preserving marginal distributions, correlation structures, and intermittency.
\newblock \emph{Advances in Water Resources}, 115:\penalty0 234--252, 2018.

\bibitem[Politis et~al.(1999)Politis, Romano, and Wolf]{Politis1999}
D.~N. Politis, J.~P. Romano, and M.~Wolf.
\newblock \emph{Subsampling}.
\newblock Springer New York, 1999.

\bibitem[Quiroga et~al.(2002)Quiroga, Kreuz, and Grassberger]{Quian2002}
R.~Quian Quiroga, T.~Kreuz, and P.~Grassberger.
\newblock Event synchronization: A simple and fast method to measure
  synchronicity and time delay patterns.
\newblock \emph{Physical Review E}, 66:\penalty0 9, 2002.

\bibitem[Radebach et~al.(2013)Radebach, Donner, Runge, Donges, and
  Kurths]{Radebach2013}
A.~Radebach, R.~V. Donner, J.~Runge, J.~F. Donges, and J.~Kurths.
\newblock Disentangling different types of el niño episodes by evolving
  climate network analysis.
\newblock \emph{Physical Review E}, 88:\penalty0 052807, 2013.

\bibitem[Rheinwalt et~al.(2012)Rheinwalt, Marwan, Kurths, Werner, and
  Gerstengarbe]{Rheinwalt2012}
A.~Rheinwalt, N.~Marwan, J.~Kurths, P.~Werner, and F.~Gerstengarbe.
\newblock Boundary effects in network measures of spatially embedded networks.
\newblock \emph{Europhysics Letters}, 100\penalty0 (2):\penalty0 28002, 2012.

\bibitem[Rheinwalt et~al.(2015)Rheinwalt, Goswami, Boers, Heitzig, Marwan,
  Krishnan, and Kurths]{Rheinwalt2015}
A.~Rheinwalt, B.~Goswami, N.~Boers, J.~Heitzig, N.~Marwan, R.~Krishnan, and
  J.~Kurths.
\newblock Teleconnections in climate networks: A network-of-networks approach
  to investigate the influence of sea surface temperature variability on
  monsoon systems.
\newblock \emph{Machine Learning and Data Mining Approaches to Climate
  Science}, 2015.

\bibitem[Romano et~al.(2018)Romano, Vinh, Verspoor, and Bailey]{Romano2018}
S.~Romano, N.~X. Vinh, K.~Verspoor, and J.~Bailey.
\newblock The randomized information coefficient: assessing dependencies in
  noisy data.
\newblock \emph{Machine Learning}, 107:\penalty0 509--549, 2018.

\bibitem[Runge et~al.(2019{\natexlab{a}})Runge, Bathiany, Bollt, Camps-Valls,
  Coumou, Deyle, Glymour, Kretschmer, Mahecha, Muñoz-Marí, van Nes, Peters,
  Quax, Reichstein, Scheffer, Schölkopf, Spirtes, Sugihara, Sun, Zhang, and
  Zscheischler]{Runge2019b}
J.~Runge, S.~Bathiany, E.~Bollt, G.~Camps-Valls, D.~Coumou, E.~Deyle,
  C.~Glymour, M.~Kretschmer, M.~D. Mahecha, J.~Muñoz-Marí, E.~H. van Nes,
  J.~Peters, R.~Quax, M.~Reichstein, M.~Scheffer, B.~Schölkopf, P.~Spirtes,
  G.~Sugihara, J.~Sun, K.~Zhang, and J.~Zscheischler.
\newblock Inferring causation from time series in {Earth} system sciences.
\newblock \emph{Nature Communications}, 10:\penalty0 1--13, 2019{\natexlab{a}}.

\bibitem[Runge et~al.(2019{\natexlab{b}})Runge, Nowack, Kretschmer, Flaxman,
  and Sejdinovic]{Runge2019}
J.~Runge, P.~Nowack, M.~Kretschmer, S.~Flaxman, and D.~Sejdinovic.
\newblock Detecting and quantifying causal associations in large nonlinear time
  series datasets.
\newblock \emph{Science Advances}, 5, 2019{\natexlab{b}}.

\bibitem[Scarsoglio et~al.(2013)Scarsoglio, Laio, and Ridolfi]{Scarsoglio2013}
S.~Scarsoglio, F.~Laio, and L.~Ridolfi.
\newblock Climate dynamics: A network-based approach for the analysis of global
  precipitation.
\newblock \emph{PLoS ONE}, 8, 2013.

\bibitem[Schreiber and Schmitz(1996)]{Schreiber1996}
T.~Schreiber and A.~Schmitz.
\newblock Improved surrogate data for nonlinearity tests.
\newblock \emph{Physical Review Letters}, 77:\penalty0 635, 1996.

\bibitem[Shao and Tu(1995)]{Shao1995}
J.~Shao and D.~Tu.
\newblock \emph{The Jackknife and Bootstrap}.
\newblock Springer New York, 1995.

\bibitem[Sporns(2010)]{Sporns2010}
O~Sporns.
\newblock \emph{{Networks of the Brain}}.
\newblock The MIT Press, 2010.

\bibitem[Stein(1999)]{Stein1999}
M.~L. Stein.
\newblock \emph{Interpolation of Spatial Data}.
\newblock Springer New York, 1999.

\bibitem[Tsonis and Roebber(2004)]{Tsonis2004}
A.~A. Tsonis and P.~J. Roebber.
\newblock The architecture of the climate network.
\newblock \emph{Physica A: Statistical Mechanics and its Applications},
  333:\penalty0 497--504, 2004.

\bibitem[Tsonis et~al.(2008)Tsonis, Swanson, and Wang]{Tsonis2008}
A.~A. Tsonis, K.~L. Swanson, and G.~Wang.
\newblock On the role of atmospheric teleconnections in climate.
\newblock \emph{Journal of Climate}, 21:\penalty0 2990--3001, 2008.

\bibitem[Vallis(2011)]{Vallis2011}
G.~Vallis.
\newblock \emph{Climate and the Oceans}.
\newblock Princeton University Press, 2011.

\bibitem[von Luxburg(2007)]{Luxburg2007}
U.~von Luxburg.
\newblock A tutorial on spectral clustering.
\newblock \emph{Statistics and Computing}, 17\penalty0 (4):\penalty0 395--416,
  2007.

\bibitem[Wiedermann et~al.(2015)Wiedermann, Donges, Kurths, and
  Donner]{Wiedermann2015}
M.~Wiedermann, J.~F. Donges, J.~Kurths, and R.~V. Donner.
\newblock Spatial network surrogates for disentangling complex system structure
  from spatial embedding of nodes.
\newblock \emph{Physical Review E}, 93, 2015.

\bibitem[Yamasaki et~al.(2008)Yamasaki, Gozolchiani, and Havlin]{Yamasaki2008}
K.~Yamasaki, A.~Gozolchiani, and S.~Havlin.
\newblock Climate networks around the globe are significantly affected by el
  niño.
\newblock \emph{Physical Review Letters}, 100:\penalty0 228501, 2008.

\end{thebibliography}

\newpage
\pagenumbering{roman}
\setcounter{page}{1}


\begin{appendices}

\listofappendices

\section{Introduction to Climate Network Terminology and Methodology}
\label{sec:intro_cn_terminology}

Here we introduce all concepts and methods around climate networks that appear in the main text. We offer both definitions as well as brief explanations for the application in climate science.

\subsection{Networks}

An undirected \textbf{network} $(V,E)$ consists of a set of \textbf{nodes/vertices} $V$ with $p$ elements and a set of undirected \textbf{links/edges} $\{v_i,v_j\}$ connecting pairs of nodes $v_i,v_j\in V$. We exclude self-loops $\{v_i,v_i\}$. While for undirected edges $\{v_i,v_j\}$ the ordering of the nodes does not matter, a directed edge $(v_i,v_j)$ starts at $v_i$ and ends at $v_j$, but we only consider undirected ones. In a \textbf{weighted network} each edge $\{v_i,v_j\}$ has a corresponding weight $w_{i,j}$; in an \textbf{unweighted network} all edges simply have weight $w_{i,j}=1$. When there is no edge between $v_i$ and $v_j$, we write $w_{i,j}=0$. In this way, a network can be respresented by its \textbf{adjacency matrix} $W= (w_{i,j})_{i,j=1,\dots,p}$.

While climate networks are commonly constructed by including the edges with the largest similarity values globally, there is an alternative, which is particularly popular in machine learning \citepA{Luxburg2007}.

\textbf{(Symmetric) $k$-nearest neighbor graphs} (short kNN graphs) include the $k$ edges with largest similarity for each node. This ensures that each node is connected to at least $k$ other nodes. This is useful when some locations have much lower characteristic similarities than others, but these locations should not be ignored by the network.

An alternative network construction procedure that is popular in climate science is to include edges $\{i,j\}$ which are significant in some sense. For this purpose the time series are reshuffled several times either randomly, or preserving certain time series characteristics. This yields a baseline distribution of similarity on each edge. Now one can include an edge, when the similarity $\hat{S}_{ij}$ of the original time series exceeds a restrictive quantile of the baseline distribution like $0.99$ or even $0.999$. Another popular approach is to calculate the mean $\hat{\mu}_{ij}^0$ and standard deviation $\hat{\sigma}_{ij}^0$ of the baseline distribution, calculate z-scores $\frac{\hat{S}_{ij}-\hat{\mu}_{ij}^0}{\hat{\sigma}_{ij}^0}$ and construct a density-threshold network from these z-scores.

Besides eyeballing the network, there are several quantities that measure different properties of the networks and therefore potentially help to understand the network structure and discover interesting behaviour. Let us introduce the most important ones for undirected similarity networks.

The \textbf{degree} $d_i$ of a node $v_i$ is the cumulative weight of its connections $d_i = \sum_{i=1}^p w_{i,j}$. It measures the importance of the node in the sense of how well connected it is in the network. A node of exceptionally large degree is called hub and often influences many other nodes.

The \textbf{(local) clustering coefficient} $C_i$ of a node $v_i$, is given by $C_i= \frac{1}{d_i (d_i-1)} \sum_{j,k=1}^p w_{i,j} w_{j,k} w_{k,i}$. It measures how many neighbors of node $v_i$ are directly connected to each other by an edge. It is $0$ when no neighbor of $v_i$ has an edge to any other neighbor of $v_i$, and it is $1$ when all neighbors of $v_i$ are interconnected as well. A set of nodes that is highly intraconnected but less connected to other nodes is often called a 'cluster' or community of nodes. 
A popular approach to calculate a \textbf{global clustering coefficient} is to simply average all local clustering coefficients \citepA{Watts1998}. This quantity is dominated by nodes with low degree. An alternative measure for undirecte networks which represents the global network structure more reasonably \citepA{Newman2018} is given by $C=\frac{\sum_{i,j,k} W_{ij}W_{jk}W_{ki}}{\sum_i d_i(d_i-1)}$.

The \textbf{weighted clustering coefficient} \citepA{Onnela2005} reflects how large triangle weights are compared to the network maximum. While there are other generalizations, this is the one implemented in the popular brain connectivity toolbox \citepA{Rubinov2010}. In recent work, \citetA{Wang2017,Fardet2021} point out that other definitions are more desirable for networks with highly heterogeneous weight distribution or numerous spurious edges with low weights.

The \textbf{shortest path length} from $v_i$ to $v_j$ quantifies the minimal number of traversed edges to move from $v_i$ to $v_j$. The shortest path length at $v_i$ is the average shortest path length from $v_i$ to any other node. Transport networks are efficient when they have small shortest path lengths. Neuroscientists often point out small-world behaviour in functional brain networks as a justification for the brain's efficiency \citepA{Bullmore2012}, robustness \citepA{LevitBinnun,Neal2017} and compositionality. Small-worldness is characterized by large clustering coefficients and low shortest path lengths. Treatments of small-worldness should be handled with care \citepA{Bialonski2010,Hlinka2017}.

The \textbf{betweenness centrality} of a node $v_i$ measures the centrality of $v_i$ in terms of how many shortest paths in the network contain $v_i$. Let $\sigma_{j,k}$ denote the number of shortest paths from $v_j$ to $v_k$ in the network, and $\sigma_{j,k}(v_i)$ the number of these paths that also contain $v_i$. The betweenness centrality is then given by $\sum_{j,k\neq i}\frac{\sigma_{j,k}(v_i)}{\sigma_{j,k}}$. Similar to shortest path length and small worldness, we find that betweenness can be heavily distorted both locally as well as globally by single false edges.

\textbf{Ricci curvature} of an edge describes how the connectivity of its network neighborhood differs from the
connectivity of a regular grid. Positive curvature indicates a highly intraconnected neighborhood or a within cluster edge, and negative curvature indicates that an edge connects different communities \citepA{Strnad2022}. Forman-Ricci curvature \citepA{Forman2003} and Ollivier-Ricci curvature \citepA{Ollivier2010} are two numerical approaches to approximate the Ricci curvature of an edge, where Forman-Ricci curvature is computationally cheaper while Ollivier-Ricci curvature is more accurate.

In climate networks, the nodes usually represent locations. Hence our network is naturally embedded in 3D space and edges have a length given by the geodesic distance between the two endpoints. In traffic networks there is a cost associated with forming edges, so that the total link length should be minimized in the network design. In similarity networks, closeby points are generally similar. This local connectivity is the baseline behaviour when analysing spatio-temporal data. Of particular interest are connected nodes that are far apart. Those indicate climatic teleconnections. An important summary statistic of the spatial properties of a network is its \textbf{link length distribution}, which subsumes the lengths of all edges in the network. Given the roughly spherical geometry of the Earth, in an infinitesimally fine grid, the number of potential links scales as $\sin(d)$ with link length $d$ in radians. When dividing the link length distribution by $\sin(d)$ (or, for a finite grid, by the number of potential links at each distance), we see how many of the potential links are formed at a given distance. For spatiotemporal data, we expect values close to $1$ for small lengths and values close to $0$ for large lengths. Random link distribution results in a roughly uniform distribution across lengths, hence a sinusoidal distribution of links in the original link length distribution. Crucially, the link length distribution does not convey any information about link bundling behaviour (cf. Fig. 16). An extensive discussion of link bundles can be found in Section 3c and supplemental Section \ref{sec:linkbundles}.

We call a grid \textbf{isotropic} when the point density is constant over the sphere, or alternatively when the distance distribution from one grid point to all others does not depend on the considered grid point. While a perfectly isotropic grid does not exist for the sphere, there exist approximately isotropic grids, like the one we generate with the Fekete algorithm \citepA{Bendito2007} or the Fibonacci grid. In a \textbf{Gaussian grid}, the grid points are equally spaced across latitudes. In a regular Gaussian grid the number of gridpoints is constant across latitudes. While they traditionally are the most popular grid type, regular Gaussian grids have diverging densities around the poles and can therefore induce complex distortions when used in climate networks.


\subsection{Similarity Measures and Estimators}

Most commonly climate and neuroscientists calculate the cross-correlation between time series $\{X_t\}_{t=1,\dots,n}$ and $\{Y_t\}_{t=1,\dots,n}$ of length $n$, defined as $\frac{1}{n}\sum_{t=1}^n X_t Y_t$. Statistically, this is imprecise terminology: While they \textit{estimate} the Pearson (cross-)correlation between the time series with the \textbf{\textit{empirical} Pearson correlation estimator} $\frac{1}{n}\sum_{t=1}^n X_t Y_t$, there are other estimators that might result in estimates that come closer to the ground truth values we aim to approximate.
For high-dimensional problems, the empirical correlation matrix is a very bad estimator: the eigenvalue spectrum of empirical correlation matrices is systematically distorted, following the Marcenko-Pastur law \citepA{Marcenko1967}. A long history of correlation estimation research \citepA{Fan2015,Pourahmadi2013} has proposed various approaches such as banding, tapering, Lasso-regularized maximum likelihood and eigenvalue shrinkage to correct these empirical distortions. One particularly popular estimator is the Ledoit-Wolf estimator \citepA{Ledoit2004}, which has a recent non-linear advancement \citepA{Ledoit2020} claimed to accommodate dimensions up to $p=10^4$. In Figure 3c we find that the linear Ledoit-Wolf estimator approximates the correlation matrix much better than empirical Pearson correlation. The \textbf{linear Ledoit-Wolf estimator} is defined as $\rho_1 \hat{S}_n + \rho_2 I_d,$ where $\hat{S}_n$ is the empirical covariance matrix, $I_d$ the identity matrix and the optimal choice of $\rho_1,\rho_2 \in \IR$ can be found in \citetA{Ledoit2004}. Estimators from spatial statistics use the spatial information of the nodes and are therefore more appropriate for random field data \citepA{Li2009,Clifford1989}. In most geostatistical models only one measurement can be made at each location; exceptions include maximum likelihood estimation of space-time variograms \citepA{Bevilacqua2012} and nearest neighbor Gaussian process models \citepA{Datta2016}. Designing anisotropic spatial correlation estimators for finely gridded time series in spherical geometry remains an important direction of future work \citepA{Chen2021,Raymaekers2021} for climate network and EOF construction. Recently, estimation of partial correlations in Bayesian or causal network approaches \citepA{Zerenner2014,Runge2019b,Gerhardus2020} has gained popularity, with the goal of distinguishing direct from indirect influences. The precision matrix can be estimated to recover conditional independences. The same estimation difficulties arise in this context in high dimension and suitable estimators should be carefully chosen \citepA{Lam2020}.

Spearman's rank correlation coefficient, short \textbf{Spearman correlation}, is another similarity measure, which measures not only linear dependence structures, but all monotonous dependences. It is given by the Pearson correlation between the time series of ranks $\{R(X_t)\}_{t=1,\dots,n}$ of $\{X_t\}_{t=1,\dots,n}$ and $\{R(Y_t)\}_{t=1,\dots,n}$ of $\{X_t\}_{t=1,\dots,n}$. It is invariant under monotonic transformations of the data. A measure of rank correlation with better finite-sample properties is given by Kendall's $\tau$ \citepA{Gilpin1993}.

\textbf{Mutual information} between two random variables $X$ and $Y$ is given by \[
MI(X,Y)=KL(P^{X,Y}, P^X \times P^Y),\]
where $KL$ denotes the Kullback-Leibler divergence, $P^{X,Y}$ denotes the joint distribution of $X$ and $Y$, and $P^X \times P^Y$ the joint distribution of $X$ and $Y$ if they were independent. It measures to which extent the two random variables are dependent. Instead of only measuring linear or monotonic relationships, mutual information detects any kind of dependency. Traditionally, binning estimators of mutual information have been popular, but these are highly sensible to the number of bins and often heavily biased \citepA{Paninski2003}. The KSG estimator \citepA{Kraskov2003} (based on a kNN approach) has been observed to be robust to the choice of its hyperparameter k \citepA{mieval}. Extensions of the KSG estimator include a bias-corrected version \citepA{gao2016demystifying} and an adaptation to discrete-continuous mixtures \citepA{ksgmixed} such as precipitation. Without distributional assumptions there exists no convergence rate that can be guaranteed in MI estimation \citepA{Antos2001}, so we still have to expect unprecise estimates. Especially for highly autocorrelated time series it is known that a lot of false positives \citepA{Runge2012,Runge2014} will arise. While evaluating the robustness of empirical MI networks, it is also worth considering other similarity measures that capture arbitrary dependences. \citetA{Romano2018} compares various information theory- and correlation-based measures. The alternative we employed is the \textbf{Hilbert Schmidt Independence Criterion} (HSIC) for random processes \citepA{Chwialkowski2014}. It measures dependence after a mean embedding of the random variables in some reproducing kernel Hilbert space. For more details see also \citetA{Gretton2012}.

\subsection{Link Bundle Definition}
\label{sec:linkbundles}

With the term `link bundle' we want to capture the case when most links are formed between two regions. While there is a lot of work on the visualisation of graphs with focus on edge bundles \citepA{Lambert2010, Hurter2012,Selassie2011, Nocke2015}, our purpose does not demand a 2D-visualisation. We are simply interested in the cumulative edge weight between two regions of interest. As we have not found a definition of link bundles in the context of spatially embedded graphs in the literature, we define our notion of link bundle over spatial graphs $(V,E)$ with adjacency matrix $A$, where $(V,d)$ is a metric space, as follows.

First define the $\varepsilon$-ball around $v$ as $B_\varepsilon(v):= \{ u\in V | d(u,v)\le \varepsilon \}$. Let $\rho_\eps(v_i,v_j)$ be the number of edges between $B_\varepsilon(v_i)$ and $B_\varepsilon(v_j)$ in the complete graph. It holds that $\rho_\eps(v_i,v_i)=\binom{|B_\eps(v_i)|}{2}$. When $B_\varepsilon(v_i)\cap B_\varepsilon(v_j)=\emptyset$, then $\rho_\eps(v_i,v_j)= |B_\varepsilon(v_i)|\cdot |B_\varepsilon(v_j)|$. We denote the cumulative weights between neighborhoods of $v_i$ and $v_j$ by
\[
W_\eps(v_i,v_j) := \sum_{k,l:\: v_k\in B_\varepsilon(v_i), v_l\in B_\varepsilon(v_j)} |A_{kl}|,
\]
where $A$ denotes the adjacency matrix. Our notions of link bundles measure whether the mean weights between regions (per possible link) exceed a prespecified threshold. Whether the graph is weighted or unweighted makes no difference in the definition. Of course weights allow for a more fine-grained differentiation of link bundle strength (Fig. 8). We allow for a tolerance ($1-c$ for unweighted graphs) of unformed edge weight as link bundles of a certain density might already be considered significant. Demanding complete link bundles might be unrealistic and sensitive to noise in practical settings.

A first idea might be a one-to-many notion, where we demand that a node $v_i$ is connected to all nodes around a point $v_j$ within a given radius $\varepsilon$. We say that there is a \textbf{$(\varepsilon,c)$-one-to-many link bundle} from $v_i$ to $v_j$ in $(V,E,A)$, if
\[
\frac{1}{|B_\varepsilon(v_i)|}\sum_{k: \: v_k\in B_\varepsilon(v_i)} |A_{kj}| \geq c.
\]
We observe such one-to-many link bundles in networks dominated by extreme events and more generally when one region is locally correlated and the other point spuriously aligns with that region while differing from its surrounding. More natural might be to consider neighborhoods around both endpoints of the link bundle. We define a \textbf{$(\varepsilon,c)$-many-to-many link bundle} between $v_i$ and $v_j$ in $(V,E,A)$, if
\[
\frac{W_\eps(v_i,v_j)}{\rho_\varepsilon(v_i,v_j)} \geq c.
\]
In applications with anisotropic link density across space, we might want to consider teleconnections to be significant if the link density between neighborhoods is large compared to the density within the neighborhoods.
There is a \textbf{$(\varepsilon,c)$-locally-weighted many-to-many link bundle} between $v_i$ and $v_j$ in $(V,E,A)$, if
\[
\frac{W_\eps(v_i,v_j)}{\rho_\varepsilon(v_i,v_j)} \geq \frac{c}{2}\left(\frac{W_\eps(v_i,v_i)}{\rho_\varepsilon(v_i,v_i)} + \frac{W_\eps(v_j,v_j)}{\rho_\varepsilon(v_j,v_j)} \right).
\]
The notion of locally weighted link bundles is also sensible to significant teleconnections between endpoints of low link density in settings of anisotropic density. Other utilisation of intra-regional link densities is conceivable. A locally weighted version of 1-to-many link bundles can be defined analogously. In order to find link bundles in practice, we propose the distance metric between two edges $(x_1,y_1)$ and $(x_2,y_2)$, where $x_1, y<1,x_2,y_2$ are nodes,

\begin{align}
    d_E\left((x_1,y_1),(x_2,y_2)\right) = \min\left\{d(x_1,x_2) + d(y_1,y_2),\quad d(x_1,y_2) + d(y_1,x_2)\right\} \label{eq:dphi}
\end{align}
in the undirected case. For the directed case, only take $d_E\left((x_1,y_1),(x_2,y_2)\right) = d(x_1,x_2) + d(y_1,y_2)$. In this distance two edges are close in feature space if and only if both ends of the edges are close together. Clusters found with for example hierarchical clustering over the distance matrix between edges would then constitute link bundles.

Using the weighted notions of link bundle, $c$ has to be chosen smaller to detect link bundles, when the similarity values are typically smaller than $1$. The locally weighted notion of link bundle requires less adaptation of $c$ since it captures the edge weights characteristic to the measured random field. Analysing the edge weights in our setting indicates that the spuriously included links only lie marginally above the threshold and can therefore be distinguished from strong links. However, a ground-truth teleconnection of intermediate strength might still be difficult to distinguish from spurious bundles. While the notions of weighted link bundles offer a more fine-grained differentiation, questions of significance have to be answered differently. Although some spurious locally weighted many-to-many link bundles occur in the weighted networks, they are a lot less frequent under suitably chosen $c$. For large densities, the notions of locally weighted and absolute many-to-many link bundles become equivalent in unweighted networks, as the local neighborhoods become completely connected. The notion of locally weighted link bundle becomes problematic in sparse graphs, when the neighborhood connectivity structure is not sufficiently established.

\section{Details of the Stochastic Ground Truth Model for Spatio-Temporal Data}
\label{sec:migrf_details}

\subsection{Time Dependence via Vector Autoregression}\label{sec:time_dep}
As the typically used data sets only resolve a fixed time resolution, we do not have to employ the computationally expensive space-time covariance functions, but opt for a vector autoregression VAR(1). As introduced in Algorithm 1, we denote the available measurements $(X_{1, t},\dots,X_{p, t})\sim N(0, \Sigma)$ of the MIGRF at time $t$ evaluated on the finite grid $\{v_i\}_{i\in [p]}\subset S^2$ with $\Sigma_{ij}=k(|v_i-v_j|)$. We generate $X_{i, t}$, for $i\in [p]$, via
$$ X_{i, t} = a_i \cdot X_{i, t-1} + \varepsilon_{i,t}, $$
where the innovations $(\varepsilon_{1,t},\dots,\varepsilon_{p,t}) \sim N(0,\Sigma^\varepsilon)$ are i.i.d. in time and $a_i$ with $|a_{i}|<1$ denotes the lag-1 autocorrelation at node $v_i$. What is left to do is to design an innovation covariance matrix $\Sigma^\varepsilon$ that generates the desired covariance structure $\Sigma$ of the MIGRF. A well-defined $\Sigma^\varepsilon$ does not exist for all choices of $\Sigma$ and $a$. Whenever it exists, it is given by $\Sigma^\varepsilon_{ij} = \Sigma_{ij} (1-a_{i}\cdot a_{j})$. This formula  can be derived by defining $A\in\IR^{p\times p}$ as the diagonal matrix with $A_{ii}=a_i$ and using
$$\begin{pmatrix}
  X_{1,t}\\ 
  X_{2,t}\\
  \vdots\\
  X_{p,t}
\end{pmatrix} = \sum_{k=0}^\infty A^k \begin{pmatrix}\varepsilon_{1, t-k}\\ \varepsilon_{2, t-k}\\ \vdots\\ \varepsilon_{p, t-k} \end{pmatrix},$$
which yields,
$$\begin{pmatrix}
  X_{1,t}\\ 
  X_{2,t}\\
  \vdots\\
  X_{p,t}
\end{pmatrix}\sim N\left(0, \sum_{k=0}^\infty A^k \Sigma^\varepsilon A^k \right)= N\left(0,\Bigl(\frac{\Sigma^{\varepsilon}_{ij}}{1-A_{ii}\cdot A_{jj}} \Bigr)_{ij}\right)\overset{!}{=}N(0,\Sigma).~$$
Solving for $\Sigma^\varepsilon$ then leads to the desired formula. 

It is easy to see that when all nodes $v_i$ have the same autocorrelation, $\Sigma^\varepsilon$ is a well-defined covariance matrix. In other cases, the above defined matrix $\Sigma^\varepsilon$ might have negative eigenvalues, making it unsuitable as a covariance matrix. As a remedy, the negative eigenvalues of $\Sigma^\varepsilon$ could be shifted to a small positive constant to make it positive definite. This procedure induces a shift in the spatial covariance $\Sigma$, that would need to be quantified. The specified autocorrelation structure remains exact by design. In our experiments in Section 3d, we initialize a random half of the points with $A_{ii}=0.2$ and the other half with $A_{ii}=0.7$. The resulting $\Sigma^\varepsilon$ is not positive definite, but we find a valid correlation matrix by shifting all negative eigenvalues of $\Sigma^\varepsilon$ to $10^{-8}$. This shift does not introduce a notable bias into the spatial ground truth correlations.

\subsection{Mat\'ern Covariance} \label{sec:matern}

The Mat\'ern covariance function appears in many fields \citepA{Whittle1954,Guttorp2006}. A random process on $\IR^d$ with Mat\'ern covariance function is a solution to the stochastic partial differential equation \citepA{Lindgren2011} \[
(\kappa^2 - \Delta)^{\alpha/2} G(x) = \phi W(x),
\]
where $W(x)$ is Gaussian white noise with unit variance, $\Delta=\sum_{i=1}^d \frac{\partial^2}{\partial x_i^2}$ is the Laplace operator, and $\alpha=\nu+d/2$. Therefore it regularly demonstrates its efficacy of modelling physical processes in spatial statistics \citepA{Bevilacqua2020,Stein2011}, and can be adjusted to model non-stationary, non-isotropic, oscillating and non-separable random fields \citepA{Lindgren2011}.

On $\IR^D$, it is given by \[
C_{\nu,\ell}(d)=\sigma^2 \frac{2^{1-\nu}}{\Gamma(\nu)}\left( \sqrt{2\nu} \frac{d}{\ell}\right)^\nu K_\nu \left(\sqrt{2\nu} \frac{d}{\ell}\right),
\]
where $\Gamma$ is the gamma function and $K_\nu$ is the modified Bessel function of second kind. Its smoothness parameter $\nu$ and scale parameter $\ell$ make it flexible as well as interpretable. The scale parameter $\ell$ determines how far the region of high correlation extends spatially. $\nu>0$ determines the smoothness of the random field; so how well the random field behaves locally. Large $\ell$ and large $\nu$ result in less fluctuation across space. A Gaussian process with Mat\'ern covariance is $\lfloor \nu \rfloor-1$ times differentiable in the mean-square sense. When $\nu=n+1/2$ for some $n\in\IN$, $C_{\nu,\ell}$ can be written as a product of an exponential and a polynomial of order $n$. For $\nu\to\infty$, it approximates the squared exponential function. So by varying $\nu$, we can interpolate between the absolute exponential kernel and the Gaussian radial basis function (\citealp{Stein1999}, ch. 2.10).

Unfortunately, the Mat\'ern is positive definite with great circle distance only if $\nu\le 1/2$. \citetA{Guinness2016} studies the use of several Mat\'ern-like covariance functions on the sphere and finds that they all fit smooth and non-smooth meteorological datasets more flexibly than alternative classes of covariance functions which respect to spherical geometry. We follow their recommendation in using the computationally efficient \textit{chordal Mat\'ern}. The chordal Mat\'ern is simply the restriction of the Mat\'ern in $\IR^{3}$ to $S^2$. An adaptation to great circle distance, called \textit{circular Mat\'ern}, shows no practical gain. The authors find maximum likelihood estimates for $\nu$ close to $0.42$ for ozone data and close to $1.46$ for $10$m height surface temperature outputs from a single run of the Community Climate System Model Version 4 (CCSM4). We employ realistic values of $\nu=0.5$ and $\nu=1.5$ in our simulations, being mostly interested in the effect of varying $\nu$ and $\ell$.

In practice, measurements from a random field are only available on a fixed, finite grid. Instead of sampling the entire random field via the Karhunen Loeve expansion \citepA{Lang2015}, we can simply consider the multivariate Gaussian on the grid points, at each time step.

\subsection{Choice of Smoothness $\nu$ and Length Scale $\ell$}\label{sec:nu_ell}

We choose $\nu$ and $\ell$ to reflect realistic values for climatic time series as well as to point out their influences on the estimation procedure. Our choices of $\nu\in\{0.5,1.5\}$ and $\ell\in\{0.1,0.2\}$ (in radians) reflect a realistic range for climatic variables \citepA{Guinness2016}. Stronger smoothness only induces subtle differences. Table \ref{table:nuellchoice} shows network-based smoothness and length scale properties of the MIGRF as well as climatic data.

 Observe that our parameter choices cover a realistic range. Given similar smoothness and length scale, the network density on simulated data is smaller as systematic teleconnections are missing.

\begin{table}[H]
\small
\centering
\begin{tabular}{ |c||c|c|c|c||c|c|c|c|c| }
\hline
\multirow{3}{*}{\minitab[c]{Random field\\ property}} & \multicolumn{4}{c||}{Simulated data} & \multicolumn{5}{c|}{Real data}\\
 & \multirow{2}{*}{\minitab[c]{\hspace{-2mm}$\nu=0.5$,\hspace{-2mm}\\ \hspace{-2mm}$\ell=0.1$\hspace{-2mm}}\hspace{-2mm}} & \multirow{2}{*}{\minitab[c]{\hspace{-2mm}$\nu=0.5$,\hspace{-2mm}\\ \hspace{-2mm}$\ell=0.2$\hspace{-2mm}}\hspace{-2mm}} & \multirow{2}{*}{\minitab[c]{\hspace{-2mm}$\nu=1.5$,\hspace{-2mm}\\ \hspace{-2mm}$\ell=0.1$\hspace{-2mm}}\hspace{-2mm}} & \multirow{2}{*}{\minitab[c]{\hspace{-2mm}$\nu=1.5$,\hspace{-2mm}\\ \hspace{-2mm}$\ell=0.2$\hspace{-2mm}}\hspace{-2mm}} & \multirow{2}{*}{\minitab[c]{pr}} & \multirow{2}{*}{\hspace{-2mm}\minitab[c]{dt2m}\hspace{-2mm}} & \multirow{2}{*}{\minitab[c]{t2m}} & \multirow{2}{*}{\minitab[c]{sp}} &  \multirow{2}{*}{\hspace{-2mm}\minitab[c]{z500}\hspace{-2mm}}  \\
 & & & & & & & & &  \\
 \hline
 Avg. local corr. & 0.544 & 0.734 & 0.704 & 0.892 & 0.520 & 0.753 & 0.826 & 0.946 & 0.953 \\
 Decorr. len., $\tau=0.2$ & 0.172 & 0.348 & 0.190 & 0.383 & 0.139 & 0.271 & 0.329 & 0.571 & 0.453 \\
 Decorr. len., $\tau=0.5$ & 0.076 & 0.152 & 0.104 & 0.216 & 0.059 & 0.141 & 0.183 & 0.346 & 0.305 \\
 Density, $\tau=0.2$ & 0.031 & 0.057 & 0.031 & 0.057 & 0.009 & 0.031 & 0.093 & 0.150 & 0.246 \\
 Density, $\tau=0.5$ & 0.001 & 0.005 & 0.002 & 0.010 & 0.002 & 0.006 & 0.013 & 0.047 & 0.117 \\
 \hline
\end{tabular}
\caption{Distributional smoothness and length scale properties of simulated and real climatic random field data. For the \textit{average local correlation} we compute the average correlation inside $\eps$-balls around all vertices. We choose the radius $\eps=5^\circ$ as in Fig. 8. For the \textit{network density} we construct unweighted Pearson correlation networks with threshold $\tau$. Our notion of \textit{decorrelation length} measures the length scale as a network feature. For each node, it is given by the minimal radius (in radians) at which the connectivity of the network inside an $\eps$-ball drops below $c=0.8$ (smaller than $1$ for robustness), finally average over all nodes. The table shows average values over 30 independent realizations. For real data we construct a single network from all observations.}
\label{table:nuellchoice}
\end{table}

\section{Why Spurious Links Occur in Bundles}
\label{sec:false_edges}

The following proposition formalizes why one spuriously large similarity estimate incurs correlated similarity estimates to be spuriously large as well, with non-negligible probability. We assume Gaussianity for simplicity. $\tau$ represents the threshold above which a link is included in the network.

\begin{prop}[\textbf{Link probability given the presence of another link}] \label{prop2}
Let $\hat{S}_1,\hat{S}_2 \sim N(\mu,\Sigma)$ be estimates of similarity values on two edges with variances $\sigma_1^2=\Sigma_{11},$ $\sigma_2^2=\Sigma_{22}$ and correlation $\rho=\frac{\Sigma_{12}}{\sigma_1 \sigma_2}$. Then, with $\eps\in\IR$,
\[
P\Big(\hat{S}_1>\tau \;\;\Big|\;\; \hat{S}_2 > \tau + \eps\Big) \geq \Phi\left( \frac{\frac{\sigma_1}{\sigma_2}\rho (\tau+\eps-\mu_2) +\mu_1 - \tau}{\sqrt{1-\rho^2} \sigma_1} \right),
\]
where $\Phi$ is the cumulative distribution function of the standard normal distribution.
\end{prop}

The resulting probability becomes arbitrarily large when $\hat{S}_2$ becomes large (by increasing $\eps$) and the correlation $\rho$ between the edge estimates is large. This explains both the occurence of link bundles as well as spuriously dense regions in the empirical networks.

\begin{proof}
For any $s\in\IR$, it holds that $\hat{S}_1|\hat{S}_2 = s \sim N\Big(\tilde{\mu}(s),\; \tilde{\sigma}^2\Big)$ with $\tilde{\mu}(s):=\mu_1 + \frac{\sigma_1}{\sigma_2}\rho (s-\mu_2)$ and $\tilde{\sigma}^2:=(1-\rho^2)\sigma_1^2$. Then, since $\left(\frac{\hat{S}_1-\tilde{\mu}(\tau+\eps)}{\tilde{\sigma}^2} \;\;\Big|\;\; \hat{S}_2 > \tau+\eps\right) \sim N(0,1)$ and with $\Phi(x)=1-\Phi(-x)$ for all $x\in\IR$, we get,

\begin{align*}
    P\Big(\hat{S}_1>\tau \;\;\Big|\;\; \hat{S}_2 > \tau + \eps\Big) &\geq \inf_{s\geq \tau+\eps} P\Big(\hat{S}_1>\tau \;\;\Big|\;\; \hat{S}_2 > s\Big)= P\Big(\hat{S}_1>\tau \;\;\Big|\;\; \hat{S}_2 > \tau+\eps\Big)\\
    &=P\Big(\frac{\hat{S}_1-\tilde{\mu}(\tau+\eps)}{\tilde{\sigma}^2}>\frac{\tau-\tilde{\mu}(\tau+\eps)}{\tilde{\sigma}^2} \;\;\Big|\;\; \hat{S}_2 > \tau+\eps\Big)\\
    &=1-\Phi\Big(\frac{\tau-\tilde{\mu}(\tau+\eps)}{\tilde{\sigma}^2}\Big)=\Phi\Big(\frac{\tilde{\mu}(\tau+\eps)-\tau}{\tilde{\sigma}^2}\Big)\\
    &=\Phi\left( \frac{\frac{\sigma_1}{\sigma_2}\rho (\tau+\eps-\mu_2) +\mu_1 - \tau}{\sqrt{1-\rho^2} \sigma_1} \right).
\end{align*}
\end{proof}

For a more involved theoretical analysis, a realistic assumption might be that the correlation structure varies across the metric space and there can be ground truth teleconnections of intermediate strength, but data over nearby points universally shows the highest correlations. In such generality, relevant for the correctness of the estimation procedure is still only the joint distribution of edge weight estimates $\{\hat{S}_e | e \in E \}$. Future work could analyse this joint distribution more closely with respect to the resulting graph features, particularly spurious link bundles and expected false discovery rate.

\section{The importance of the number of effective samples}\label{sec:effective_samples}

The number of time steps $n$ has a dominating effect on errors in the network. With large enough time length $n\to\infty$, eventually all random patterns disappear, as the similarity estimates on each edge become more precise. However, the crucial quantity that determines the estimation variance is not $n$ itself, but the \textit{effective} time length $n_{\text{eff}}$, which also depends on the autocorrelation structure. Let us introduce this quantity formally and quantify the relation between $n$ and $\nef$.

Given two time series $\{X_t\}_{t=1,\dots,n}$ and $\{Y_t\}_{t=1,\dots,n}$ of length $n$, their empirical Pearson correlation (often called cross-correlation) is given by \[
\hat{C}_{xy}=\frac{1}{n}\sum_{t=1}^n X_t Y_t.\]

Now $\hat{C}_{xy}$ does not perfectly estimate the true underlying Pearson correlation between the time series, but has some positive variance. Hence, we define the effective time length $\nef$ as the minimal number of independent observations $\{X'_i\}$ and $\{Y'_i\}$ such that the empirical cross correlation estimator $\frac{1}{\nef}\sum_{i=1}^{\nef} X'_i Y'_i$ has at most the same variance as it has under the autocorrelation structure of the given time series of length $n$.

Below, under some assumptions, we derive the relation \begin{align}
n=\sigma_\infty \cdot \nef, \label{eq:neff}
\end{align}
where $\sigma_\infty$ denotes the asymptotic variance of empirical cross correlation in the central limit theorem. Now observe that $\sigma_\infty$ explodes when the autocorrelation of both time series is large: For two independent AR(1)-processes with variance 1 and lag-1 autocorrelation $\alpha\in(0,1)$ and $\beta\in(0,1)$, respectively, it holds that \[
\sigma_\infty =\sum_{k=-\infty}^\infty Cov(X_1, X_{1+k}) Cov(Y_1,Y_{1+k}) = 1 + 2 \sum_{k=1}^\infty \alpha^k \beta^k = 1 + 2\frac{\alpha \beta}{1-\alpha \beta}.
\]
Note that for $\alpha\to 1$ and $\beta\to 1$, we get $\sigma_\infty\to\infty$. Intuitively, the given time series are essentially a single observation, when the autocorrelation approaches $1$; and correlation can not be estimated with a single observation. Given that the lag-1 autocorrelation patterns of real climatic variables range from negative values to values close to 1 (Fig. 12), the effective length of climatic time series highly depends on the location. Concretely, two AR(1)-time series with lag-1 autocorrelation $\alpha=\beta=0.9$ induce an increased estimation variance of $\sigma_\infty=9.53$, while $\alpha=\beta=0.95$ already results in $\sigma_\infty= 19.51$. Conversely, given $40$ years of monthly observations $n=480$, with the above autocorrelation patterns only yields $\nef=50.39$ and $\nef=24.60$ effective samples, respectively.

\textbf{Derivation of equation (\ref{eq:neff}).} The central limit theorem for the empirical cross-correlation between two independent, normalized stationary time series $\{X_t\}$ and $\{Y_t\}$ (cf. \citep[p. 236ff]{Brockwell1991}) states that, under structural assumptions, $\sqrt{n}\left( \frac{1}{n}\sum_{t=1}^n X_t Y_t \right)$ is asymptotically normal with mean $0$ and variance \[
\sigma_\infty = \sum_{k=-\infty}^\infty Cov(X_1, X_{1+k}) Cov(Y_1,Y_{1+k}).
\]

Because the variance of our estimates decays as $1/n$ in the central limit theorem, we need time series of length $n= \sigma_\infty \cdot \nef$ to reach the same estimation variance as $\nef$ i.i.d. samples for which $\sigma_\infty = 1$.

\section{Our Results Retain Their Validity Under Larger Time Length $n$}\label{sec:large_n}

In this section we analyse how our results evolve with increasing effective time length $n$.

While all estimated quantities better approximate the ground truth, the amount of errors in the networks are still troublesome and our conceptual findings still hold. In particular, the bias stemming from anisotropic autocorrelation only slightly improves with increasing $n$ (Fig. \ref{fig:fig11n500}), because this bias stems from the \textit{relative} difference in estimation variance depending on the autocorrelation strength. Also, few false shortcuts suffice to distort quantities like betweenness and shortest path lengths \textit{globally} (Fig. \ref{fig:fig5n500} and Fig. \ref{fig:fig7n1000}). Weighted networks benefit much quicker from larger $n$ in terms of degree estimation (Fig. \ref{fig:fig7n1000}) and spurious link bundles (Fig. \ref{fig:fig8n500}).

\begin{figure}[H]
    \hfill
    \begin{subfigure}[b]{0.99\textwidth}
    \centering
    \includegraphics[width=\textwidth]{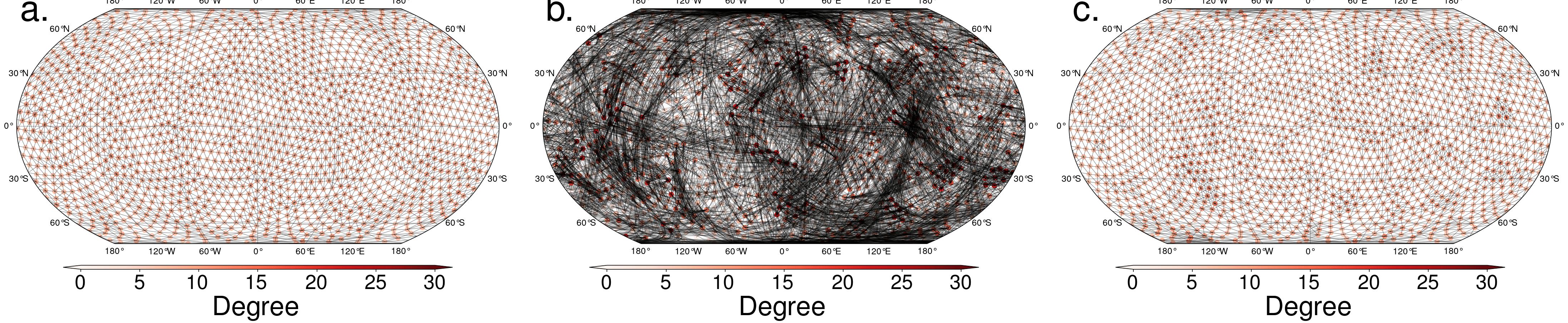}
    \end{subfigure}
    \caption{Same as Fig. 2 but with $n=1000$. The estimated networks look slightly better. But the same conceptual distortions are still clearly present: Empirical Pearson correlation fails as an estimator of heavy-tailed data. High degree clusters emerge as a result of localized correlation structure.}
    \label{fig:fig2n1000}
\end{figure}

\begin{figure}[H]
    \hfill
    \begin{subfigure}[b]{0.99\textwidth}
    \centering
    \includegraphics[width=\textwidth]{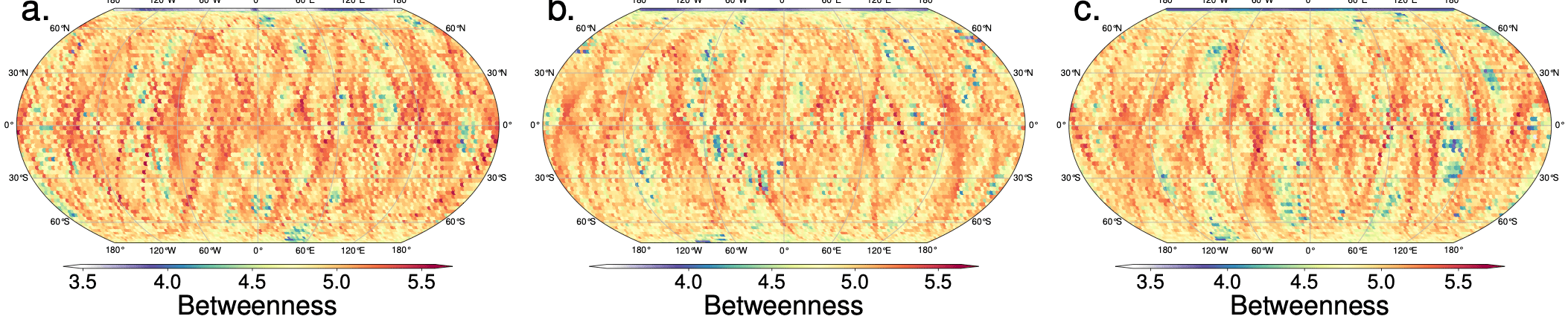}
    \end{subfigure}
    \caption{Same as Fig. 5 but with $n=500$. Few false links are enough to create spurious global betweenness structures.}
    \label{fig:fig5n500}
\end{figure}

\begin{figure}[H]
    \hfill
    \begin{subfigure}[b]{0.99\textwidth}
    \centering
    \includegraphics[width=\textwidth]{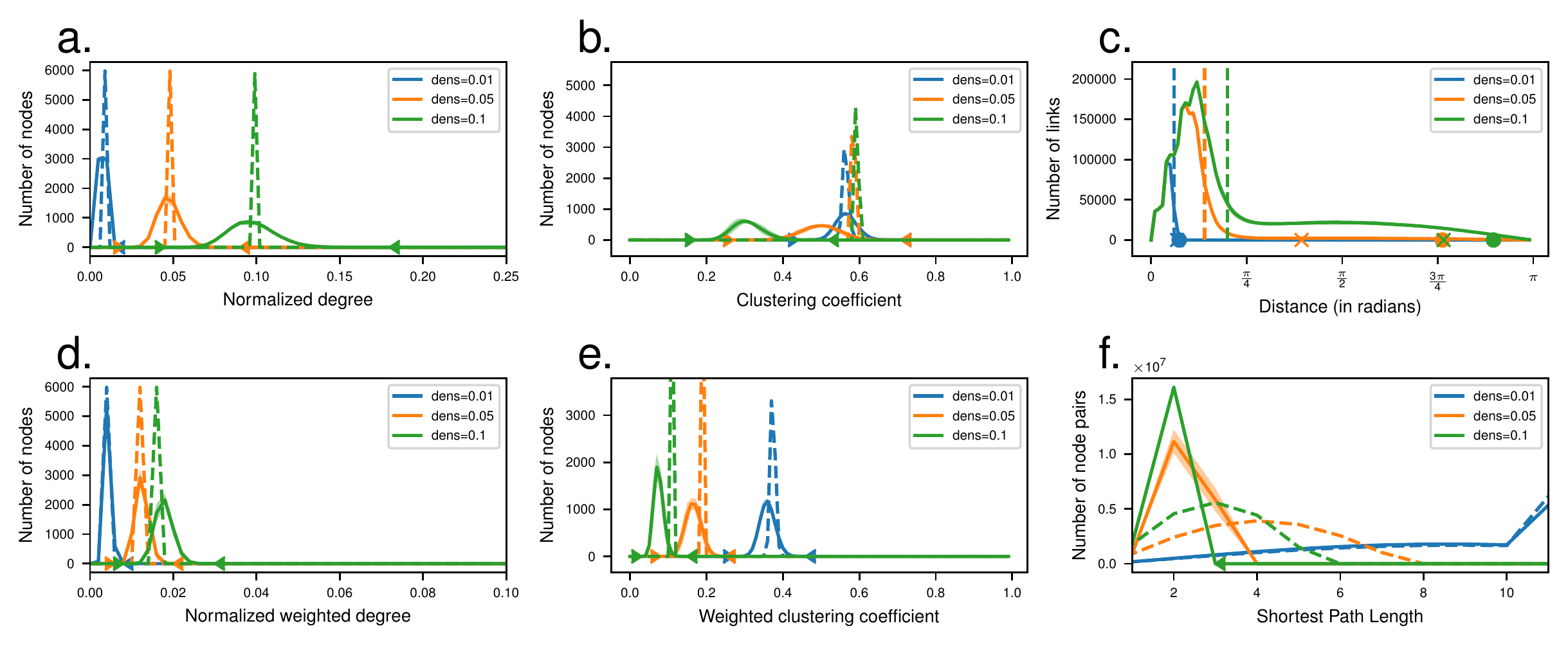}
    \end{subfigure}
    \caption{Same as Fig. 7 but with $n=500$. Compare to Fig. \ref{fig:fig7n1000} and observe all quantities slowly converging to their ground truth.}
    \label{fig:fig7n500}
\end{figure}

\begin{figure}[H]
    \hfill
    \begin{subfigure}[b]{0.99\textwidth}
    \centering
    \includegraphics[width=\textwidth]{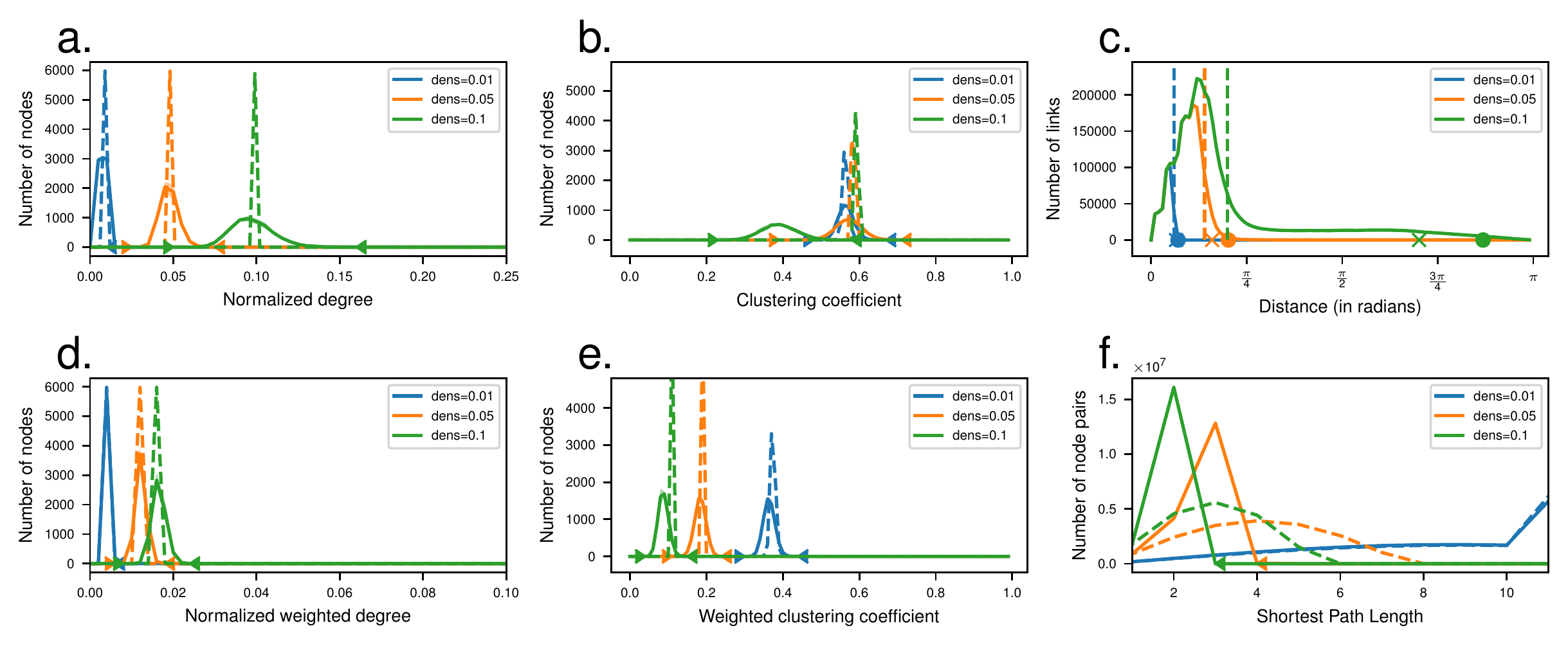}
    \end{subfigure}
    \caption{Same as Fig. 7 but with $n=1000$. Compare to Fig. \ref{fig:fig7n500} and observe all quantities slowly converging to their ground truth. Weighted degree converges faster than unweighted degree. To distort shortest path length single false shortcuts suffice.}
    \label{fig:fig7n1000}
\end{figure}

\begin{figure}[H]
    \hfill
    \begin{subfigure}[b]{0.99\textwidth}
    \centering
    \includegraphics[width=\textwidth]{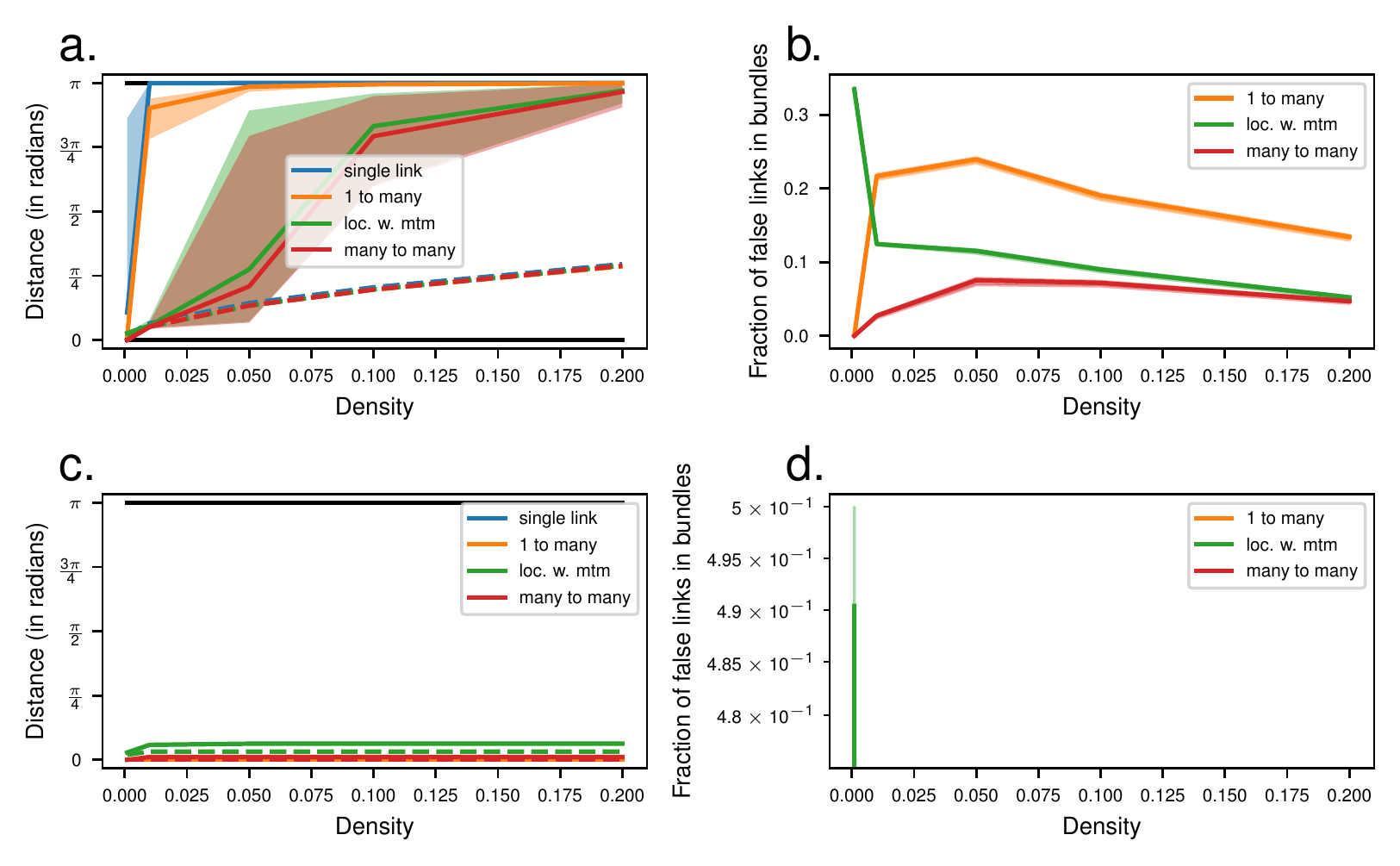}
    \end{subfigure}
    \caption{Same as Fig. 8 but with $n=500$. While the amount and length of unweighted link bundles has not decreased, the edge weights have converged faster, so that the weighted link bundles have disappeared.}
    \label{fig:fig8n500}
\end{figure}

\begin{figure}[H]
    \hfill
    \begin{subfigure}[b]{0.99\textwidth}
    \centering
    \includegraphics[width=\textwidth]{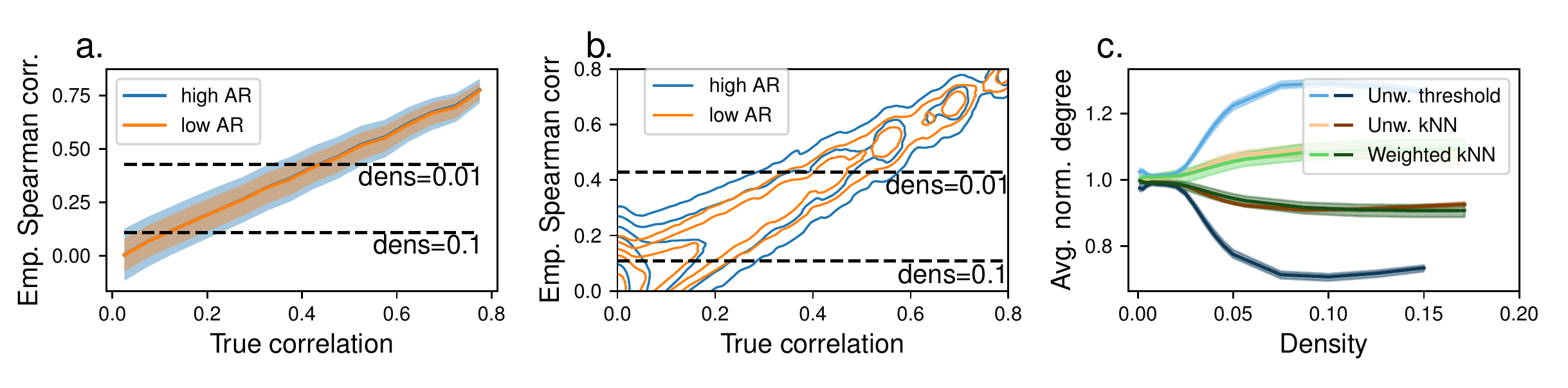}
    \end{subfigure}
    \caption{Same as Fig. 11 but with $n=500$. While the estimation variance has significantly decreased in absolute terms, the relative bias introduced through anisotropic autocorrelation has not improved.}
    \label{fig:fig11n500}
\end{figure}

\section{Further Simulation Results}
\label{sec:more_plots}

\subsection{Empirical Betweenness and Forman Curvature Distribution}\label{sec:betw_forman}

Sparse lattice-type graphs contain nodes of spuriously high betweenness as soon as one false (tele-)connection between regions is present. The network density where the betweenness distribution is maximally distorted depends on the length scale of the random field (Fig. \ref{fig:betw_curv}).

An alternative betweenness measure, which has gained popularity recently, is based on the concept of curvature \citepA{Forman2003,Ollivier2010}. It is and edge-based measure and does not consider shortest paths but a distance between neighborhoods of nodes. For Forman curvature we find a strong negative bias (Fig. \ref{fig:betw_curv}) as for the clustering coefficient in Fig. 7. Both measures are based on triangle counts.

\begin{figure}[H]
    \begin{subfigure}[b]{0.33\textwidth}
    \centering
    \includegraphics[width=\textwidth]{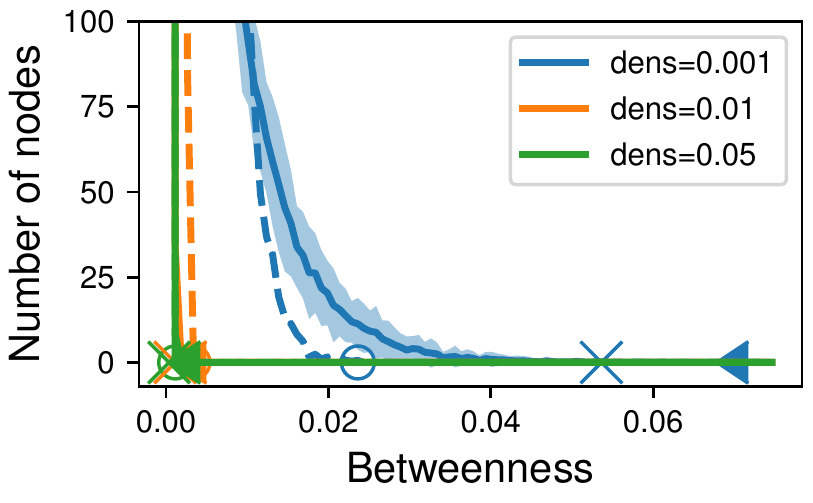}
    \label{fig:betwhist1}
    \end{subfigure}
    \hfill
    \begin{subfigure}[b]{0.33\textwidth}
    \centering
    \includegraphics[width=\textwidth]{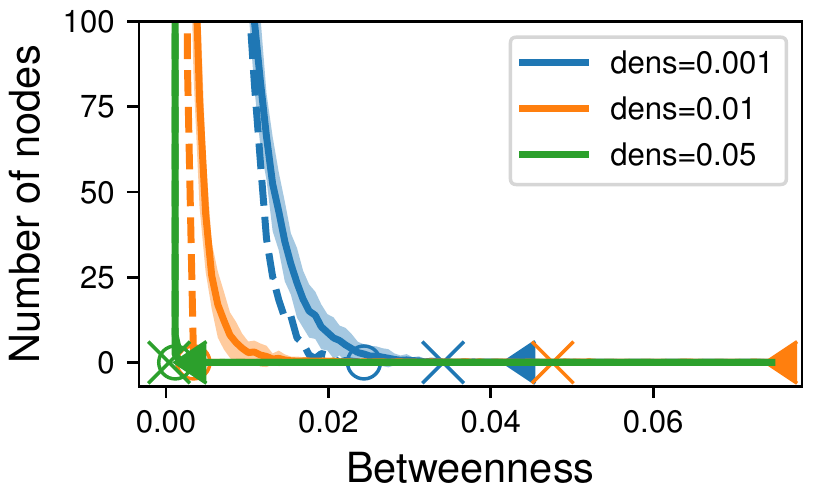}
    \label{fig:betwhist2}
    \end{subfigure}
    \hfill
    \begin{subfigure}[b]{0.33\textwidth}
    \centering
    \includegraphics[width=\textwidth]{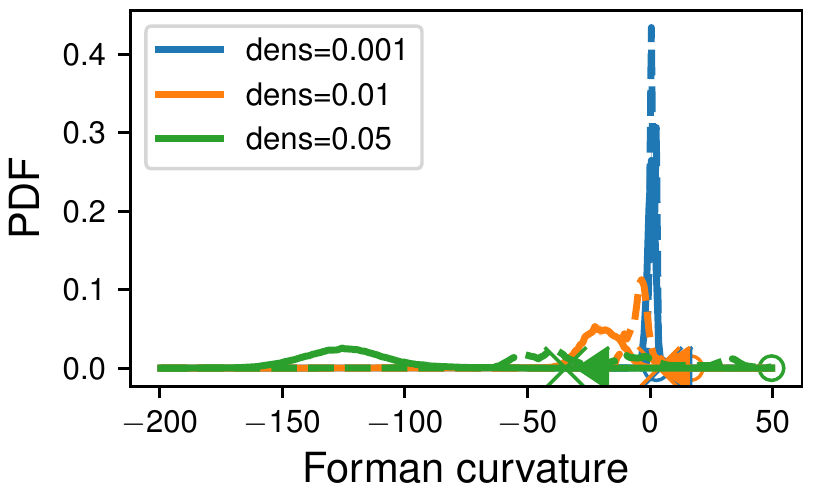}
    \label{fig:forman1}
    \end{subfigure}
    \caption{\textbf{Betweenness distributions in empirical networks.} Betweenness histograms with $2\sigma$-uncertainty bands from 30 independent realizations. 'x' and triangle denote the average and $95\%$-quantile of maximal betweenness values between runs, circle marks the maximal betweenness value in the ground truth graph. Denser networks contain less distortion as many noisy links average out spurious betweenness. \textit{Left}: Density-threshold graphs using Spearman correlation from a MIGRF($\nu=0.5$, $\ell=0.1$). \textit{Middle}: Same as left for MIGRF($\nu=1.5$, $\ell = 0.2$). Note that maximal distortion occurs at higher density compared to smaller length scale, as very sparse networks here do not contain spurious long-range links. The network density, where the reliable lattice estimation transitions into more false links, is the most distorted regime in terms of betweenness. \textit{Right}: Forman curvature distribution on the edges for networks from an MIGRF$(\nu=1.5,\ell=0.1)$. The empirical distribution is heavily downward biased, as false links are mostly 'between-cluster' links, i.e. have negative curvature. The more the lattice-like, clustered ground truth structure is disrupted, the more negative the distribution becomes.}
    \label{fig:betw_curv}
\end{figure}

\subsection{Network Measures For Small Length Scale}

Networks with large densities can not be well estimated as ground truth correlations are close to $0$ already at short distances. This critical distance can be seen in the drop in the link length distribution plot that marks a phase transition from systematic short links to random false links. The estimates of network characteristics become unreliable already at small network densities, as more false links occur in all network density regimes.

\begin{figure}[H]
    \hfill
    \begin{subfigure}[b]{0.99\textwidth}
    \centering
    \includegraphics[width=\textwidth]{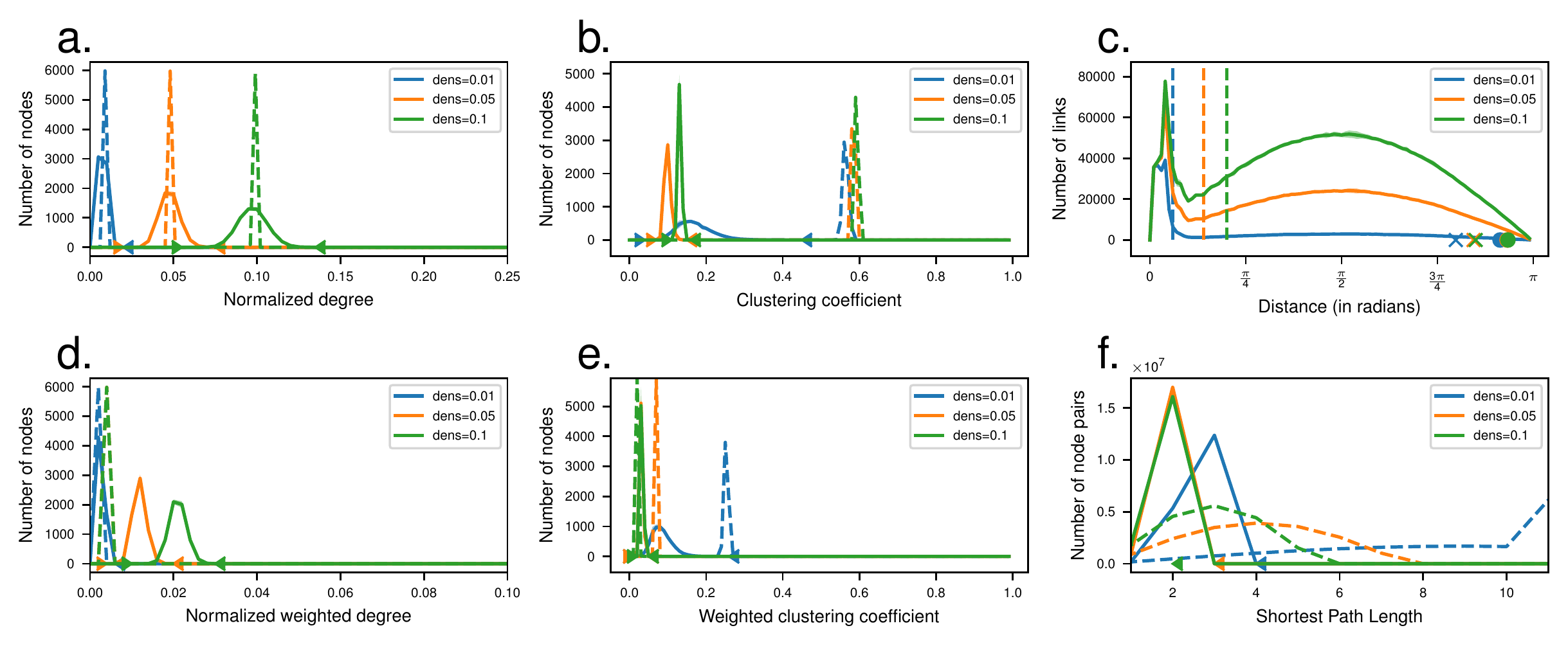}
    \end{subfigure}
    \caption{\textbf{Node/edge measure distributions for small length scales.} Same as Fig. 7 but with $\nu=0.5$ and $\ell=0.1$. More false links in dense graphs, because more edges are empirically indistinguishable from 0. Less spread out empirical distributions, because less pronounced link bundles are formed and false links rather average out across nodes.}
    \label{fig:graphmeasures2}
\end{figure}

\subsection{Networks from Noisy Measurements}\label{sec:noisy_meas}

Isotropic additive white noise can be understood as a modification of the correlation function to include a 'nugget effect' \citepA{Clark2010}. Then the correlation function approaches a values smaller than $1$, when the distance between points approaches $0$. Figure \ref{fig:graphmeasuresnoise} shows the same graph measures as Fig. 7, but for data with isotropic nugget effect. Under additive noise, all population correlations decrease, hence they lie closer to each other and variance in the graph construction increases. As the correlation between similarity estimates is now bounded away from 1, the independent noise on each node will sometimes produce outliers. Under non-negligible estimation variance and for infinitesimally fine grids, this typically leads to non-vanishing fractions formed as well as unformed links between regions, no matter how small the diameter of these regions. More errors occur in the network, irrespective of distance. Even in sparse networks, under large noise, the link length distribution and shortest path length are very distorted. Otherwise the behaviour is very similar. Under fixed grid resolution, continuous but non-smooth random fields and random fields with nugget effect can resemble each other. 

\begin{figure}[H]
    \hfill
    \begin{subfigure}[b]{0.99\textwidth}
    \centering
    \includegraphics[width=\textwidth]{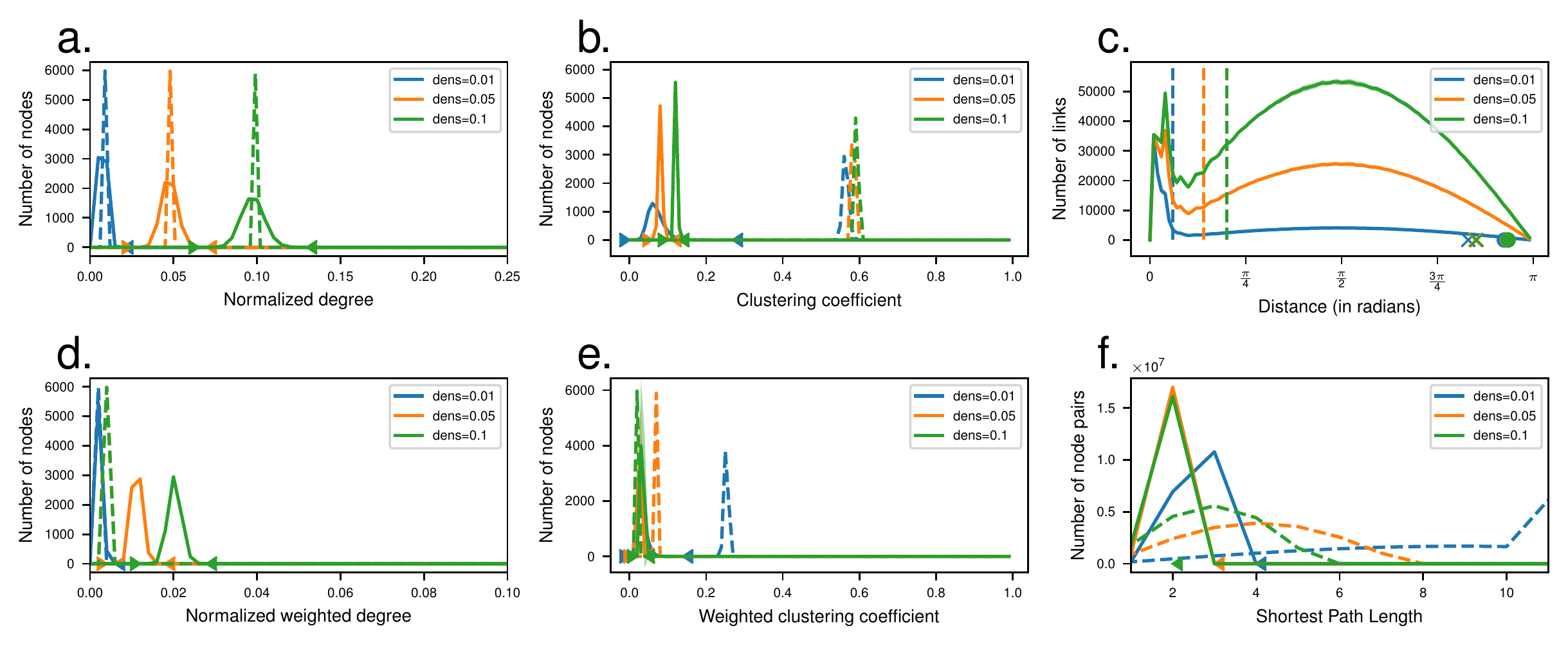}
    \end{subfigure}
    \caption{Same as Fig. 7 but with $0.7 N(0,1)$-additive noise on the data, $\nu=0.5$ and $\ell=0.1$.}
    \label{fig:graphmeasuresnoise}
\end{figure}

\subsection{Network Measures of kNN Graphs}

Figure \ref{fig:graphmeasuresknn} shows the same graph measures as Figure 7, but for kNN graphs with similar (but not the same) densities. As the ground truth graphs of threshold and kNN graphs are almost identical for isotropic data, the network characteristics behave very similarly. Observe a more pronounced peak in the unweighted degree distribution. As the ground truth behaviour of kNN and threshold graphs is the same in the isotropic setting, they do not differ much. In kNN graphs the nodes become more similar so that extreme values lie closer together. In case of the clustering coefficient this does not mean that that they lie closer to the ground truth values.

\begin{figure}[H]
    \hfill
    \begin{subfigure}[b]{0.99\textwidth}
    \centering
    \includegraphics[width=\textwidth]{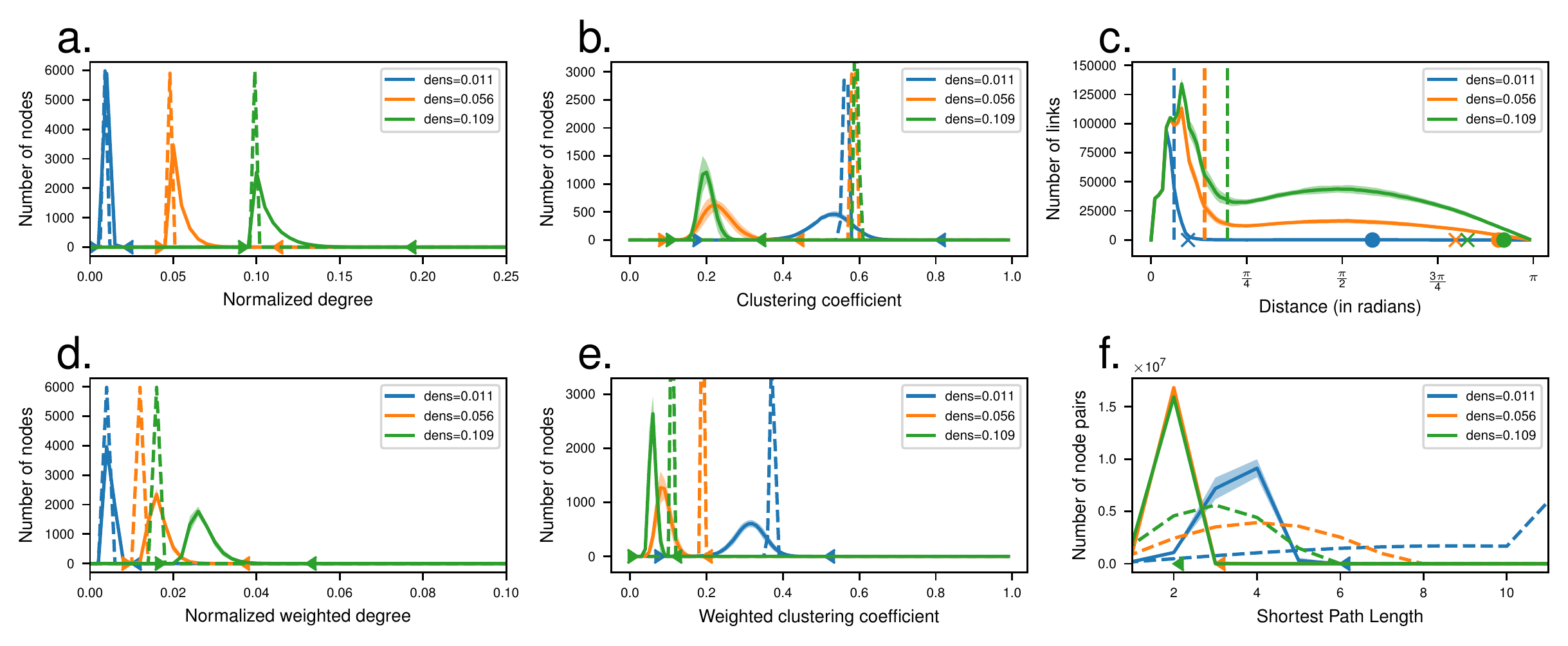}
    \end{subfigure}
    \caption{Same as Fig. 7 but for kNN graphs.}
    \label{fig:graphmeasuresknn}
\end{figure}

\subsection{Network Measures for Temperature Data}

Figure \ref{fig:graphmeasurest2m} shows the same graph measures for monthly t2m. Degree and clustering coefficient distributions are more stretched out, as different locations differ in distribution, but otherwise - especially the link length distribution - the plots look very similar to the empirical networks from our isotropic data. kNN graphs (Fig. \ref{fig:graphmeasurest2mknn}) augment this similarity, as they make the network more isotropic.

\begin{figure}[H]
    \hfill
    \begin{subfigure}[b]{0.99\textwidth}
    \centering
    \includegraphics[width=\textwidth]{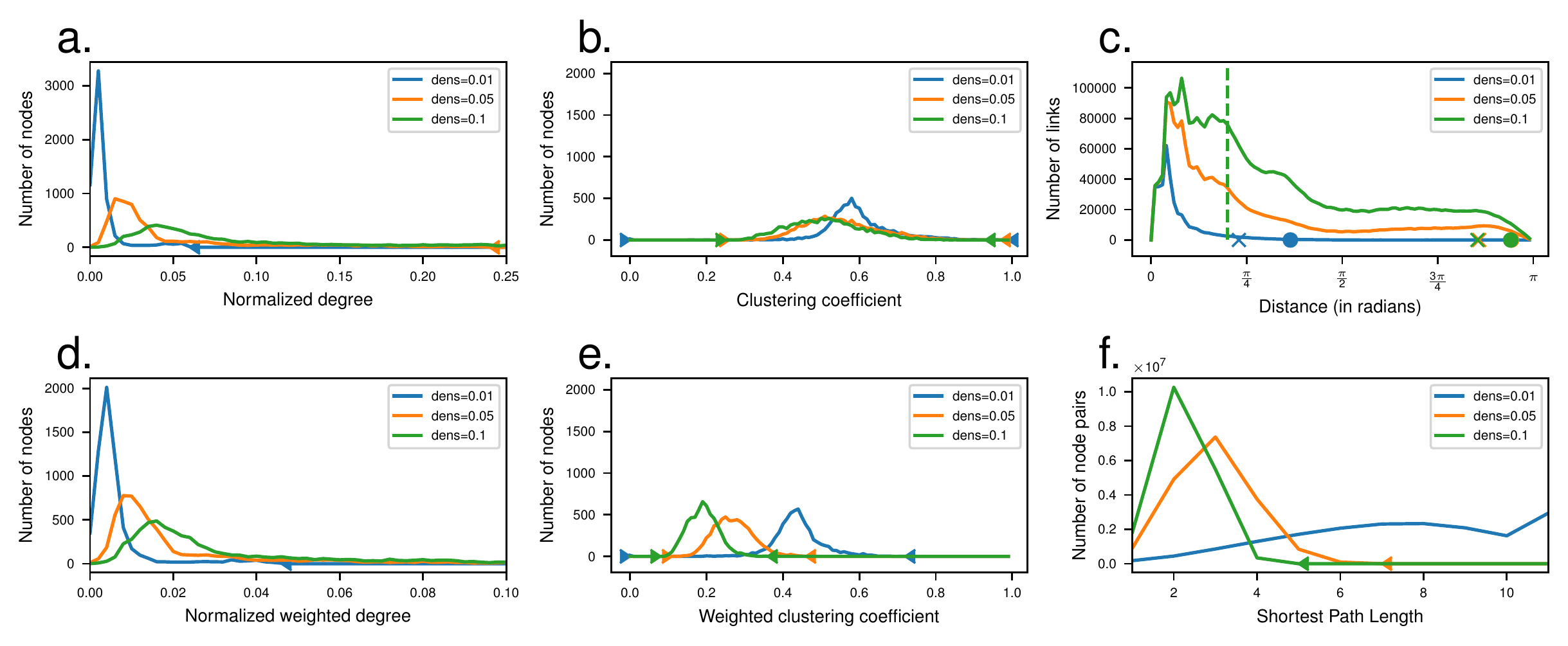}
    \end{subfigure}
    \caption{Same as Fig. 7 but for t2m. Observe even more stretched out degree and clustering coefficient distributions as there are distributional differences between locations. Otherwise there is striking similarity to our isotropic data, especially in the link length distribution.}
    \label{fig:graphmeasurest2m}
\end{figure}

\begin{figure}[H]
    \hfill
    \begin{subfigure}[b]{0.99\textwidth}
    \centering
    \includegraphics[width=\textwidth]{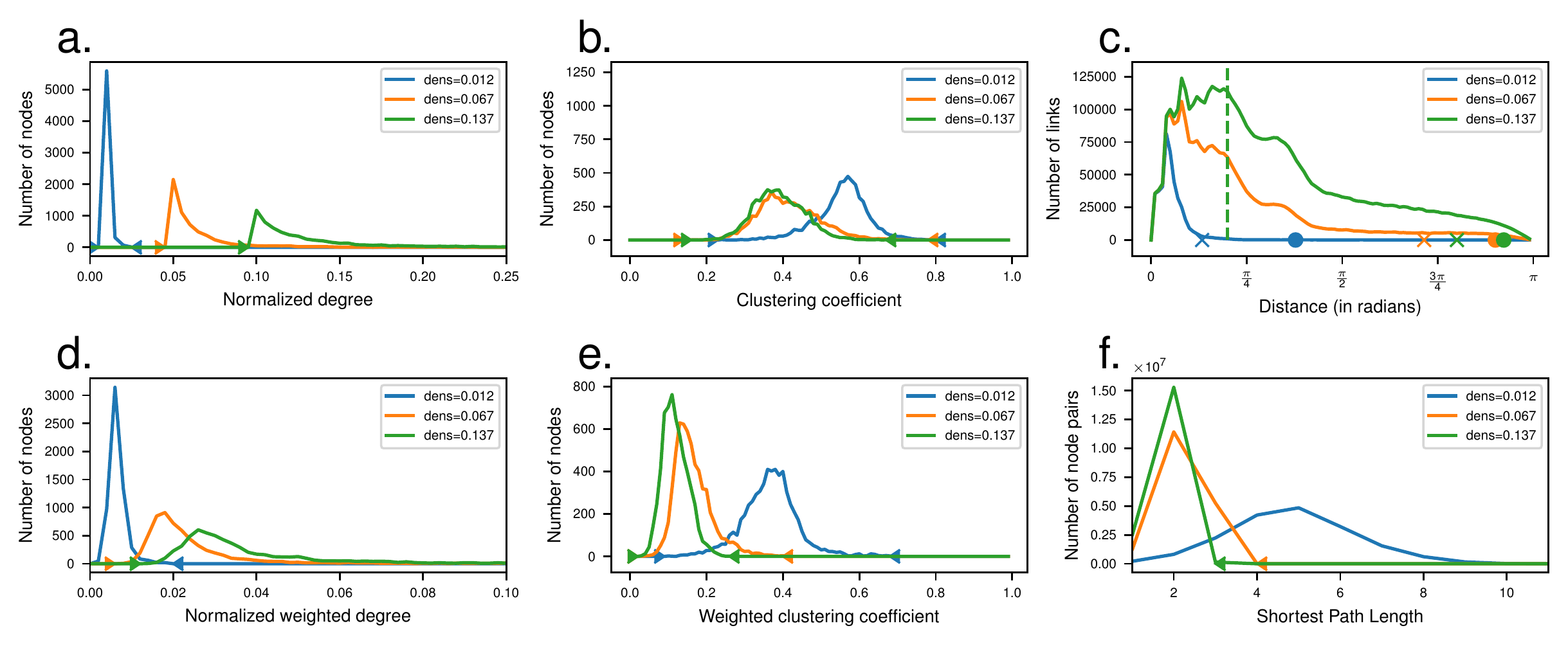}
    \end{subfigure}
    \caption{Same as Fig. 7 but kNN graphs for t2m. kNN graphs increase isotropy of the network. Its characteristics look even more similar to those of our isotropic data.}
    \label{fig:graphmeasurest2mknn}
\end{figure}

\subsection{Empirical Mutual Information Networks Show Less Bundling Behaviour}

\begin{figure}[H]
    \hfill
    \begin{subfigure}[b]{0.49\textwidth}
    \centering
    \includegraphics[width=\textwidth]{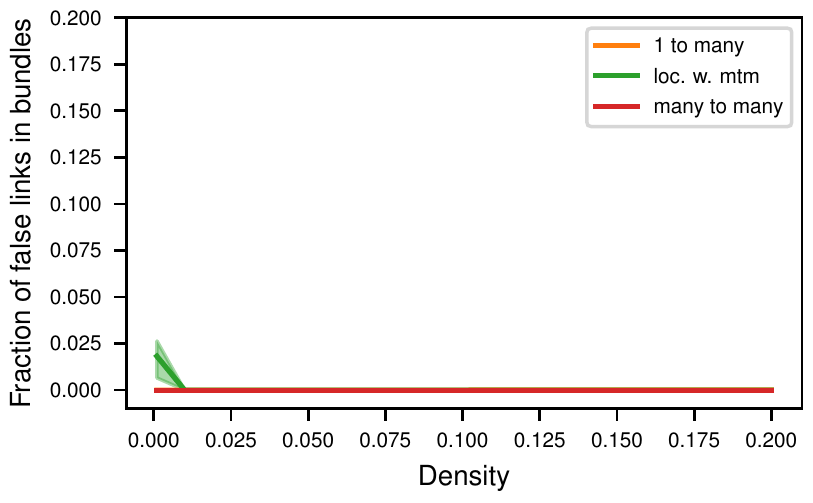}
    \label{biksgtele1}
    \end{subfigure}
    \hfill
    \begin{subfigure}[b]{0.49\textwidth}
    \centering
    \includegraphics[width=\textwidth]{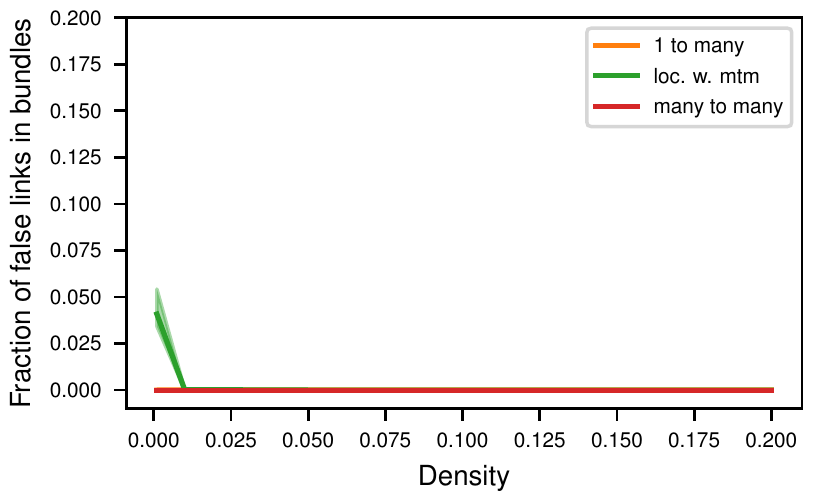}
    \label{biksgtele2}
    \end{subfigure}
    
    \begin{subfigure}[b]{0.49\textwidth}
    \centering
    \includegraphics[width=\textwidth]{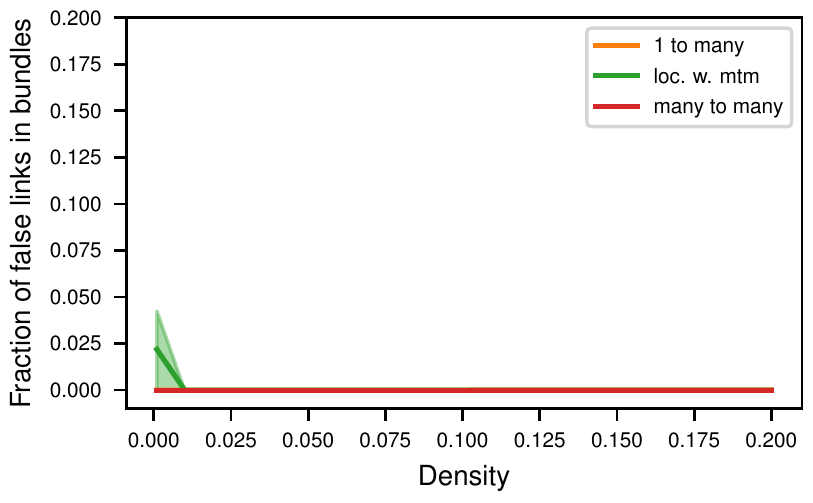}
    \label{biksgtele4}
    \end{subfigure}
    \hfill
    \begin{subfigure}[b]{0.49\textwidth}
    \centering
    \includegraphics[width=\textwidth]{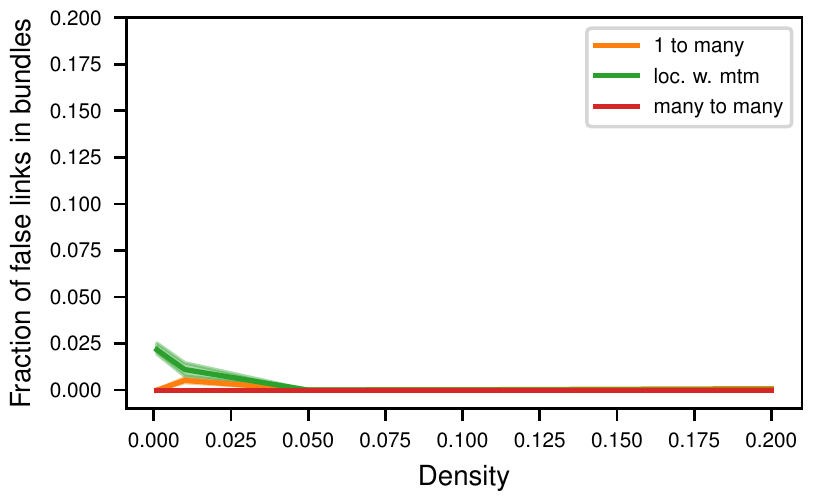}
    \label{biksgtele5}
    \end{subfigure}
    \hfill
    \caption{\textbf{Links are much less bundled under the BIKSG estimator.} The fraction of false links in bundles as in Fig. 8 for unweighted BIKSG estimated networks for different hyperparameters of the MIGRF. The top row shows $\nu=0.5$, the bottom row shows $\nu=1.5$. The left column shows $\ell=0.1$, the right column shows $\ell=0.2$. Although the estimator makes a lot more errors, it does not transmit the localized correlation structure in the underlying data.}
    \label{fig:biksgtele}
\end{figure}

The BIKSG mutual information estimates are not as spatially dependent as the Pearson correlation estimates. Thus although the mutual information networks contain more false links, these are not as bundled (Fig. \ref{fig:biksgtele}) compared to empirical Pearson correlation in Fig. 8. Hence the estimator plays an important role in how much structural as well as spurious bundling behaviour is present in empirical networks.

\subsection{Unstable Links in Climate Networks}

Figure \ref{fig:llreal} shows the link length distributions and length distributions of differing links of bootstrapped Pearson correlation networks from several climatic variables. Observe that short links are formed first and then the distribution approaches a sinusoidal form. Especially for precipitation, the network of density $0.1$ has a distribution close to that of random graphs with uniform link distribution; most links differ between bootstrap samples. Geopotential heights operate on long length scales and have few differing links. This indicates that these networks are robust, especially for large densities (Fig. \ref{fig:realz}).

\begin{figure}[H]
    \hfill
    \begin{subfigure}[b]{0.33\textwidth}
    \centering
    \includegraphics[width=\textwidth]{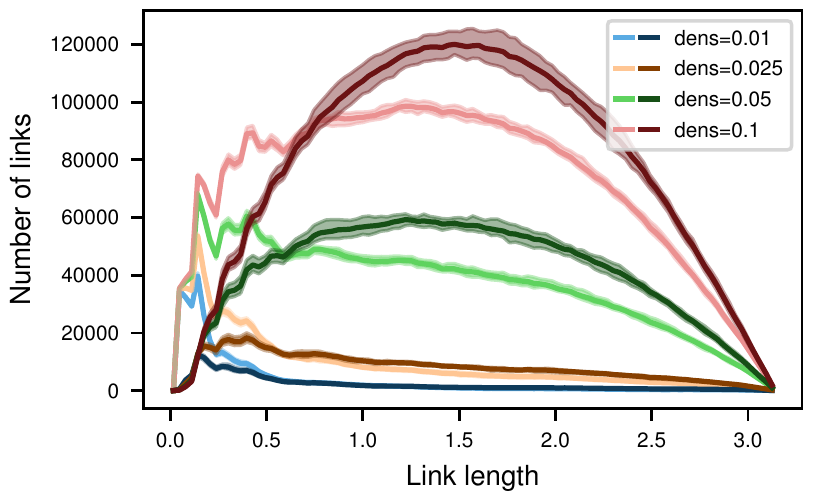}
    \label{graphm77}
    \end{subfigure}
    \hfill
    \begin{subfigure}[b]{0.33\textwidth}
    \centering
    \includegraphics[width=\textwidth]{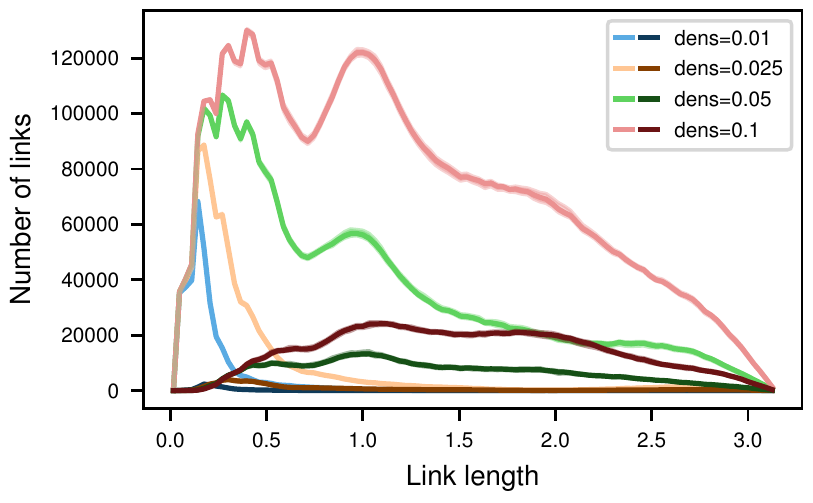}
    \label{graphm87}
    \end{subfigure}
    \hfill
    \begin{subfigure}[b]{0.33\textwidth}
    \centering
    \includegraphics[width=\textwidth]{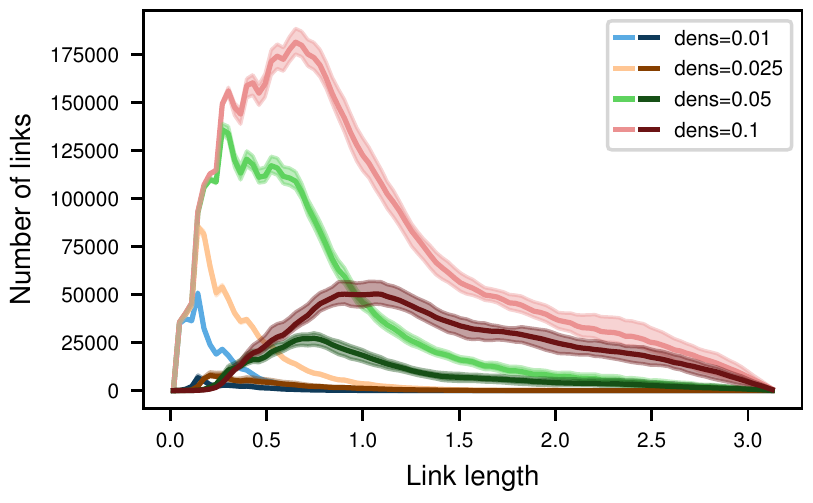}
    \label{graphm97}
    \end{subfigure}
    
    \hfill
    \begin{subfigure}[b]{0.33\textwidth}
    \centering
    \includegraphics[width=\textwidth]{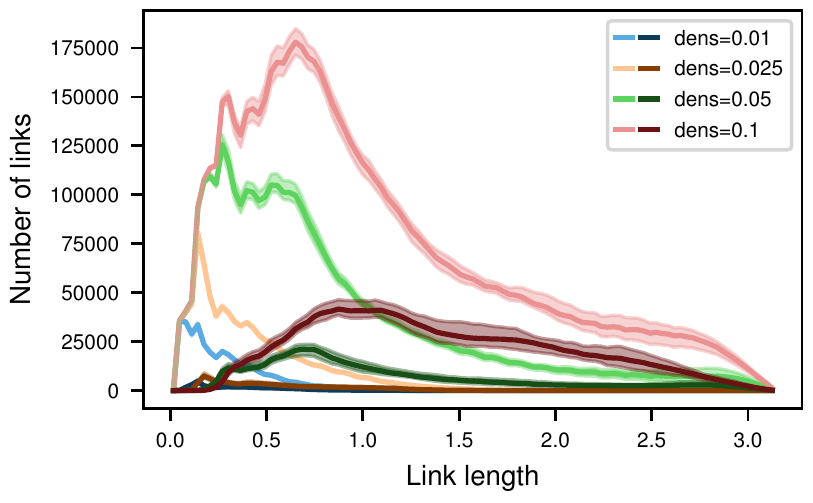}
    \label{graphm107}
    \end{subfigure}
    \hfill
    \begin{subfigure}[b]{0.33\textwidth}
    \centering
    \includegraphics[width=\textwidth]{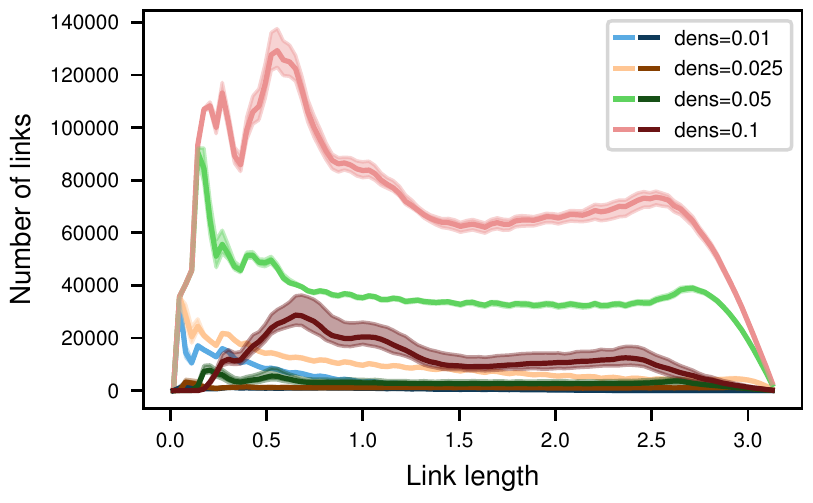}
    \label{graphm127}
    \end{subfigure}
    \hfill
    \begin{subfigure}[b]{0.33\textwidth}
    \centering
    \includegraphics[width=\textwidth]{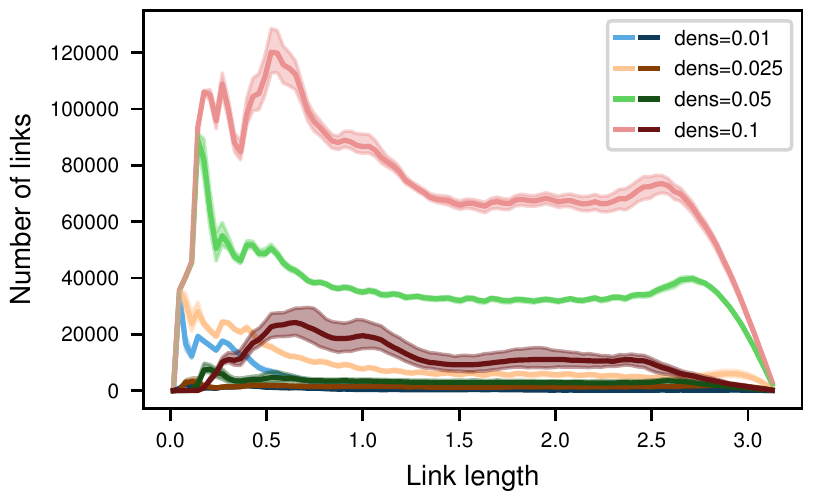}
    \label{graphm117}
    \end{subfigure}
    \caption{Link length distributions (light) and length distribution of differing links of bootstrapped Pearson correlation networks from real data. \textit{Top}: Precipitation, daily t2m, surface pressure. \textit{Bottom}: Geopotential heights at $850$, $500$ and $250$ mb.}
    \label{fig:llreal}
\end{figure}

\begin{figure}[H]
    \begin{subfigure}[b]{0.49\textwidth}
    \centering
    \includegraphics[width=\textwidth]{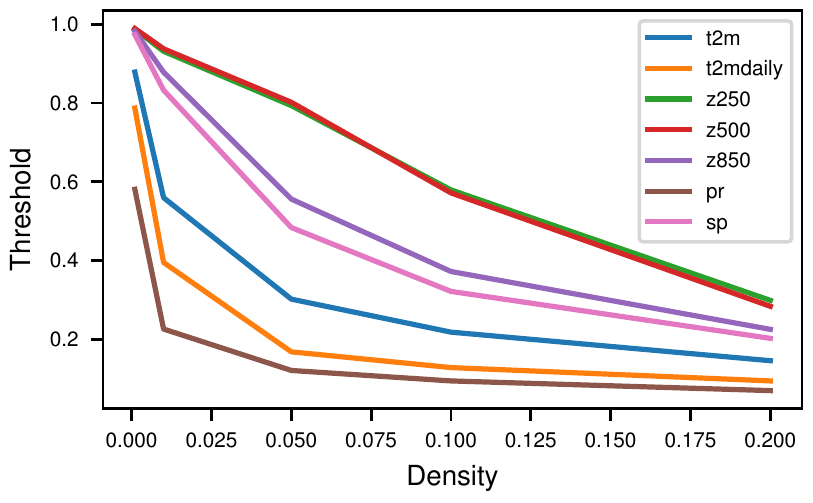}
    \label{sub:densthresh}
    \end{subfigure}
    \hfill
    \begin{subfigure}[b]{0.49\textwidth}
    \centering
    \includegraphics[width=\textwidth]{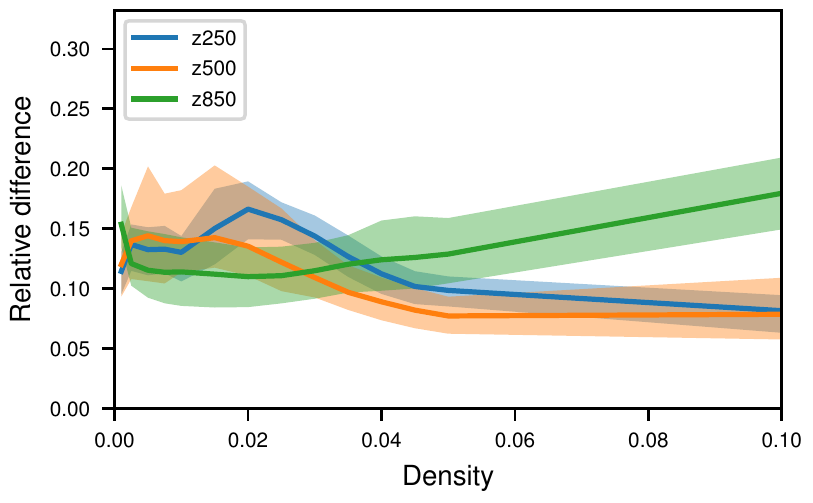}
    \label{fig:densthresh_z}
    \end{subfigure}
    \caption{\textit{Left}: The threshold of the Pearson correlation network as a function of network density. \textit{Right}: Same as Fig. 4 but for geopotential heights. How the relative difference between empirical networks is computed is described in Section 3a2.}
    \label{fig:realz}
\end{figure}

\subsection{Distribution of mean and variance estimates from IAAFT surrogates}\label{sec:iaaft_distrib}

Here we analyse why all edges in sparse IAAFT-based z-score networks are formed by nodes with low autocorrelation. For this purpose we compare the estimates on the edges between nodes of low autocorrelation with the estimates on the edges between nodes of high autocorrelation. Figure \ref{fig:delta_iaaft} shows the difference between nodes of high and nodes of low autocorrelation in the distributions of mean estimates $\hat{\mu}_{ij}^0$ (left) and variance estimates $\hat{\sigma}_{ij}^0$ (right) based on IAAFT resampling. Positive values indicate more weight in the distribution of highly autocorrelated nodes and vice-versa. Naturally, the variance of the mean estimates grows with autocorrelation (left). The variance estimates for highly autocorrelated nodes are larger (right), reflecting the true increased estimation variance. Since we divide by variance estimates, which are small in absolute value, the highest z-scores are attained for nodes with low autocorrelation.

\begin{figure}[H]
    \begin{subfigure}[b]{0.49\textwidth}
    \centering
    \includegraphics[width=\textwidth]{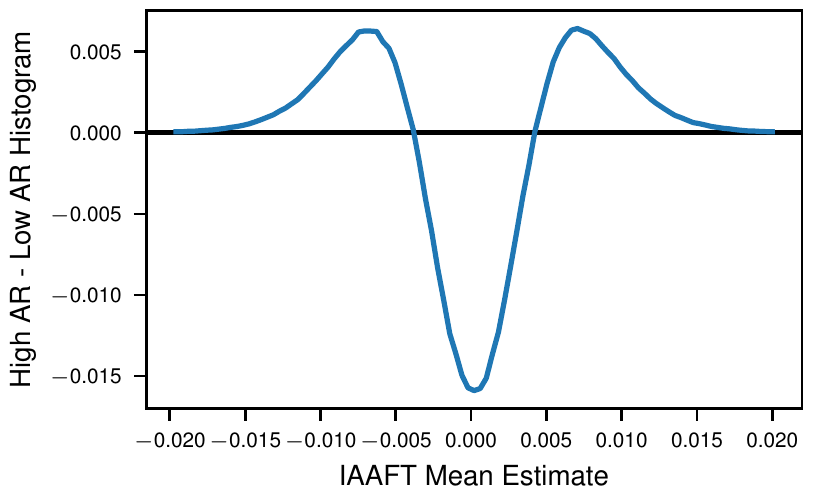}
    \label{sub:delta_mean}
    \end{subfigure}
    \hfill
    \begin{subfigure}[b]{0.49\textwidth}
    \centering
    \includegraphics[width=\textwidth]{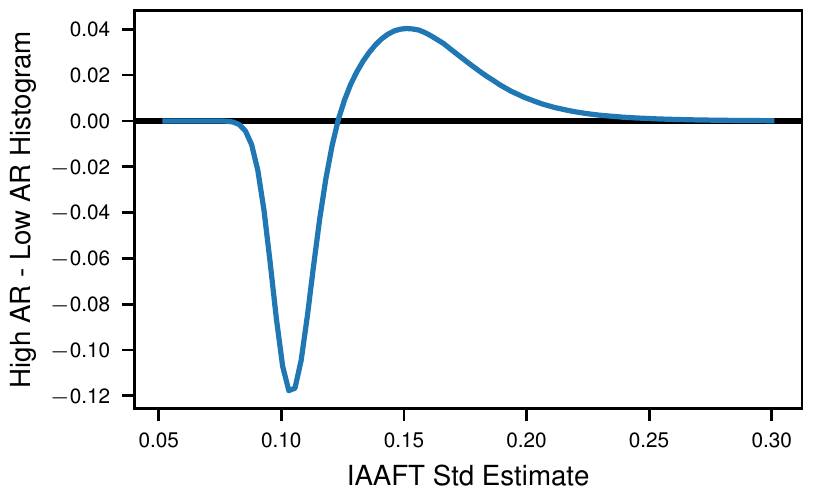}
    \label{fig:delta_std}
    \end{subfigure}
    \caption{Difference between nodes of high lag-1 autocorrelation $0.7$ and nodes of low lag-1 autocorrelation $0.2$ of the normalized histograms of mean estimates (\textit{left}) and standard deviation estimates (\textit{right}) based on IAAFT resampling. Positive values indicate more weight in the distribution associated with high autocorrelation and negative values indicate more weight in the distribution associated with low autocorrelation. Mean estimates have larger variance for high autocorrelation. Variance estimates are larger for high autocorrelation, as they should be.
    }
    \label{fig:delta_iaaft}
\end{figure}


\bibliographystyleA{arxivstyle/plainnat_clean}
\bibliographyA{arxiv_version.bib}
\end{appendices}


\end{document}